\documentclass[12pt]{article}
\usepackage{arxiv}

\usepackage{graphicx}
\usepackage{natbib} %comment out if you do not have the package
\usepackage{url} % not crucial - just used below for the URL 
\newcommand{\blind}{0}

\usepackage{amsmath}
\usepackage{amsfonts} 
\usepackage{amsthm}
\usepackage{algorithm}
\usepackage{algorithmic}
\usepackage{subfigure}
\usepackage[utf8]{inputenc} % allow utf-8 input
\usepackage[T1]{fontenc} % use 8-bit T1 fonts
\usepackage{hyperref} % hyperlinks

\usepackage{booktabs} % professional-quality tables
\usepackage{nicefrac} % compact symbols for 1/2, etc.
\usepackage{microtype} % microtypography
\usepackage{xcolor} % colors
\usepackage{array,tabularx}
\renewcommand{\shorttitle}{Bayesian Robust Regression}
\begin{document}

	\def\spacingset#1{\renewcommand{\baselinestretch}%
		{#1}\small\normalsize} \spacingset{1}
	
%%%%%%%%%%%%%%%%%%%%%%%%%%%%%%%%%%%%%%%%%%%%%%%%%%%%%%%%%%%%%%%%%%%%%%%%%%%%%%

\if0\blind 
{
\title{\bf A Bayesian Robust Regression Method for Corrupted Data Reconstruction}
\author{
Zheyi Fan\thanks{Academy of Mathematics and Systems Science, Chinese Academy of Sciences, China;
School of Mathematical Sciences, University of Chinese Academy of Sciences, China. \texttt{fanzheyi@amss.ac.cn}, \texttt{Qingpeihu@amss.ac.cn}}, \\
\And
Zhaohui Li\thanks{H.\ Milton Stewart School of Industrial and Systems Engineering, Georgia Institute of Technology, USA. \texttt{zhaohui.li@gatech.edu}}, \\
\And
Jingyan Wang\thanks{Beijing Institute of Spacecraft System Engineering, China. \texttt{yanering@tom.com}, \texttt{xiong\_ztb@126.com}}, 
\And
Dennis K. J. Lin\thanks{Department of Statistics, Purdue University, West Lafayette, Indiana, \texttt{dkjlin@purdue.edu}}, 
\And
Xiao Xiong\footnotemark[3], 
\And
Qingpei Hu\footnotemark[1]
}

\maketitle
} 
\fi
\if1\blind
{
\begin{center}
{\LARGE\bf A Bayesian Robust Regression Method for Corrupted Data Reconstruction}
\end{center}
\medskip
} \fi
\bigskip
\begin{abstract}
Because of the widespread existence of noise and data corruption, recovering the true regression parameters with a certain proportion of corrupted response variables is an essential task. Methods to overcome this problem often involve robust least-squares regression, but few methods perform well when confronted with severe adaptive adversarial attacks. In many applications, prior knowledge is often available from historical data or engineering experience, and by incorporating prior information into a robust regression method, we develop an effective robust regression method that can resist adaptive adversarial attacks. First, we propose the novel TRIP (hard \textbf{T}hresholding approach to \textbf{R}obust regression with s\textbf{I}mple \textbf{P}rior) algorithm, which improves the breakdown point when facing adaptive adversarial attacks. Then, to improve the robustness and reduce the estimation error caused by the inclusion of priors, we use the idea of Bayesian reweighting to construct the more robust BRHT (robust \textbf{B}ayesian \textbf{R}eweighting regression via \textbf{H}ard \textbf{T}hresholding) algorithm. We prove the theoretical convergence of the proposed algorithms under mild conditions, and extensive experiments show that under different types of dataset attacks, our algorithms outperform other benchmark ones. Finally, we apply our methods to a data-recovery problem in a real-world application involving a space solar array, demonstrating their good applicability.
\end{abstract}
\noindent%
{\it Keywords:} Robust regression; Hard thresholding; Bayesian reweighting; Variational inference
\vfill
\newpage

\newpage
\spacingset{1.5} 

%%%%%%%%%%%%%%%%%%%%%%%%%%%%%%%%%%%%%%%%%%%%%%%%%
\section{Introduction}\label{introduction}

Ordinary least-squares methods are used widely because of their simplicity and ease of operation. However, because of the inevitable existence of outliers, least-squares methods such as linear regression may cause significant bias in practical applications \citep{rousseeuw2005robust}. Many robust regression methods have been proposed to reduce the impact of outliers, but with most of them it is difficult to ensure good results under a high proportion of outliers with complex distribution. For example, one extreme outlier will severely affect the results of Huber regression \citep{huber1992robust} and least absolute deviation (LAD) regression \citep{pollard1991asymptotics}, and recent methods \citep{bhatia2017consistent, pmlr-v99-suggala19a} give good results only when the generation of outliers is independent of the data. Herein we propose two robust regression algorithms, i.e., TRIP (hard \textbf{T}hresholding approach to \textbf{R}obust regression with s\textbf{I}mple \textbf{P}rior) and BRHT (robust \textbf{B}ayesian \textbf{R}eweighting regression via \textbf{H}ard \textbf{T}hresholding), which incorporate the prior information of the parameter for more-robust regression. 

The motivation for this work is recovering data from a space solar array (SSA) on a satellite in low earth orbit (LEO). As an important component of a satellite, the SSA powers the entire satellite platform to guarantee normal operation \citep{kim2010analytical}, so the satellite owner or ground station must monitor the power data to ensure the health status of the SSA \citep{brandhorst2008space}. While a satellite is in orbit, the output power of its SSA often degrades slowly \citep{sheng2019residual}, but satellite design lifetimes often exceed ten years, which makes guaranteeing reliability a challenging and essential task. It is wholly unrealistic to repair or replace a broken part of an SSA in space, therefore there is great value in analyzing the degradation mode of the SSA to obtain important reliability indicators, such as its lifetime.

However, data transmission between satellites and ground stations often suffers corruption due to bandwidth limitations or environmental interference \citep{elaasar2005satellite}. Also, the performance of the devices on a satellite decreases gradually during operation, leading to reduced data quality and causing outliers or even censoring data. Moreover, for some older satellites, their small memory capacity significantly limits the data storage; when the main memory of the satellite is full but the satellite is far from the ground station (thereby preventing the former from transmitting data), previous data are overwritten by new data, which may produce many outliers. Consequently, conventional statistical methods are not applicable to the original SSA data, and one must first recover/reconstruct the data to solve for the missing data and outliers \citep{gull1978image}.

To recover the signal, one idea is to use a robust regression method to fit the signal of each period to remove the influence of outliers and estimate the real parameter information. However, the proportion of such data outliers is relatively high, and their distribution can be complicated and unknown because of multiple patterns of formation mechanisms (see Section~\ref{sec:application} for details). An important concept in robust regression is the breakdown point $\alpha$, which is a measure of robustness and is defined as the proportion of corruptions that the robust least-squares regression (RLSR) algorithm can tolerate. Thus, it is difficult to model the generation mode of outliers, such as the widely used gross error model \citep{hampel1968contributions}. The conventional robust regression methods perform poorly on such data, and there is an urgent need to develop a new method to deal with such complex problems.

Herein, we propose two robust regression algorithms, i.e., TRIP and BRHT, to recover data for such problems. 
These two algorithms use prior information to enhance the effect of robust regression and increase the breakdown point of the algorithm. In this work, the prior information is a rough estimate of the real parameters, which may not be very accurate. Given the periodicity of SSA data, those in each period are similar, we can select a period with few outliers and obtain parameters by using polynomial regression. These parameters can then be treated as common prior information and can be used to recover all periods of the data. The TRIP algorithm is computationally efficient, while the BRHT algorithm is more accurate but less efficient; therefore, the TRIP algorithm is more appropriate for large-scale data or online data analysis. 
These two algorithms are first proposed in \cite{fan2022nips}, in which the main focus is the construction and performance of the algorithms.
In this paper, we rigorously formulate the Bayesian robust regression with a statistical interpretation. 
Besides the convergence of algorithms, the quality of the prior is a key factor that can affect the statistical inference of the regression model. 
Thus, we propose a data-driven method for determining the prior distribution of parameters.
The method proposed in this paper is applied to a real-world SSA reconstruction problem.

For analysis, we use the power data of an anonymous SSA recorded for 300 days by an orbiting LEO satellite. The data have several important characteristics besides periodicity: (i) they exhibit an overall change in trend and several jump points, which indicate several SSA failures at specific times; (ii) the data number 20 million in total, thereby requiring a denoising method with low time complexity; (iii) most problematically, there are many outliers that unlike random noise have a certain pattern, which makes it difficult to eliminate them using conventional methods. This type of outlier occupies a high proportion (ca.\ $20\%$--$30\%$) in all periods, thereby requiring the breakdown point of the robust regression method to exceed this proportion. Given the large amount of data, we use the TRIP algorithm for data recovery, and the result shows that it recovers the true information perfectly. Based on the recovered data, we select a representative point in each period to extract the overall trend of the SSA power, and those data can be used for further degradation analysis or extracting other reliability indicators. Currently, the global aerospace industry is accelerating and China is planning to build a space station, so a data-recovery algorithm for SSAs has high application value.

The rest of this paper is organized as follows. In Section~2, we review previous work on robust regression methods and robust Bayesian models. In Section~3, we formulate the problem and introduce the required notation and tools. In Section~4, we describe our proposed TRIP and BRHT algorithms in detail, and in Section~5 we discuss their theoretical properties. In Section~6, we present extensive experimental results that demonstrate the excellent performance of our proposed algorithms. In Section~7, we propose an SSA data-recovery algorithm, and its results show that it has good recovery performance and high application value. Finally, we conclude the paper in Section~8.

%%%%%%%%%%%%%%%%%%%%%%%%%%%%%%%%%%%%%%%%%%%%%%%%
\section{Related Work}\label{literature_review}

To meet the challenge of learning reliable regression coefficients in the presence of significant data corruption, researchers have paid increasing attention to RLSR, which has excellent application value in many fields. Examples can be found in signal processing \citep{huang2017scalable,lin2019robust,zoubir2018robust}, economics \citep{wang2022prediction}, industry \citep{wang1997robust}, biology \citep{lourencco2011robust}, and computer vision \citep{Robust_Regression_Methods_for_Computer_Vision}.

To solve the RLSR problem, a commonly used data-generation model is $\mathbf{y}=X^{T}\mathbf{w}^{*}+\mathbf{b}^{*}+\boldsymbol{\epsilon}$, where $\mathbf{y}$ contains the response data of size $n$ (i.e., a column vector of size $1\times n$), $\mathbf{w}^*$ is the true regression coefficient we wish to recover, and $\boldsymbol{\epsilon}$ is a dense white-noise vector representing the random observation error. The vector $\mathbf{b}^{*}$ is $k$-sparse, i.e., only $k$ elements are non-zero, representing $k$ outliers in the data. In past decades, numerous methods were developed to solve the problem represented by \eqref{problem}, but those methods typically perform well only under specific conditions. The main challenge is the low breakdown point of conventional methods, with many RLSR algorithms failing to guarantee theoretical convergence as the value of $k$ increases. For example, the conventional Huber regression \citep{huber1992robust} can only tolerate $\alpha=1/n$ outliers, as can LAD regression \citep{pollard1991asymptotics}. Some recent methods such as that due to \citet{mcwilliams2014fast} use weighted subsampling for linear regression, but then the breakdown point is only $\alpha=\textit{O}(1/\sqrt{d})$, where $d$ is the dimension of the covariate. \citet{PSBR} proposed a robust gradient estimator that can be applied to linear regression, but their method only tolerates corruption up to $\alpha=\textit{O}(1/\log{d})$. Some robust regression estimators have high breakdown value, including repeated median regression \citep{siegel1982robust}, least median squares \citep{rousseeuw1984least}, least trimmed squares \citep{rousseeuw1984least}, S-estimators \citep{rousseeuw1984robust}, and MM-estimators \citep{yohai1987high}. However, those methods are either very computationally expensive or struggle to guarantee solution optimality.

Some methods have a higher breakdown point, but those tend to assume a specific pattern of data corruption. In general, there are two types of corruption, i.e., oblivious adversarial attack (OAA) and adaptive adversarial attack (AAA). In an OAA, the opponent generates $k$ sparse corruption vectors while completely ignoring $X$, $\mathbf{w}^*$, and $\boldsymbol{\epsilon}$, while an AAA is a more severe attack in which the opponent can view everything in the model including $X$, $\mathbf{w}^*$, and $\boldsymbol{\epsilon}$ before determining $\mathbf{b}^*$. \citet{bhatia2017consistent} and \citet{pmlr-v99-suggala19a} reported excellent results against OAAs by using a novel hard-thresholding method. Indeed, \citet{pmlr-v99-suggala19a} suggested that $\alpha$ may even approach unity as $n \to \infty$. The recent online fast robust regression algorithm \citep{OnlineRobustRegression} also has consistent convergence under mild conditions by using the stochastic gradient decent algorithm. However, those methods cannot resist an AAA, the handling of which is challenging, with many methods guaranteeing only a very low breakdown point, especially when the data are not distributed normally \citep{OptimalRobustLinearRegressioninNearlyLinearTime, RobustRegressionRevisited, Robustregressionwithcovariatefiltering}. \citet{bhatia2015robust} proposed a robust regression method based on a thresholding operator, and their breakdown point reached $1/65$ for a noiseless model, i.e., $\boldsymbol{\epsilon} \equiv 0$; however, their method gives consistent estimation only in noiseless cases. \citet{karmalkar2018compressed} obtained good results in sparse robust linear regression by applying $\text{L}_1$ regression, and the breakdown point of their method reached 0.239; however, their estimation is consistent only when the observation error $\boldsymbol{\epsilon}$ is sparse. \citet{diakonikolas2019efficient} considered the situation in which $X$ and $\mathbf{y}$ may have outliers simultaneously, and they proposed a filter algorithm in which the error bound is $\textit{O} (\alpha \log (1 / \alpha)\sigma)$; however, their method requires either accurate data covariance of the true data distribution or numerous unlabeled correct data to estimate the data covariance, which are often unavailable in practice.

The limitations of the aforementioned methods can be attributed to a lack of prior knowledge from the real data, making it difficult to distinguish the set of correct points in the case of an AAA. \citet{gulccehre2016knowledge} showed that prior information is effective for improving the accuracy of machine learning methods, and in many application scenarios in industry, economics, and biology, prior knowledge such as previous experimental or engineering data is available. Be it physical or historical, incorporating prior knowledge into statistical methods has attracted increasing attention in the statistical community \citep{leonard1978density, mccarthy2005profiting}, and the view is that incorporating prior information into robust regression can make the regression more robust and accurate.

The typical approach for integrating prior information is the Bayesian method, which provides a way of formalizing the process of learning from data to update beliefs in accordance with recent notions of knowledge synthesis \citep{congdon2007bayesian}. However, the generic Bayesian method is also sensitive to outliers, so a robust Bayesian method should be considered for more-reliable estimates in the presence of data corruption. \citet{polson2010shrink} used the local variance to assign each point a local parameter that makes the estimation result robust, and \citet{wang2018general} proposed a local parameterization method and used empirical Bayesian estimation to determine the global parameters. \citet{bhatia2019bayesian} proposed a Bayesian descent method using an unadjusted Langevin algorithm (ULA), which guarantees convergence in a finite number of steps, and \citet{wang2017robust} used Bayesian reweighting to assign different weights to samples to reduce the impact of outliers.

%%%%%%%%%%%%%%%%%%%%%%%%%%%%%%%%%%%%%%%%%%%%%%%%
\section{Problem Formulation}\label{problem formulation}

In this study, we focus on the problem of RLSR under AAAs. Given a data matrix $X=[\mathbf{x}_{1},\dots,\mathbf{x}_{n}]\in \mathbb{R}^{d\times n}$, the corresponding response vector $\mathbf{y}\in \mathbb{R}^{n}$, and a certain number $k$ of corruptions in the data, the RLSR problem can be described as
\begin{equation}
(\hat{\mathbf{w}},\hat S)=\arg\min_{\substack{\mathbf{w}\in\mathbb{R}^{p},S\subset[n]\\ |S|=n-k}}\sum_{i \in S}(y_{i}-\mathbf{x}_{i}^{T}\mathbf{w})^2. \label{problem}
\end{equation}
That is, we aim to recover the correct point set $S$ (which represents the set of corrupted data) and the regression coefficient $\mathbf{w}^*$ simultaneously to achieve the minimum regression error. However, this problem is NP hard \citep{Recovery_of_Sparsely_Corrupted_Signals} and so difficult to optimize directly.

Suppose that the true coefficient of the regression model is denoted by $\mathbf{w}^*$. As noted in Section~\ref{literature_review}, the response vector $\mathbf{y} \in \mathbb{R}^{n}$ is generated by
\begin{equation}
\mathbf{y}=X^{T}\mathbf{w}^{*}+\mathbf{b}^{*}+\boldsymbol{\epsilon}.\label{generation}
\end{equation}
The perturbations to the response vector have two parts: (i) the adversarial corruption vector introduced by $\mathbf{b}^*$, which is a $k$-sparse vector, and (ii) the dense observation error $\epsilon_{i} \sim \mathcal{N} (0,\sigma^2)$. Our goal is to recover the true regression coefficient $\mathbf{w}^{*}$ while simultaneously determining the corruption set $S$.

We begin by giving an explicit statistical interpretation of \eqref{problem}. Given the model in \eqref{generation} and assuming that $\sigma$ is known, the likelihood is immediately obtained by $\ell(\mathbf{w}, \mathbf{b}\mid y_{i},\mathbf{x}_i,\sigma^2) \propto \prod_{i=1}^n \exp\{-\frac{1}{2\sigma^2} (y_{i}-\mathbf{x}_{i}^{T}\mathbf{w}-b_i)^2\}$, and the maximum likelihood estimation (MLE) can be obtained by maximizing the likelihood. Note that for non-zero $b_i$, the corresponding component $\exp\{-\frac{1}{2\sigma^2} (y_{i}-\mathbf{x}_{i}^{T}\mathbf{w}-b_i)^2\}$ is maximized by setting $\hat{b}_i = y_{i}-\mathbf{x}_{i}^{T}\mathbf{w}$ for any given $\mathbf{w}$, and the maximum value is a constant, i.e., unity. Thus, the likelihood value is affected by only $S\subset[n]$, which indicates the zero components in $\mathbf{b}$. Treating $b_i$ as a nuisance parameter, the MLE problem is transformed into solving the following maximizing-profile log-likelihood (plug-in estimation for non-zero $b_i$) problem:
\begin{equation}
(\hat {\mathbf{w}},\hat S)=\arg\max_{\substack{\mathbf{w}\in\mathbb{R}^{p},S\subset[n]\\ |S|=n-k}}\sum_{i \in S}\log\ell(\mathbf{w}\mid y_{i},\mathbf{x}_i,\sigma^2), \label{likelihood}
\end{equation}
where $\log\ell(\mathbf{w}\mid y_{i},\mathbf{x}_i,\sigma^2) = (y_{i}-\mathbf{x}_{i}^{T}\mathbf{w})^2$. In this formulation, the MLE for $\mathbf{b}$ is plugged in to eliminate $\log\ell(\mathbf{w}\mid y_{i},\mathbf{x}_i,\sigma^2)$ where $b_i \neq 0$ in \eqref{likelihood}. In other words, \eqref{likelihood} can be interpreted as meaning that the likelihood of outliers has no effect on the likelihood if they are correctly detected. On the other hand, $\hat S$ represents the dataset after removing the detected outliers, and the likelihood can be maximized by detecting the outlier dataset and minimizing the mismatch between the linear model and the data in the non-outlier dataset.

Based on the above statistical interpretation, it is natural to convert \eqref{likelihood} into a Bayesian version. We assign a prior $p_{\mathbf{w}}(\mathbf{w})$ for the coefficients in the model; note that the posterior and likelihood multiplied by the prior differ by only one regularization constant, so we ignore this constant. We consider finding the subset $S$ and parameter $\mathbf{w}$ that maximize the unnormalized profile posterior:
\begin{equation}
(\hat {\mathbf{w}},\hat S)=\arg\max_{\substack{\mathbf{w}\in\mathbb{R}^{p},S\subset[n]\\ |S|=n-k}}\log p_{\mathbf{w}}(\mathbf{w})+\sum_{i \in S}\log\ell(\mathbf{w}\mid y_{i},\mathbf{x}_i,\sigma^2), \label{TRIP likelihood}
\end{equation}
which is the problem to be solved by the TRIP algorithm, as shown in Section~\ref{TRIP section}. From a different perspective, this can be treated as regularizing $\mathbf{w}$ by using the log-prior $\log(p_{\mathbf{w}}(\mathbf{w}))$. Intuitively, $\log(p_{\mathbf{w}}(\mathbf{w}))$ prevents the estimation of $\mathbf{w}$ from deviating too far from the true value. However, this regularization term still leads to an unavoidable bias in the final estimation, so a more effective method is needed to decrease this bias.

To solve this, we propose adding a localization parameter $\mathbf{r}$ that reflects the change introduced by each additional sample. First, we introduce the reweighted probabilistic model (RPM) proposed by \citet{wang2017robust} for traditional linear regression. Given the covariates $X$, the response $\mathbf{y}$, and priors $p_{\mathbf{r}}(\mathbf{r})$ and $p_{\mathbf{w}}(\mathbf{w})$, the RPM can be formulated as
\begin{equation}
p(\mathbf{y},\mathbf{w},\mathbf{r}| X)=\frac{1}{Z}p_{\mathbf{w}}(\mathbf{w})p_{\mathbf{r}}(\mathbf{r})\prod\limits_{i=1}^{n}\ell(\mathbf{w} |y_{i} ,\mathbf{x}_i,\sigma^2)^{r_{i}},\label{RPM-likelilood}
\end{equation}
where $\mathbf{r}$ is the local weight assigned to each sample, $Z$ is the normalizing constant, $\ell(\mathbf{w} |y_{i} ,\mathbf{x}_i,\sigma^2)$ represents the likelihood of the normal distribution $\mathcal{N} (\mathbf{x}_{i}^{T}\mathbf{w},\sigma^2)$, and $p_{\mathbf{w}}(\mathbf{w})$ and $p_{\mathbf{r}}(\mathbf{r})$ are the priors of $\mathbf{w}$ and $\mathbf{r}$, respectively. By adding this localization parameter $\mathbf{r}$ and ignoring the normalizing constant, the problem in \eqref{TRIP likelihood} can be transformed into the following form under this RPM setting: 
\begin{equation}
(\hat{\mathbf{w}},\hat S)=\arg\max_{\substack{\mathbf{w}\in\mathbb{R}^{p},\mathbf{r}\in\mathbb{R}_{+}^{n}\\S\subset[n],|S|=n-k}}
\log p_{\mathbf{w}}(\mathbf{w})+\sum\limits_{i \in S} [r_{i}\log\ell(\mathbf{w} |y_{i},\mathbf{x}_i,\sigma^2)+\log p_{\mathbf{r}}(r_{i})],\label{final problem}
\end{equation} 
where the prior $p_{\mathbf{r}}(r_i)$ can be set to any nonnegative random variable distribution.

The main purpose of adding the localization parameter $\mathbf{r}$ is to change $\log\ell(\mathbf{w} |y_{i} ,\mathbf{x}_i,\sigma^2)$ into $r_i\log\ell(\mathbf{w} |y_{i} ,\mathbf{x}_i,\sigma^2)$, which has a greater change amplitude. In the initial stage of optimization, the estimated $\mathbf{r}$ is relatively small because of the existence of outliers, which means that the regularization term $\log(p_{\mathbf{w}}(\mathbf{w}))$ dominates the estimation and so the search for estimation is not too far from the prior. When the estimation is close to the real parameter, each term $r_i$ in the chosen set $S$ is large, which means that the regularization term has low influence when calculating the estimation. To solve this problem, we propose the novel BRHT algorithm to solve \eqref{final problem}, which can also be regarded as integrating adaptive weights into robust regression.

Note that we treat $\sigma^2$ as a known fixed parameter whose value can be tuned to adjust the influence of the prior distribution. The main reason for doing this is that in the initial stage of the algorithm described in Section~\ref{method}, the estimated $\sigma^2$ is large because of the existence of outliers, and this makes the estimation of $\mathbf{w}$ excessively biased to the prior distribution. This bias is harmful, especially when the prior is not sufficiently accurate, and this phenomenon is shown in Section~\ref{TRIP section}.
\theoremstyle{plain} 
\newtheorem{defination}{Definition}

%%%%%%%%%%%%%%%%%%%%%%%%%%%%%%%%%%%%%%%%%%%%%%%% 
\section{Methodology}\label{method}

To solve \eqref{TRIP likelihood}, in Section~\ref{TRIP section} we propose the relatively simple TRIP algorithm, which demonstrates how adding a prior affects the hard-thresholding method. To improve the robustness and accuracy of the estimation, in Section~\ref{BRHT section} we propose the BRHT algorithm to solve \eqref{final problem}.

%%%%%%%%%%%%%%%%%%%%%
\subsection{TRIP: Hard Thresholding Approach to Robust Regression with Simple Prior}\label{TRIP section}

In this subsection, we propose the robust TRIP regression algorithm; see Algorithm~\ref{alg:TRIP}. Here, only the prior $p_{\mathbf{w}}(\mathbf{w})$ is considered, and the localization parameter $\mathbf{r}$ is not added to the model. We assume that the variance $\sigma^2$ of $\epsilon_i$ is known and fixed, and we assign a normal prior $\mathcal{N} (\mathbf{w}_0,\Sigma_0)$ to $p_{\mathbf{w}}(\mathbf{w})$, where $\mathbf{w}_0$ and $\Sigma_0$ are specified in advance. With the above simple parameter settings, \eqref{TRIP likelihood} is equivalent to
\begin{equation}
(\hat{\mathbf{w}},\hat S)=\arg\min_{\substack{\mathbf{w}\in\mathbb{R}^{p},S\subset[n]\\ |S|=n-k}}\sum_{i \in S}(y_{i}-x_{i}^{T}\mathbf{w})^2 +(\mathbf{w}-\mathbf{w}_0)^{T}M(\mathbf{w}-\mathbf{w}_0), \label{problem simple}
\end{equation}
where $M=(\Sigma_0/\sigma^2)^{-1}$. Inspired by the hard-thresholding method proposed by \citet{bhatia2017consistent}, which was focused on recovering the errors rather than selecting the ``cleanest'' set, we can reformulate \eqref{problem simple} as $\arg\min_{\mathbf{w}\in\mathbb{R}^{p},\|\mathbf{b}\|_{0}\le k^*} \|X^T\mathbf{w}-(\mathbf{y}-\mathbf{b})\|_2^2 +(\mathbf{w}-\mathbf{w}_0)^TM(\mathbf{w}-\mathbf{w}_0)$. Thus, if we have an estimation $\hat{\mathbf{b}}$ of the corruption vector $\mathbf{b}^*$, then the estimation of $\mathbf{w}^*$ can be obtained easily by $\hat{\mathbf{w}}=(XX^T+M)^{-1}[X(\mathbf{y}-\hat{\mathbf{b}})+M\mathbf{w}_0]$. By substituting this estimation into the optimization problem, we obtain a new formulation of the problem:
\begin{equation}
\min_{\|\mathbf{b}\|_{0}\le k^*}
\|(P_{MX}-I)(\mathbf{y}-\mathbf{b})+P_{MM}\mathbf{w}_0\|_2^2,\label{eq:obj_trip}
\end{equation}
where $P_{MX}=X^T(XX^T+M)^{-1}X$, $P_{MM}=X^T(XX^T+M)^{-1}M$. For the sake of simplicity, we denote the objective function of \eqref{eq:obj_trip} as $f(\mathbf{b})$, and the hard-thresholding step in the TRIP algorithm can be viewed as $\mathbf{b}^{t+1}=HT_k(\mathbf{b}^{t}-\nabla f(\mathbf{b}^{t}))$, where $k$ is the selected corruption coefficient. The hard-thresholding operator $HT_k$ is defined as follows.

\begin{defination}[Hard Thresholding]
For any vector $\mathbf{b} \in \mathbb{R}^n$, let $\delta_{\mathbf{b}}^{-1}(i)$ represent the index of the $i^{th}$ largest element in $\mathbf{b}$. Then, for any $k<n$, the hard-thresholding operator is defined as $\hat{\mathbf{b}}=HT_k(\mathbf{b})$, where $\hat{\mathbf{b}}_i=\mathbf{b}_i$ if $\delta_{\mathbf{b}}^{-1}(i)\le k$ or $0$ otherwise.
\end{defination}

\spacingset{1}
\begin{algorithm}[ht]
\caption{\textbf{TRIP}: hard \textbf{T}hresholding approach to \textbf{R}obust regression with s\textbf{I}mple \textbf{P}rior}
\label{alg:TRIP}
\begin{algorithmic}[1]
\REQUIRE Covariates $X=[\mathbf{x}_{1},\dots,\mathbf{x}_{n}]$, responses $\mathbf{y}=[y_{1},\dots,y_{n}]^T$,
prior knowledge $\mathbf{w}_0$, \\penalty matrix $M$,
corruption index $k$, tolerance $\epsilon$
\ENSURE solution $\hat{\mathbf{w}}$
\STATE $\mathbf{b}^0 \gets \mathbf{0}$, $t\gets 0$,\\
$P_{MX}\gets X^T(XX^T+M)^{-1}X$, $P_{MM}\gets X^T(XX^T+M)^{-1}M$
\WHILE {$\|\mathbf{b}^{t}-\mathbf{b}^{t-1}\|_{2}>\epsilon$}
\STATE $\mathbf{b}^{t+1}\gets HT_k(P_{MX} \mathbf{b}^{t}+(I-P_{MX})\mathbf{y}-P_{MM}\mathbf{w}_0)$
\STATE $t\gets t+1$;
\ENDWHILE
\RETURN $\hat{\mathbf{w}} \gets (XX^T)^{-1}X(\mathbf{y}-\mathbf{b}^{t})$
\end{algorithmic}
\end{algorithm}
\spacingset{2}

The difference between the proposed TRIP algorithm and the original CRR (consistent robust regression) algorithm \citep{bhatia2017consistent} is the form of the iteration step, which in both TRIP and CRR is expressed uniformly as $HT_k(\mathbf{y}-X^T\mathbf{w}^t)$, but with $\mathbf{w}^t=(XX^T+M)^{-1}[X(\mathbf{y}-\mathbf{b}^t)+M\mathbf{w}_0]$ in TRIP and $\mathbf{w}^t=(XX^T)^{-1}X(\mathbf{y}-\mathbf{b}^t)$ in CRR. In CRR, $\mathbf{w}^t$ is just a simple least-squares estimation, whereas the prior added in TRIP can be regarded as adding a quadratic regularization in each iteration. This quadratic regularization avoids any iteration candidate that is too far from the prior mean, which also helps to ensure a numerically stable solution. Therefore, as long as the prior is not too misspecified, TRIP is more likely to identify the uncorrupted points, and its final result is more robust than that of CRR.

Therefore, the prior plays an important role in TRIP, but its weight in the solution depends entirely on the matrix $M=(\Sigma_0/\sigma^2)^{-1}$. Therefore, if we either use an estimation of $\hat{\sigma}^2$ to replace $\sigma^2$ or give $\sigma^2$ a prior to calculate its posterior distribution, then the overestimation of $\sigma^2$ causes a severe increase in $M$ in the initial iteration steps, and this misleads the iteration and causes some deviation in the final results. To overcome this difficulty, we can treat $M$ directly as an adjustable parameter to control by specifying a form such as $M=sI$, where $s$ is a positive number, and the suitable parameter can be chosen through five-fold or ten-fold cross validation.

%%%%%%%%%%%%%%%%%%%%%
\subsection{BRHT: Robust Bayesian Reweighting Regression via Hard Thresholding}\label{BRHT section}

In this subsection, we describe how Bayesian reweighting is combined with hard thresholding to give the more robust BRHT algorithm; see Algorithm~\ref{alg:BRHT}. As shown in Section~\ref{problem formulation}, BRHT is designed to solve \eqref{final problem}. The specific form of the prior $p_{\mathbf{r}}(r_i)$ can be set to any nonnegative random variable distribution, including (but not limited to) the gamma, beta, or log-normal distribution. Here, we still use the normal distribution $\mathcal{N} (\mathbf{w}_0,\Sigma_0)$ as the form of $p_{\mathbf{w}}(\mathbf{w})$.

To solve the optimization problem in \eqref{final problem}, we use the two-step BRHT algorithm. Similar to TRIP, BRHT also estimate $\mathbf{w}$ and $\mathbf{b}$ alternately, that is, fixing one to estimate the other. The estimation of $\mathbf{b}$ in BRHT is still $\mathbf{b}^{t+1}\gets HT_k(\mathbf{y}-X^T\mathbf{w}^t)$, while $\mathbf{w}^t$ is calculated by maximizing the log-posterior of the RPM model: 
\begin{equation}
(\mathbf{w}^t,\mathbf{r}^t)=\arg\max_{\mathbf{w}\in \mathbb{R}^{d},\mathbf{r}\in \mathbb{R}_{+}^{n}}\log p_{\mathbf{w}}(\mathbf{w})+\log p_{\mathbf{r}}(\mathbf{r})+\sum\limits_{i=1}^{n} r_{i}\log\ell(y_{i}-b_i^t\mid \mathbf{w},\mathbf{x}_i,\sigma^2).\label{log-likelood}
\end{equation}
However, direct inference of \eqref{log-likelood} is difficult because of the nonconvexity of this problem. In general, the parameters in this RPM model can be divided into two parts: the global variable $\mathbf{w}$ and the local latent variable $\mathbf{r}$. 

\spacingset{1}
\begin{algorithm}[ht]
\caption{\textbf{BRHT}: robust \textbf{B}ayesian \textbf{R}eweighting regression via \textbf{H}ard \textbf{T}hresholding}
\label{alg:BRHT}
\begin{algorithmic}[1]
\REQUIRE Covariates $X=[\mathbf{x}_{1},\dots,\mathbf{x}_{n}]$, responses $\mathbf{y}=[y_{1},\dots,y_{n}]^T$, prior distribution $p_{\mathbf{r}}(\mathbf{r})$, $p_{\mathbf{w}}(\mathbf{w})$, corruption index $k$, tolerance $\epsilon$
\ENSURE solution $\hat{\mathbf{w}}$
\STATE $\mathbf{b}^0 \gets \mathbf{0}$, $t\gets 0$,\\
\WHILE {$\|\mathbf{b}^{t}-\mathbf{b}^{t-1}\|_{2}>\epsilon$}
\STATE $\mathbf{w}^t\gets \text{VBEM}(X,\mathbf{y}-\mathbf{b}^{t},p_{\mathbf{r}}(\mathbf{r}),p_{\mathbf{w}}(\mathbf{w}))$
\STATE $\mathbf{b}^{t+1}\gets HT_k(\mathbf{y}-X^T\mathbf{w}^t)$
\STATE $t\gets t+1$;
\ENDWHILE
\RETURN $\hat{\mathbf{w}} \gets (XX^T)^{-1}X(\mathbf{y}-\mathbf{b}^{t})$
\end{algorithmic}
\end{algorithm}

\begin{algorithm}[ht]
\caption{\textbf{VBEM}: \textbf{V}ariational \textbf{B}ayesian \textbf{E}xpectation \textbf{M}aximization}
\label{alg:VBEM}
\begin{algorithmic}[1]
\REQUIRE Covariates $X=[\mathbf{x}_{1},\dots,\mathbf{x}_{n}]$, responses $\mathbf{y}=[y_{1},\dots,y_{n}]^T$,
prior distribution $p_{\mathbf{r}}(\mathbf{r})$, $p_{\mathbf{w}}(\mathbf{w})$ 
\ENSURE solution $\hat{\mathbf{w}}$
\REPEAT
\STATE update $q(\mathbf{r})$
\STATE update $q(\mathbf{w})$
\UNTIL{convergence}
\RETURN $\hat{\mathbf{w}} \gets$ MAP($q(\mathbf{w})$)
\end{algorithmic}
\end{algorithm}
\spacingset{2}

To solve the inference problem of RPM, a feasible method is variational Bayesian expectation maximization (VBEM). We set $q(\mathbf{w},\mathbf{r})=q(\mathbf{w})q(\mathbf{r})$ to approximate the true posterior after particular iterations of VBEM (Algorithm~\ref{alg:VBEM}), and we replace the estimation of $\mathbf{w}$ by the maximum a~posteriori (MAP) estimation from the final approximate posterior. Full details about the VBEM method are given in Appendix~A.

It is reasonable to ask why we use Bayesian reweighting, other than the advantage mentioned in Section~\ref{problem formulation}. Note that the iteration step in the TRIP algorithm is $\mathbf{b}^{t+1}\gets HT_k(P_{MX} \mathbf{b}^{t}+(I-P_{MX}) \mathbf{y}-P_{MM} \mathbf{w}_0)=HT_k(\mathbf{y}-X^T\mathbf{w}^t)$, where $\mathbf{w}^t=(XX^T+M)^{-1}[X(\mathbf{y}-\mathbf{b}^t)+M\mathbf{w}_0]$. Although we can show that TRIP already guarantees theoretical convergence, the estimation of $\mathbf{w}^*$ in every iteration still uses least-squares with a penalty term, which is easily affected by the corrupted points. This disadvantage forces us to assign a higher weight to the prior to resist severe data corruption in the case of AAAs. However, a higher weight on the prior means a larger estimation bias. When applying Bayesian reweighting, the estimation in each step is more robust than the least-squares result, and thus the weight on the prior can be reduced to guide the iteration. Therefore, the estimation bias is relatively small and the results are more robust. This is reflected in the experimental results presented in Section~\ref{experiment}.

We should also explain why we only use a prior for a few parameters. It is important to ensure that the prior weights are neither too high nor too low. As mentioned earlier, if we treat $\sigma^2$ as the parameter to be estimated, then this places too much weight on the prior. We also ensure that $p_{\mathbf{w}}(\mathbf{w})$ does not create more uncertainty, such as by setting $p_{\mathbf{w}}(\mathbf{w})$ to $p(\mathbf{w},u)=\mathcal{N} (\mathbf{w}_0,u^{-1}\Sigma_0)Gam(u|a_{u},b_{u})$. Data corruption means that the subset of the training data may vary greatly from the prior information. Thus, when calculating the posterior, the variance of $\mathbf{w}$ controlled by $u$ is very large to fit the data, and so the prior information $\mathbf{w}_0$ has lower weights in the inference step. The above problems also mislead the estimation and the selection of subsets, so we do not consider the uncertainty of these quantities and simply treat them as model parameters to be set in advance. An adjustment method for all the parameters in BRHT is described in Appendix~D.

\subsection{Data-driven Priors}\label{choose_prior}

Although prior information is common in daily application scenarios, there are also cases that lack prior information or in which the prior does not fit the actual data. Therefore, in this subsection we expand the original algorithm to explore how to apply our methods in those cases. As stated in Section~\ref{literature_review}, plenty of robust regression methods have been proposed in past decades, and although their results may not be accurate enough under complex noise, their main advantage is robustness. Therefore, the solutions given by other robust regression algorithms can be used for generating prior distributions. Intuitively, the main purpose of using an informative prior distribution herein is to provide regularization for the objective function \eqref{likelihood} so that the prior mean is less sensitive to the estimation error. In particular, the prior mean can be a plug-in value of parameter estimation from state-of-the-art robust regression methods. The other hyperparameter $\Sigma$ in prior $p_{\mathbf{w}}(\mathbf{w})$ can be selected by cross-validation by giving $\Sigma$ the specific form $\Sigma = sI$. This method is similar to the empirical Bayesian method \citep{maritz2018empirical}, where the hyperparameters of the prior model are estimated from data using a non-Bayesian method.

In Section~\ref{experiment}, we use LAD regression \citep{pollard1991asymptotics} to extract prior information from data, and we show that the choice of this prior can significantly reduce the estimation error.

\theoremstyle{plain} 
\newtheorem{theorem}{Theorem}

%%%%%%%%%%%%%%%%%%%%%%%%%%%%%%%%%%%%%%%%%%%%%%% 
\section{Theoretical Convergence Analysis}\label{theoretical analysis}

In this section, we establish a convergence theory for the TRIP algorithm, and we explain clearly how the prior effectively enhances the convergence of the RLSR model; Theorems~1 and 2 summarize the results. Also, Theorems~3--5 provide a theoretical guarantee for the BRHT algorithm, demonstrating further the special properties achieved by using Bayesian reweighting.

Before presenting the convergence result, we introduce some notation. Let $\mathbf{\lambda}^t:=(XX^T+M)^{-1}X(\mathbf{b}^{t}-\mathbf{b}^{*})$, $\mathbf{g}:=(I-P_{MX})\mathbf{\epsilon}$, and $\mathbf{f}:=P_{MM}(\mathbf{w}^*-\mathbf{w}_0)$. Let $S_t:=[n]\backslash \text{supp}(\mathbf{b}^t)$ be the chosen subset that is considered to be uncorrupted, and let $I_{t}:= \text{supp}(\mathbf{b}^t)\cup \text{supp}(\mathbf{b}^*)$, where $\text{supp}(\mathbf{b})=\{i,b_i\neq 0, i=1,\dots,n\}$ is the support of vector $\mathbf{b}$.

To prove the positive effect of the prior on the RLSR problem, we require the properties of \textit{subset strong convexity (SSC)} and \textit{subset strong smoothness (SSS)}. Given a set $S\subset [n]$, $X_{S}:=[\mathbf{x}_{i \in S}]\in \mathbb{R}^{d\times |S|}$ signifies the matrix with columns in $S$, and the smallest and largest eigenvalues of a symmetric matrix $A$ are denoted by $\lambda_{\text{min}}(A)$ and $\lambda_{\text{max}}(A)$, respectively. Proposed by \citet{bhatia2015robust}, these two properties are intended to standardize the generation of the data matrix so that it is not too abnormal.

\begin{defination}[SSC Property]
A matrix $X \in \mathbb{R}^{d\times n}$ is said to satisfy the SSC property at level $m$ with constant $\lambda_{m}$ if $\lambda_{m}\le \min_{|S|=m}\lambda_{\text{min}}(X_{S}X_{S}^{T}).$
\end{defination}

\begin{defination}[SSS Property]
A matrix $X \in \mathbb{R}^{d\times n}$ is said to satisfy the SSS property at level $m$ with constant $\Lambda_{m}$ if $\max_{|S|=m}\lambda_{\text{max}}(X_{S}X_{S}^{T}) \le \Lambda_{m}.$
\end{defination}

\begin{theorem}\label{theorem1}
Let $X=[\mathbf{x}_1,\dots,\mathbf{x}_n]\in \mathbb{R}^{d\times n}$ be the given data matrix and $\mathbf{y}=X^{T}\mathbf{w}^{*}+\mathbf{b}^{*}+\boldsymbol{\epsilon}$ be the corrupted output with sparse corruption of $\|\mathbf{b}^*\|_0\le k^{*}$. For a specific positive semidefinite matrix $M$, $X$ satisfies the SSC and SSS properties such that $2\frac{\Lambda_{k+k^{*}}}{\lambda_{min}(XX^T+M)}<1$. Then, if $k>k^{*}$, it is guaranteed with a probability of at least $1-\delta$ that for any $\varepsilon,\delta>0$, $\|\mathbf{b}^{T_0}-\mathbf{b}^*\|_{2}\le \varepsilon + \textit{O}(e_0)+\textit{O}(\frac{\sqrt{\Lambda_{k+k^*}}\lambda_{max}(M)}{\lambda_{min}(XX^T+M)})\|\mathbf{w}^*-\mathbf{w}_0\|_2$ after $T_{0}=\textit{O}(\log (\frac{\|\mathbf{b}^*\|_2}{\varepsilon}))$ iterations of TRIP, where $e_0=\textit{O}(\sigma\sqrt{(k+k^*)\log\frac{n}{\delta(k+k^*)}})$ under the normal design. 
\end{theorem}

\begin{theorem}\label{theorem2}
Under the conditions of Theorem~1 and assuming that $\mathbf{x}_i\in \mathbb{R}^d$ are generated from the standard normal distribution, if $k>k^{*}$, it is guaranteed with a probability of at least $1-\delta$ that for any $\varepsilon,\delta>0$, the current estimation coefficient $\mathbf{w}_{T_{0}}$ satisfies $\|\mathbf{w}_{T_{0}}-\mathbf{w}^*\|_2\le \textit{O}(\frac{1}{\sqrt{n}})(\varepsilon+e_0)+\textit{O}(\frac{\sqrt{k+k^*}\lambda_{\max}(M)}{n^{3/2}})\|\mathbf{w}^*-\mathbf{w}_0\|_2$ after $T_{0}=\textit{O}(\log (\frac{\|\mathbf{b}^*\|_2}{\varepsilon}))$ steps.
\end{theorem}

For positive semidefinite matrices $XX^T$ and $M$, we have $\lambda_{min}(XX^T+M)\ge \lambda_{min}(XX^T)+\lambda_{min}(M)$. Thus, the condition $2\frac{\Lambda_{k+k^{*}}}{\lambda_{min}(XX^T+M)}<1$ in Theorem~\ref{theorem1} is weaker than the condition $2\frac{\Lambda_{k+k^{*}}}{\lambda_{min}(XX^T)}<1$ of Lemma~5 in \citet{bhatia2017consistent}, which shows that a prior can effectively improve the convergence of the algorithm. Assigning a higher weight to a prior means that $M$ has larger eigenvalues, and so the convergence condition is satisfied more easily. Consequently, the TRIP algorithm can tolerate a higher proportion of outliers than can the CRR method of \citet{bhatia2017consistent}, and it achieves a higher breakdown point. In fact, under the condition $\lim_{n\to \infty}\frac{\lambda_{min}(M)}{n}=\xi$, we can give an approximate expression for the breakdown point of TRIP when $\xi$ is not too large, i.e., $k^{*}\le k\le (0.3023-\sqrt{0.0887-0.0040\xi})n$; see Appendix~C.2 for details. However, the improved convergence comes at the cost of probably more bias, as can be seen from Theorem~\ref{theorem2}. If the data corruption is such that $k^*$ is $\textit{O}(n)$ and the maximum eigenvalue of $M$ is also $\textit{O}(n)$, then the bias of $\hat{\mathbf{w}}$ cannot be decreased by adding more samples, which shows that there is a trade-off between convergence and accuracy. A reliable prior improves both accuracy and convergence by being given a higher weight, but an inaccurate prior can also be helpful in practice as long as it is quite different from the distribution of outliers.

To prove the properties of our BRHT algorithm, we define the following two intermediate variables to simplify the description: 
\begin{gather}
U(\mathbf{w},\mathbf{r},S)=\log p_{\mathbf{w}}(\mathbf{w})+\sum_{i \in S}[\log p_{\mathbf{r}}(r_i)+r_i \log\ell(y_{i}\mid \mathbf{w},\mathbf{x}_i,\sigma^2)],\\
M(\mathbf{w},\mathbf{r},\mathbf{b})=\log p_{\mathbf{w}}(\mathbf{w})+\sum_{i }[\log p_{\mathbf{r}}(r_i)+r_i \log\ell(y_{i}-b_i\mid \mathbf{w},\mathbf{x}_i,\sigma^2)].
\end{gather}

\setcounter{theorem}{4}
\renewcommand{\thetheorem}{\arabic{theorem}}
\begin{theorem}\label{theorem3}
Suppose that the prior of $r_i$ is independently and identically distributed (i.i.d.). We consider $\mathbf{w}_t,\mathbf{r}_t=\arg \max_{\mathbf{w}\in \mathbb{R}^{d},\mathbf{r}\in \mathbb{R}_{+}^{n}}M(\mathbf{w},\mathbf{r},\mathbf{b}_{t})$, which is the estimation in the $t^{th}$ iteration step of the BRHT algorithm, and $\mathbf{b}_t= HT_k(\mathbf{y}-X^T\mathbf{w}_{t-1})$ is obtained from the hard-thresholding step. Then, we have that $U(\mathbf{w}_t,\mathbf{r}_t,S_{t+1})\ge U(\mathbf{w}_{t-1},\mathbf{r}_{t-1},S_{t})$.
\end{theorem}

\begin{theorem}\label{theorem4}
Consider a data matrix $X$ and a specific positive semidefinite matrix $M$ satisfying the SSC and SSS properties such that $2\frac{\Lambda_{k+k^{*}}}{\lambda_{min}(XX^T+M)}<1$. Then, there exist $l>0$ and $0<\gamma\le1+\epsilon$, where $\epsilon$ is a small number, such that if $k>k^{*}$ and $\Sigma$ in the prior $p_{\mathbf{w}}(\mathbf{w})$ is $l \sigma^2 M^{-1}$, it is guaranteed with a probability of at least $1-\delta$ that for any $\varepsilon,\delta>0$, $\|\mathbf{b}^{T_0}-\mathbf{b}^*\|_{2}\le \varepsilon + \textit{O}(e_0)+\textit{O}(\frac{\sqrt{\Lambda_{k+k^*}}\lambda_{max}(M)}{\lambda_{min}(XX^T+M)})\gamma\|\mathbf{w}^*-\mathbf{w}_0\|_2$ after $T_{0}=\textit{O}(\log (\frac{\gamma\|\mathbf{b}^*\|_2}{\varepsilon}))$ iterations of BRHT, where $e_0=\textit{O}(\sigma\sqrt{(k+k^*)\log\frac{n}{\delta(k+k^*)}})$ under the normal design. 
\end{theorem}

\begin{theorem}\label{theorem5}
Under the conditions of Theorem~6 and with $\mathbf{x}_i\in \mathbb{R}^d$ generated from the standard normal distribution, there exist $l>0$ and $0<\gamma\le1+\epsilon$, where $\epsilon$ is a small number, such that if $k>k^{*}$ and $\Sigma$ in the prior $p_{\mathbf{w}}(\mathbf{w})$ is $l\sigma^2 M^{-1}$, it can be guaranteed with a probability of at least $1-\delta$ that for any $\varepsilon,\delta>0$, the current estimation coefficient $\mathbf{w}_{T_{0}}$ satisfies 
$\|\mathbf{w}_{T_{0}}-\mathbf{w}^*\|_2\le \textit{O}(\frac{1}{\sqrt{n}})(\varepsilon+e_0)+\textit{O}(\frac{\sqrt{k+k^*}\lambda_{\max}(M)}{n^{3/2}})\gamma\|\mathbf{w}^*-\mathbf{w}_0\|_2$ after $T_{0}=\textit{O}(\log (\frac{\gamma\|\mathbf{b}^*\|_2}{\varepsilon}))$ steps.
\end{theorem}

Theorem~\ref{theorem3} shows that the BRHT algorithm is valid because the objective function of \eqref{final problem} increases in each iteration step. Theorems~\ref{theorem4} and \ref{theorem5} guarantee the convergence of the parameter. In particular, if we introduce a prior $p_{\mathbf{w}}(\mathbf{w})= \mathcal{N} (\mathbf{w}_0,\Sigma_0)$ to the TRIP algorithm such that it is convergent, then there exists $t>0$ such that the prior $p_{\mathbf{w}}(\mathbf{w})= \mathcal{N} (\mathbf{w}_0,t\Sigma_0)$ in the BRHT algorithm guarantees convergence, and the bias of the estimation of $\mathbf{w}^*$ shrinks by a factor $\gamma$ compare with that of TRIP. Note that $t$ is usually large in practice, which causes a less informative prior in BRHT. The bias of the estimator $\hat{\mathbf{w}}$ can be reduced significantly by selecting a less informative prior, so even if the data are corrupted, the BRHT algorithm ensures convergence without sacrifice regarding bias. The proofs of all these theorems are given in Appendix~C.

\section{Numerical Studies}\label{experiment}

In this section, we begin by considering how to effectively ``corrupt'' the dataset using two different attacks, i.e., OAA and AAA, then we report an extensive experimental evaluation to verify the robustness of the proposed methods. 

\subsection{Data, Prior, and Evaluation Metrics}\label{Data, Prior and Metrics}

In our experiments, the data generation can be divided into two steps. First, we generate the basic model: the true coefficient $\mathbf{w}^*$ is chosen to be a random unit norm vector, the covariants $\mathbf{x}_i$ are i.i.d.\ in $\mathcal{N} (0,I_{{d}})$, the data are generated by $y_{i}=\mathbf{x}_{i}^{T}\mathbf{w}^*+\epsilon_{i}$, where $\epsilon_{i}$ are i.i.d.\ in $\mathcal{N} (0,1^2)$ in the experiments. Second, we generate the corrupted data using three types of attack, i.e., OAA/AAA and a leverage point attack (LPA), as described in Section~\ref{Corruption Method}, the aim being to produce $k^*$ corrupted responses in the whole dataset.

Suppose that prior information is available, with the prior mean being a rough estimation of the parameters from historical data or engineering knowledge. To this end, we generate prior coefficient $\mathbf{w}_0$ by $\mathbf{w}^*+\nu\mathbf{u}$, where $\mathbf{u}$ is sampled from $\mathcal{N} (0,I)$ and $\nu$ is a nonnegative number ($\nu$ is set to 0.5 unless stated otherwise). In other words, the prior mean is chosen as an initial guess that might have significant bias. $\Sigma_0$ takes the form $sI$, where $s$ takes different values for each method. All parameters are fixed in each experiment.

We measure the performance of the regression coefficients by the standard $L_2$ error, i.e., $r_{\hat{\mathbf{w}}}=\|\hat{\mathbf{w}}-\mathbf{w}^*\|_2$. To diagnose whether the algorithm converges, we used the termination criterion $\|\mathbf{w}^{t+1}-\mathbf{w}^t\|_2\le 10^{-4}$. All results were averaged over 10 runs.

\subsection{Corruption Methods}\label{Corruption Method} 

To demonstrate the efficiency of the proposed methods, we apply two different attacks to the synthetic dataset, i.e., OAA and AAA, the details of which are as follows.

\noindent\textbf{OAA}: The set of corrupted points $S$ is selected as a uniformly random $k$-sized subset of $[n]$. The corresponding response variables are corrupted by a uniform distributed attack, i.e., the corrupted data are $y_i=\mathbf{x}_{i}^{T}\mathbf{w}^{*}+b_i+\epsilon_i$, where $b_i$ are sampled from the uniform distribution $U[0,10]$ and the white noise $\epsilon_i\sim\mathcal{N} (0,1^2)$.

\noindent\textbf{AAA}: We propose an adaptive data corruption algorithm (ADCA) that corrupts the data using all information from the true data distribution. This algorithm is quite similar to TRIP, and full details of the ADCA (Algorithm~\ref{alg:ADCA}) are given in Appendix~B. $\delta$ is set to $0.1n$ for $n=1000, p=200$ and to $0.2n$ for $n=2000, p=100$.

\spacingset{1}
\begin{algorithm}[ht]
\caption{\textbf{ADCA}: \textbf{A}daptive \textbf{D}ata \textbf{C}orruption \textbf{A}lgorithm}
\label{alg:ADCA}
\begin{algorithmic}[1]
\REQUIRE Covariates $X=[\mathbf{x}_{1},\dots,\mathbf{x}_{n}]$, responses $\mathbf{y}=[y_{1},\dots,y_{n}]^T$, true parameter $\mathbf{w}^*$, penalty coefficient $\delta$,
corruption index $k$, tolerance $\epsilon$
\ENSURE solution $\hat{\mathbf{w}}$
\STATE $\mathbf{b}^0 \gets \mathbf{0}$, $t\gets 0$,\\
$P_{\delta X}\gets X^T(XX^T-\delta I)^{-1}X$, $P_{\delta}\gets X^T(XX^T-\delta I)^{-1}\delta I$
\WHILE {$\|\mathbf{b}^{t}-\mathbf{b}^{t-1}\|_{2}>\epsilon$}
\STATE $\mathbf{b}^{t+1}\gets HT_k(P_{\delta X} \mathbf{b}^{t}+(I-P_{\delta X})\mathbf{y}+P_{\delta}\mathbf{w}^*)$
\STATE $t\gets t+1$;
\ENDWHILE
\STATE$\hat{\mathbf{w}} \gets (XX^T)^{-1}X(\mathbf{y}-\mathbf{b}^t)$
\STATE $C\gets supp(\mathbf{b}^{t})$
\RETURN $\mathbf{y}_{C}=X_{C}^{T}\hat{\mathbf{w}}$
\end{algorithmic}
\end{algorithm}
\spacingset{2}

In the OAA case, we employ the prior $p_{\mathbf{w}}(\mathbf{w}) = N(\mathbf{w}_0, \frac{1}{0.05n}I)$ for TRIP and $p_{\mathbf{w}}(\mathbf{w}) = N(\mathbf{w}_0, \frac{1}{0.01n}I)$ for BRHT. In the AAA case, the prior is $p_{\mathbf{w}}(\mathbf{w}) = N(\mathbf{w}_0, \frac{1}{0.2n}I)$ for TRIP and $p_{\mathbf{w}}(\mathbf{w}) = N(\mathbf{w}_0, \frac{1}{0.04n}I)$ for BRHT. This means that the prior has a higher weight in AAA, and the prior information of both TRIP and BRHT is enhanced by a factor of four. The other parameters for these two algorithms are fixed. The prior distribution $p_{\mathbf{r}}(\mathbf{r})$ is set to the gamma distribution $Ga(4,10)$ unless stated otherwise.
We also design an LPA to test the robustness of our methods; see Appendix~E for more results under this attack.

\subsection{Benchmark Methods}

Our methods are compared with three baselines: (i) CRR \citep{bhatia2017consistent} is an effective robust regression method when there are many random outliers; (ii) reweighted robust Bayesian regression (RRBR) \citep{wang2017robust} allows us to judge whether our proposed methods are better than the original method; (iii) Rob-ULA \citep{bhatia2019bayesian} is an effective robust Bayesian inference method that converges approximately to the real posterior distribution in a finite number of steps in the presence of outliers.

The parameters of RRBR are the same as those of the BRHT algorithm, except for the hard-thresholding part. In Appendix~E, to test the regression effect in the noiseless model, we compare our methods with TORRENT \citep{bhatia2015robust} method.

\begin{figure*}[tbp]
\centering
\subfigure[]
{
\begin{minipage}{3.6cm}
\centering
\includegraphics[width=3.6cm]{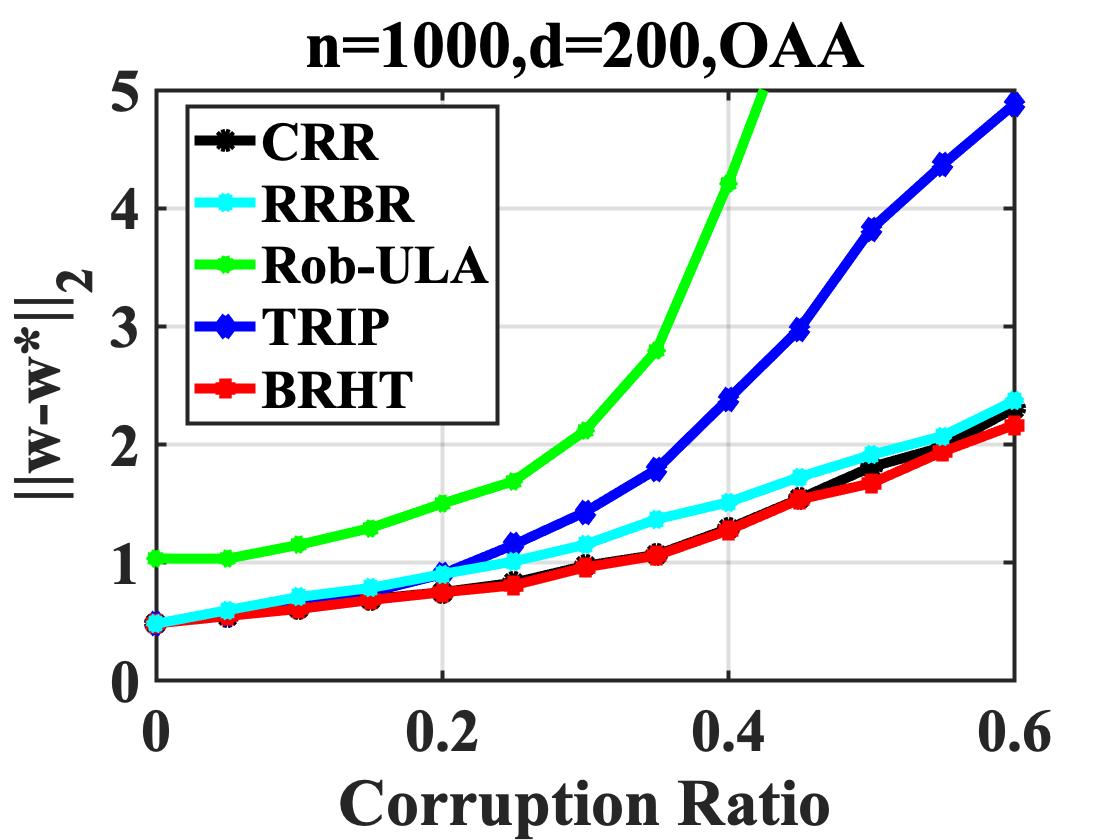}
\end{minipage}
}
\subfigure[]
{
\begin{minipage}{3.6cm}
\centering
\includegraphics[width=3.6cm]{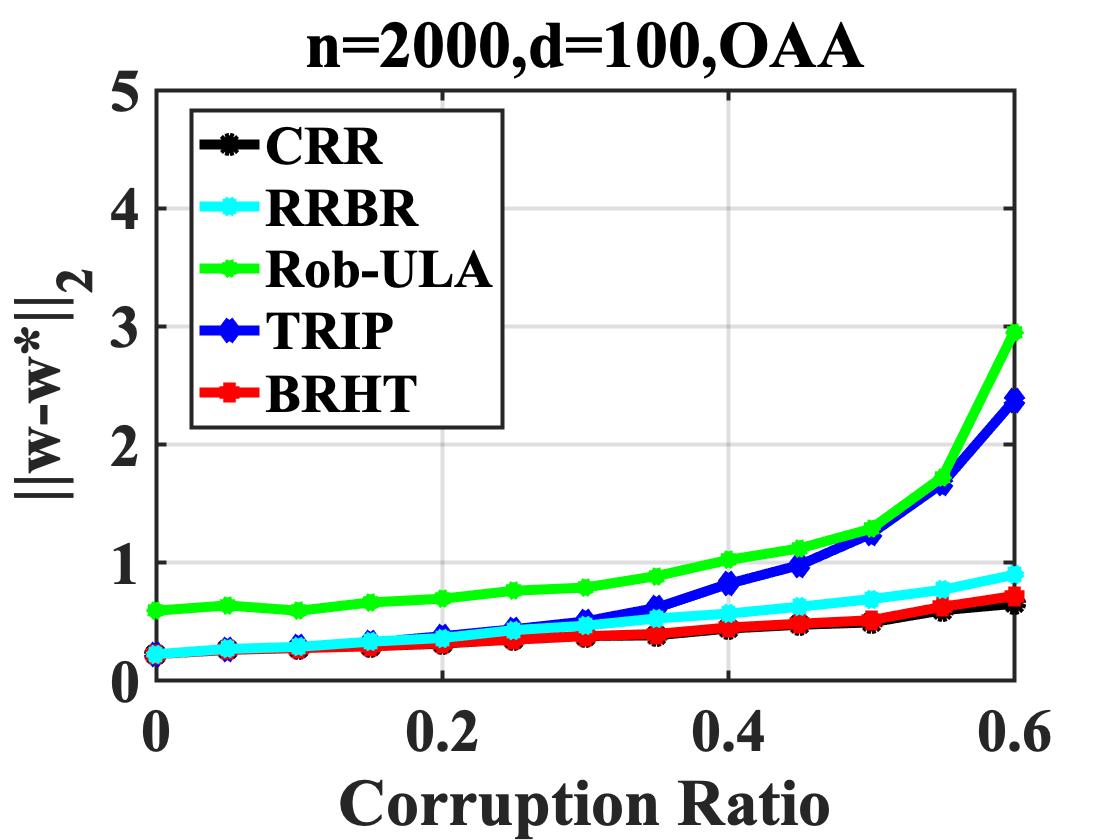}
\end{minipage}
}
\subfigure[]
{
\begin{minipage}{3.6cm}
\centering
\includegraphics[width=3.6cm]{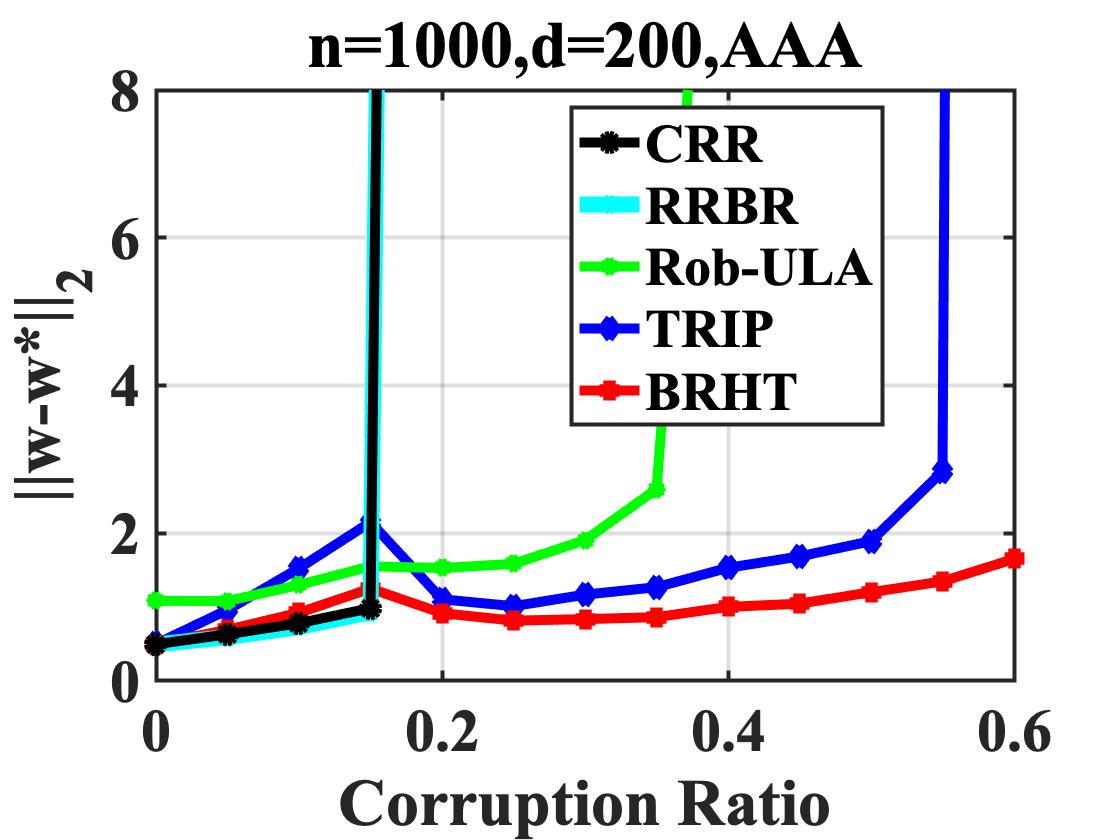}
\end{minipage}
}
\subfigure[]
{
\begin{minipage}{3.6cm}
\centering
\includegraphics[width=3.6cm]{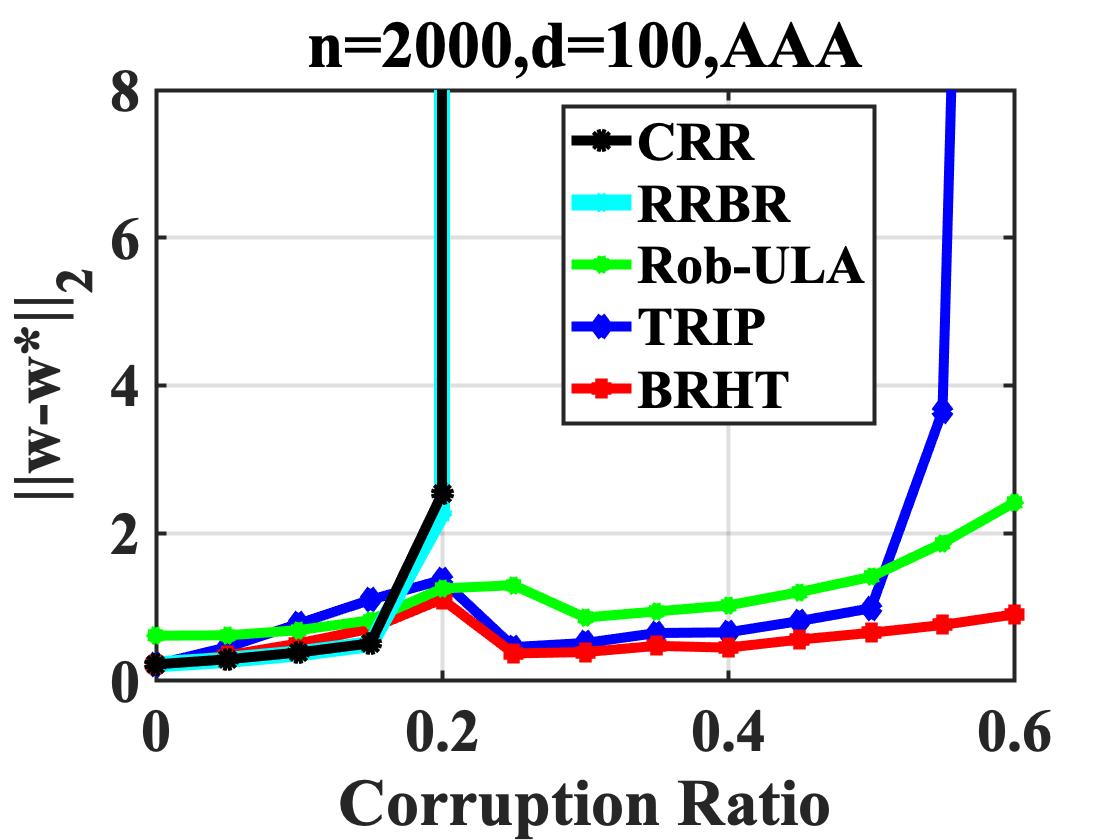}
\end{minipage}
}
\caption{Recovery of parameters with respect to number of data points $n$, dimensionality $d$, and corruption ratio $\alpha$.} 
\label{figure1}
\end{figure*}

\begin{figure*}[tbp]
\addtocounter{subfigure}{0}
\centering
\subfigure[]
{
\begin{minipage}{3.6cm}
\centering
\includegraphics[width=3.6cm]{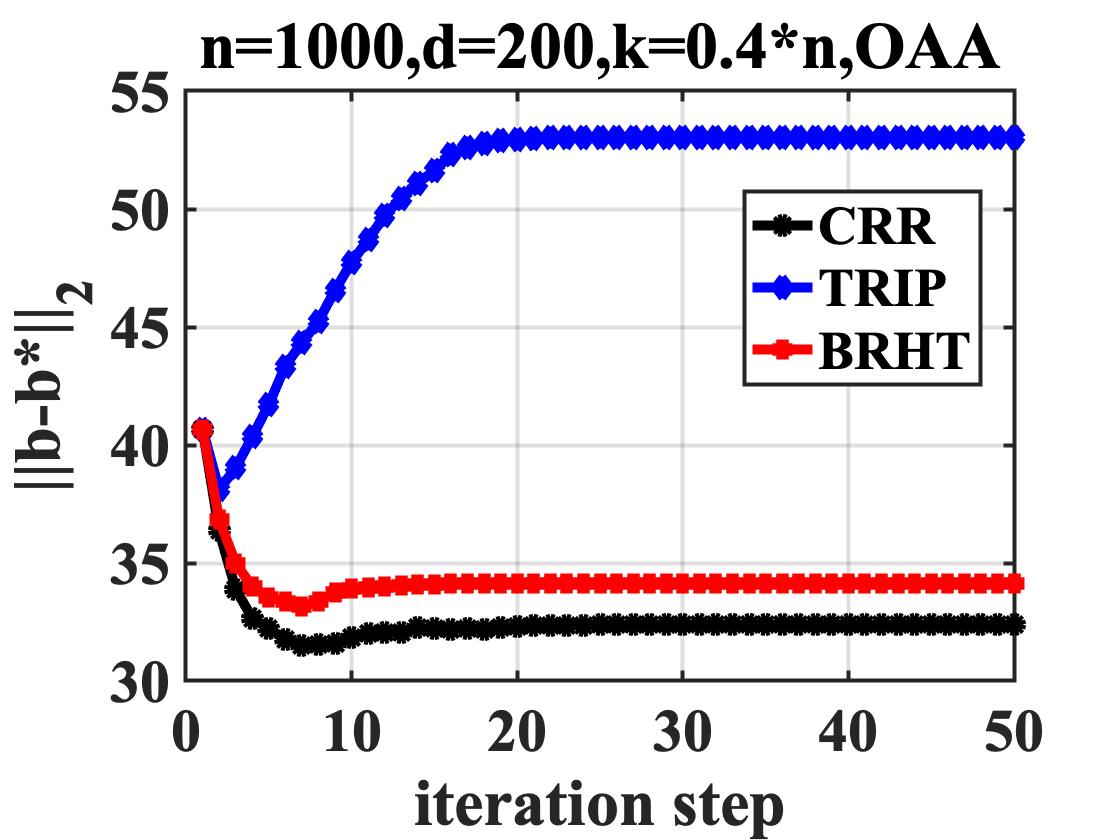}
\end{minipage}
}
\subfigure[]
{
\begin{minipage}{3.6cm}
\centering
\includegraphics[width=3.6cm]{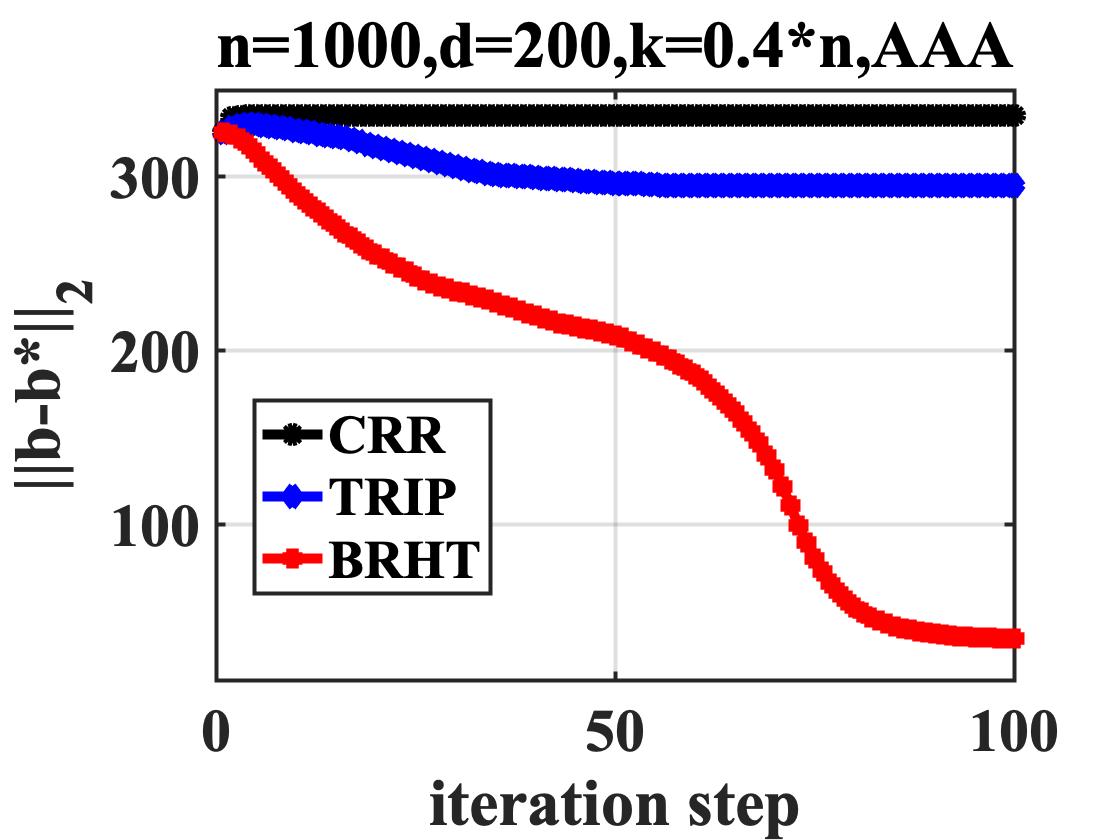}
\end{minipage}
}
\subfigure[]
{
\begin{minipage}{3.6cm}
\centering
\includegraphics[width=3.6cm]{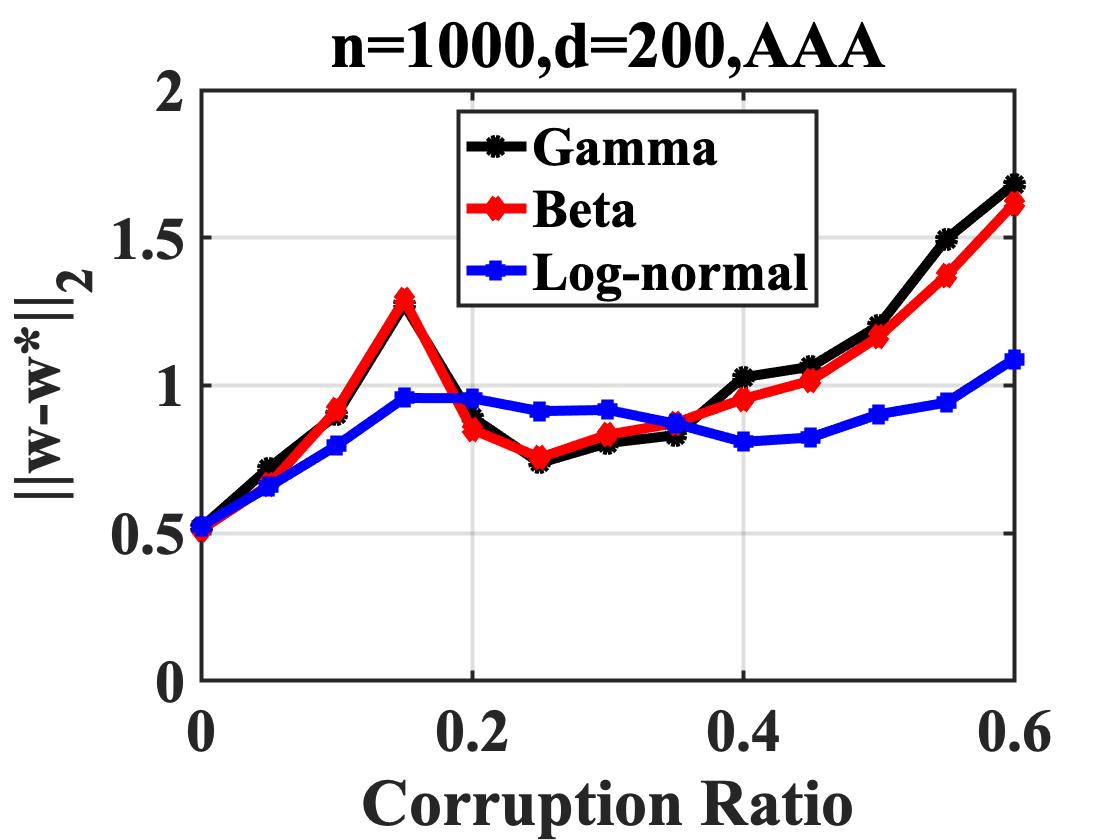}
\end{minipage}
}
\subfigure[]
{
\begin{minipage}{3.6cm}
\centering
\includegraphics[width=3.6cm]{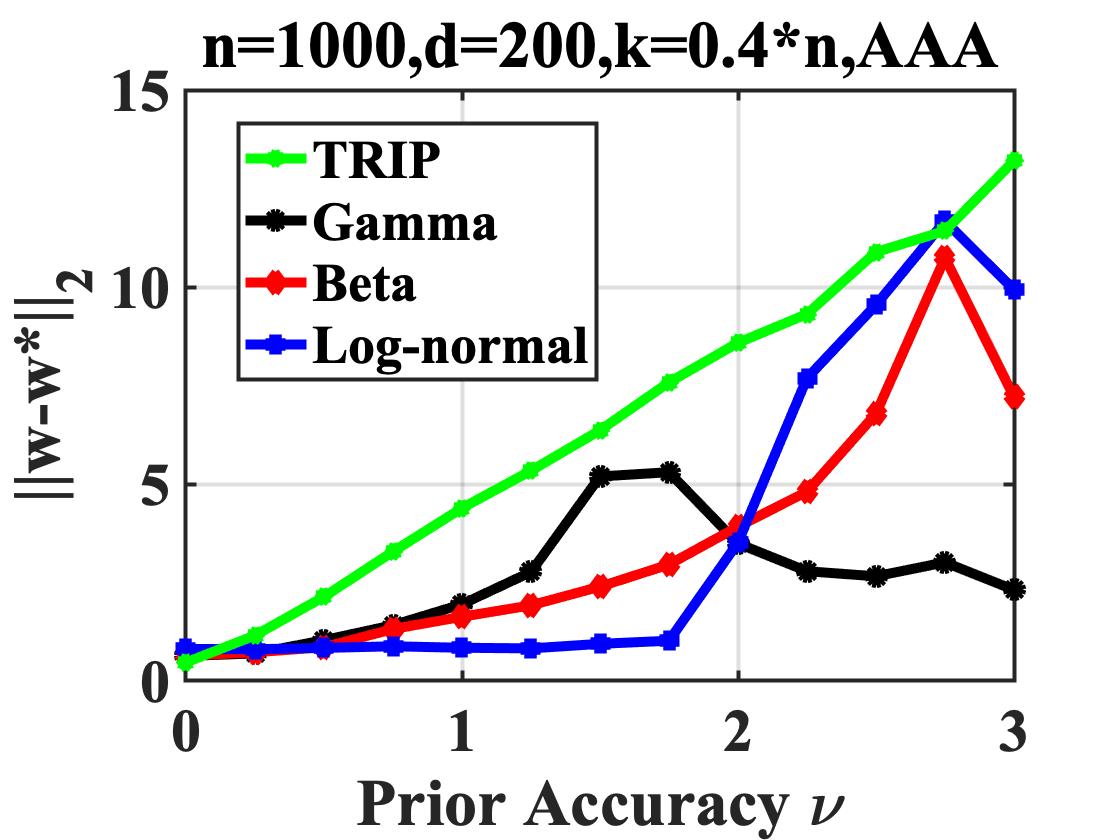}
\end{minipage}
}
\caption{(a),~(b) Convergence characteristics of TRIP and BRHT. (c),~(d) Convergence under different prior $p_{\mathbf{r}}(\mathbf{r})$ and coefficient prior $p_{\mathbf{w}}(\mathbf{w})$.}
\label{figure2}
\end{figure*}

\subsection{Recovery Properties of Coefficients and Uncorrupted Sets}\label{Recovery Properties}

In this subsection, we summarize the simulation results obtained by the different methods. The standard $L_2$ errors of the parameter estimation are shown in Figure~\ref{figure1}, where we compare the $L_2$ error of parameter estimation when the corruption ratio increases under different sample size and parameter dimensions. In particular, we consider two cases, i.e., $n=1000, d=200$ and $n=2000, d=100$, and for each case the results are evaluated for an OAA and an AAA. As can be seen, CRR, RRBR, and Rob-ULA are excellent robust regression methods or Bayesian inference methods, but they all have limitations when faced with different types of attack. CRR performs the best when faced with an OAA because it is theoretically unbiased, but it collapses rapidly when faced with an AAA, as shown in Figures~\ref{figure1}(c) and \ref{figure1}(d). RRBR and Rob-ULA consider priors, but RRBR cannot resist an AAA, and Rob-ULA produces poor results under an OAA, as shown in Figures~\ref{figure1}(a)--\ref{figure1}(d). TRIP performs well against an AAA, while BRHT is not only optimal against an AAA but also performs similarly to CRR against an OAA. This shows that of the algorithms compared in this experiment, BRHT is the most robust.

In Figures~\ref{figure2}(a) and \ref{figure2}(b), we investigate the convergence properties of TRIP and BRHT. As shown in Figure~\ref{figure2}(a) for an OAA, the estimation errors of CRR and BRHT decrease gradually, while that of TRIP increases to ca.~50. This indicates that TRIP puts too much weight on prior information, resulting in a greater estimation error. However, as stated in Theorem~\ref{theorem2}, the higher the prior weight, the easier it is to meet the convergence conditions. This means that the convergence of TRIP should be easier to guarantee than that of BRHT, and the error in the case of an AAA should also be ca.~50 as in the case of an OAA. However, from Figure~\ref{figure2}(b), BRHT converges under a weaker prior, while TRIP becomes trapped around a local optimum. This shows that BRHT only needs to integrate weak prior information to ensure convergence, which guarantees the accuracy of the algorithm.

Figures~\ref{figure2}(c) and \ref{figure2}(d) show the convergence properties under different weights of the prior $p_{\mathbf{r}}(\mathbf{r})$ and coefficient prior $p_{\mathbf{w}}(\mathbf{w})$. In these experiments, we tried more types of prior on BRHT to confirm its efficacy, with the newly added priors being the log-normal distribution $\log N(1,1)$ and the beta distribution $\beta(10,20)$. In Figure~\ref{figure2}(c), the results of applying these three types of prior show no significant difference with increasing corruption ratio. This shows that BRHT is not especially sensitive to the weight of prior $p_{\mathbf{r}}(\mathbf{r})$ when this prior is relatively reliable. Figure~\ref{figure2}(d) shows the effects of applying the three types of prior with increasing prior uncertainty. Of the different methods, TRIP always performs the worse. A log-normal distribution is the best choice when the prior $p_{\mathbf{w}}(\mathbf{w})$ is relatively close to the real parameters, and a gamma distribution is more robust when the prior is imprecise. This result can be used as a reference for selecting a prior.

\subsection{Results of Applying a Data-driven Prior}

In this subsection, we show the experimental results of applying a data-driven prior. As the prior mean $\mathbf{w}_0$, we use the solution of LAD regression \citep{pollard1991asymptotics}, which is a commonly used robust regression method that reduces the influence of outliers by using the absolute loss function. However, the efficiency of this estimation is relatively low when there are no outliers, and the breakdown point is only $1/n$. Because an AAA is so severe that neither can LAD regression resist it, in this experiment we apply an OAA in which we also let the set of corrupted points $S$ be selected as a uniformly random $k$-sized subset of $[n]$. The uncorrupted response variables are set as $y_i=\mathbf{x}_{i}^{T}\mathbf{w}^{*}+\epsilon_i$, while the corrupted ones are sampled from the uniform distribution $U[0,10]$. The white noise is $\epsilon_i\sim\mathcal{N} (0,1)$. This attack is slightly stronger than the OAA in Section~\ref{Corruption Method}, the aim being to collapse the data more effectively and make the experimental results clearer. In this case, we employ the prior $p_{\mathbf{w}}(\mathbf{w}) = N(\mathbf{w}_0, \frac{1}{0.05n}I)$ for TRIP and $p_{\mathbf{w}}(\mathbf{w}) = N(\mathbf{w}_0, \frac{1}{0.01n}I)$ for BRHT.

Figure~\ref{LAD} shows the $L_2$ error of parameter estimation by the different algorithms for increasing corruption ratio, where LAD+TRIP and LAD+BRHT denote TRIP and BRHT using the result of LAD as the prior. As in Section~\ref{Recovery Properties}, we consider two cases, i.e., $n=1000, d=200$ and $n=2000, d=100$. In the case of $n=1000, d=200$, the largest corruption ratio is 0.4, which is because LAD performs less well for larger corruption ratio, which means that if we still use this data-driven prior, then this unreliable prior will lead to a bad result. By using the solution of LAD regression as the prior mean, both TRIP and BRHT exhibit excellent robustness, as shown in Figures~\ref{LAD}(a) and \ref{LAD}(b). LAD indeed shows robustness, but its estimation error is usually larger than that of CRR. However, TRIP and BRHT behave much better than do LAD and CRR, which means that the prior is effective in helping the former two algorithms to find the correct solutions. This also shows that our algorithms can reduce the estimation errors of other algorithms, and the combination of algorithms may be better than the original CRR algorithm.

\begin{figure}[ht]
\centering
\subfigure[]
{
\begin{minipage}{6cm}
\centering
\includegraphics[width=6cm]{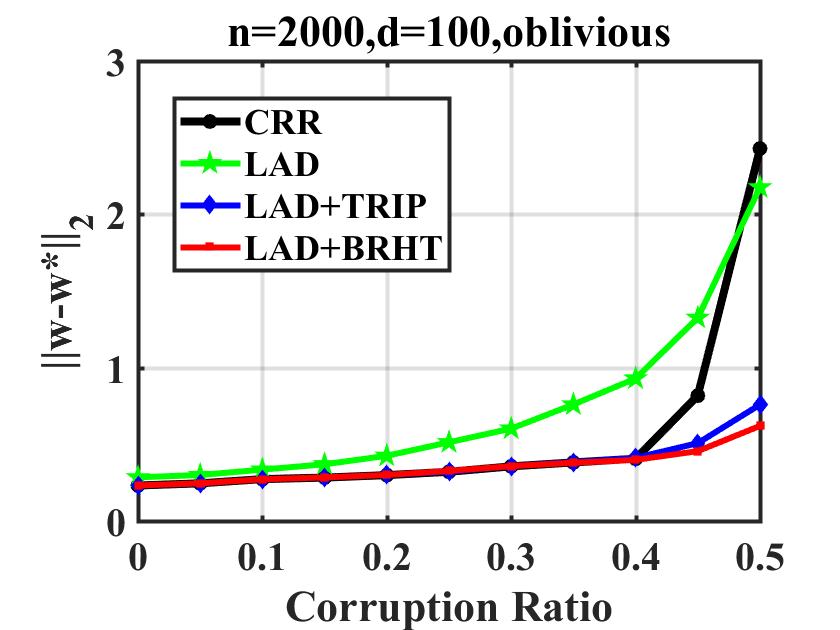}
\end{minipage}
}
\subfigure[]
{
\begin{minipage}{6cm}
\centering
\includegraphics[width=6cm]{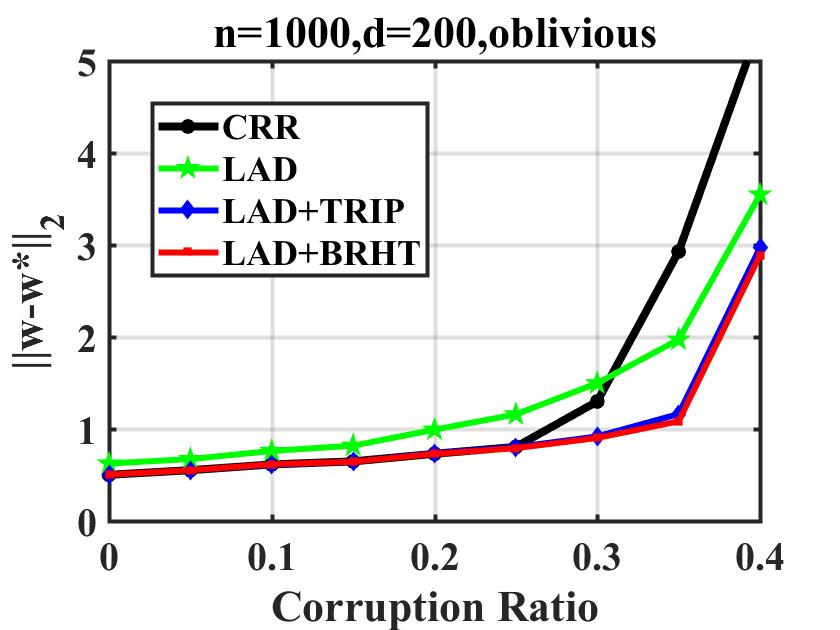}
\end{minipage}
}
\caption{Recovery of parameters with respect to number of data points $n$, dimensionality $d$, and corruption ratio $\alpha$.}
\label{LAD}
\end{figure}

\section{Application: Data Analysis for Space Solar Array}\label{sec:application}

%\cmtzh{Better to split this section to subsections, e.g., Background, Data analysis, Results, etc. }

In this section, we introduce a real-world application, i.e., SSA data analysis. First, we briefly introduce the background of SSA data analysis (Section~\ref{subsec:background}), then we describe in detail the application of the proposed procedure to SSA data reconstruction (Section~\ref{subsec:data_analysis}), and finally we present results (Section~\ref{subsec:result}).

\subsection{Background}\label{subsec:background}

Our method can be applied to two main scenarios: (i) situations in which a rough understanding of the actual data based on previous experiments or results can be easily obtained as prior information (e.g., in economics, biology, industry, etc.); (ii) periodic data, for which because the data of each period are roughly similar, we can manually select a period with basically no noise as the prior data of all periods. In this subsection, we introduce a real-world application of periodic data denoising, which pertains to SSA data from LEO satellites as introduced in Section~1.

The original data describe the power changes of an SSA recorded by its satellite during four years of operation. The data points are in the form of $(t_{ij},p_{ij})$, where $t_{ij}$ is the $j$th recorded time point of the $i$th period, and $p_{ij}$ is the power value of the SSA at time $t_{ij}$. These data have several significant characteristics: (i) they are nearly periodic, which means that they are similar between each period, but perhaps with some slight changes; (ii) they have a decreasing trend, which is caused by the degradation of the SSA; (iii) there are several jump points, each corresponding to some sudden failure in the SSA; (iv) most difficultly, there are many outliers. Regarding (iv), because the main memory of the satellite is relatively small, the data storage capacity is limited; data cannot be transmitted when the satellite is far from the receiving station, and in that case, if the main memory is full, then old data are overwritten by newly recorded data, causing data loss and generating many outliers. Outliers of this type occupy a high proportion (ca.\ $20\%$--$30\%$) of data points, and unlike random noise, they have a certain pattern, which makes them difficult to eliminate by conventional methods. The data number 20 million in total, thereby requiring an efficient denoising method, and we now intercept the data from 300 days in the first year for analysis.

From expert knowledge, the typical power of the SSA is between 26 and 30 (anonymous data with removing unit). By removing points with power values below 26 (known as outliers), the overall distribution of the original data is as shown in Figure~\ref{raw data amplifying}(a). In each period, the fluctuation range of the power data does not exceed 5; see Figure~\ref{raw data amplifying}(b) for example. Thus, the large dispersion of the overall data indicates the outliers in the original data. By locally amplifying the data, the fluctuation in each period can be observed. The uncorrupted data are in the form of a typical waveform as shown in Figure~\ref{raw data amplifying}(b), but there are many outliers in the real data as shown in Figure~\ref{raw data amplifying}(c), making the data analysis challenging.

\begin{figure}[tbp]
\centering
\subfigure[]
{
\begin{minipage}{5cm}
\centering
\includegraphics[width=5cm]{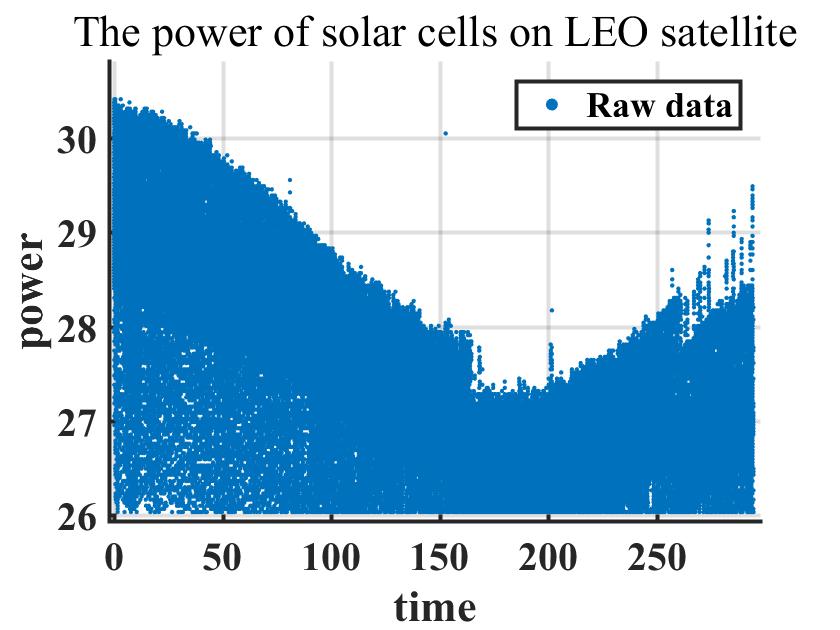}
\end{minipage}
}
\subfigure[]
{
\begin{minipage}{5cm}
\centering
\includegraphics[width=5cm]{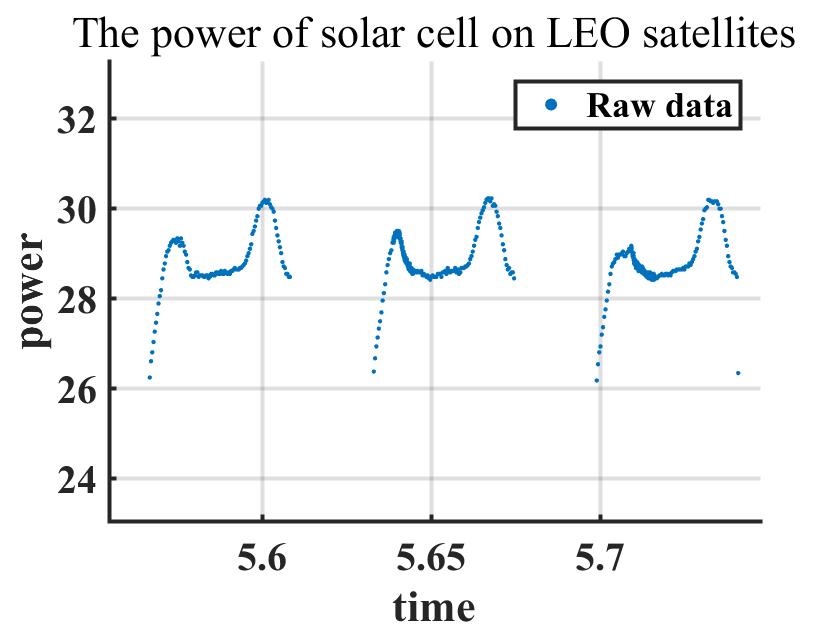}
\end{minipage}
}
\subfigure[]
{
\begin{minipage}{5cm}
\centering
\includegraphics[width=5cm]{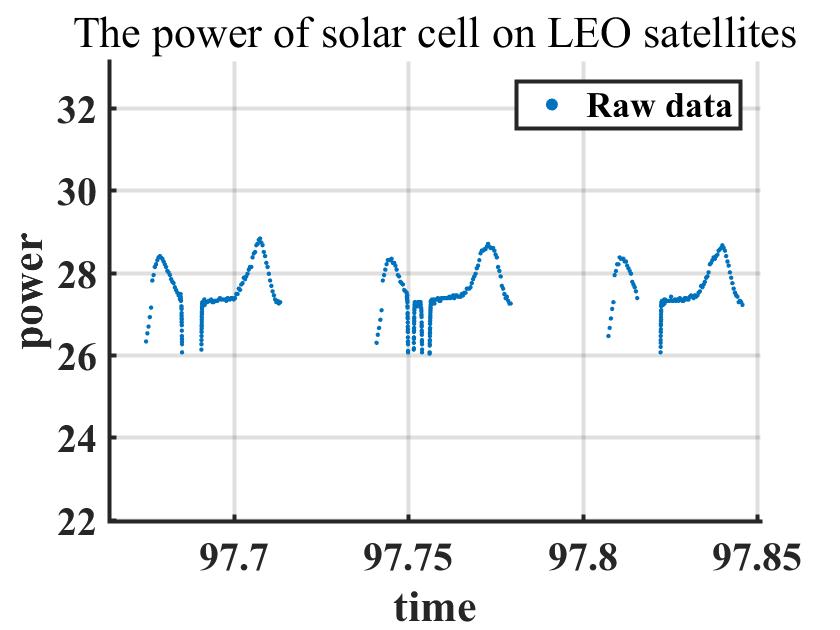}
\end{minipage}
}
\caption{Overall and partially enlarged views of original data.}
\label{raw data amplifying}
\end{figure}

\subsection{Data Analysis}\label{subsec:data_analysis}

Given the many data, we use the faster TRIP algorithm to estimate the true information and remove outliers. For each period, we use Chebyshev polynomials to fit the data because they are orthogonal basis functions and their fitting result is stable \citep{mason2002chebyshev}. As discussed in Section~\ref{choose_prior}, we estimate the coefficients of polynomial model from one standard period and set them as the prior mean for the regression parameters in RLSR model. The standard period is obtained by manually selecting a period and removing its noise points, as shown in Figure~\ref{standard}(a), where the blue dots are the standard-period data and the orange line is the linear least-squares fitting curve obtained using a nine-degree Chebyshev polynomial; the fitting formula for the standard period is $p_{st}(x)= \sum_{i=0}^{9} \gamma_i^{st} T_i(t)$, where $T_0(t)=1, T_1(t)=t, T_{n+1} (t)=2tT_n (t)-T_{n-1}(t), n=1,2,\dots, 8$ are Chebyshev polynomials.

The estimated parameters $[\gamma_{0}^{st},\gamma_{1}^{st},\dots,\gamma_{9}^{st}]^T$ of the standard period can be regarded as the prior mean $\mathbf{w}_0$ of all periods. In the TRIP algorithm, we set the prior precision matrix as $M=s\times \text{diag}\{0,1,1,\dots,1\}$, the first diagonal element of which is set to 0, i.e., we assign a noninformative prior $\gamma_{i0}$ because the intercept $\gamma_{i0}$ varies from the standard period. On the other hand, because the shape of each period is relatively consistent, we use a relatively large weight $s=1$ to ensure better recovery.

The period length of the SSA data is $T$, and they contain $n_{period}$ periods. The $i$th period contains $n_i$ data points, and first we subject them to a simple transformation so that they match the prior, i.e., $\tilde{t}_{ij}=t_{ij}-(i-1)T, j =1,\dots,n_i$. Let $X_i=[X^1_i, \dots,X^{n_i}_i]$ denote the matrix generated by the Chebyshev basis associated with $\tilde{t}_{ij}$, i.e., $X^j_i=[T_0(\tilde{t}_{ij}),\dots,T_9(\tilde{t}_{ij})]^T$. We use the TRIP algorithm to recover the current $i$th period, and by applying this recovery process, we obtain the recovered data $(\tilde{t}_{ij},p_i (\tilde{t}_{ij}))$, which we move back to the original location and obtain the final result $(t_{ij},p_i (\tilde{t}_{ij}))$. The overall recovery algorithm is given in Algorithm~\ref{alg:denoising algorithm}.

\spacingset{1}
\begin{algorithm}[tbp]
\caption{Denoising Algorithm for Data from Space Solar Array}
\label{alg:denoising algorithm}
\begin{algorithmic}[1]
\REQUIRE SSA data of LEO satellite $(t_{ij},p_{ij})$, standard-period parameter $\mathbf{w}_0$, penalty matrix $M$, corruption ratio $\alpha$, tolerance $\epsilon$, period length $T$
\ENSURE cleaned data $(t_{ij},p_i (\tilde{t}_{ij}))$
\WHILE {$i<n_{period}$}
\FOR{$j=1:n_i$}
\STATE $\tilde{t}_{ij}= t_{ij}-(i-1)T$
\ENDFOR
\STATE $k=\alpha n_i$;
\STATE $\mathbf{w}_i\gets TRIP(X_i,p_i,\mathbf{w}_0,M,k,\epsilon)$
\STATE $(t_{ij},p_{ij})\gets (t_{ij},X^j_i\mathbf{w}_i)$
\ENDWHILE
\RETURN $\{(t_{ij},p_{ij})\}$
\end{algorithmic}
\end{algorithm}
\spacingset{2}

\subsection{Results}\label{subsec:result}

Figure~\ref{standard}(b) shows the results recovered by TRIP and CRR. Although the rightmost two periods are thoroughly corrupted, our TRIP algorithm can still restore the data to their original state, which shows that our method can identify the uncorrupted points and obtain good recovery results.

\begin{figure}[tbp]
\centering
\subfigure[]
{
\begin{minipage}{6cm}
\centering
\includegraphics[width=6cm]{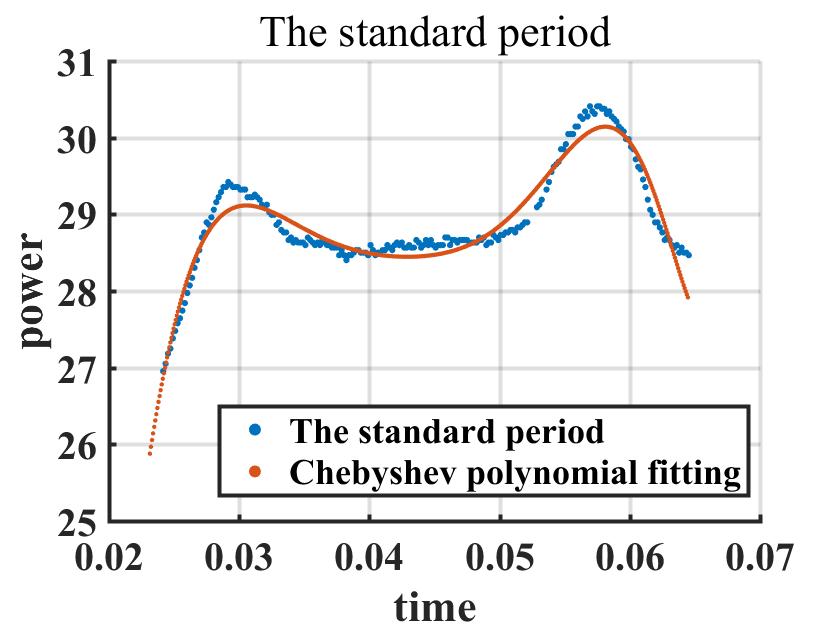}
\end{minipage}
}
\subfigure[]
{
\begin{minipage}{6cm}
\centering
\includegraphics[width=6cm]{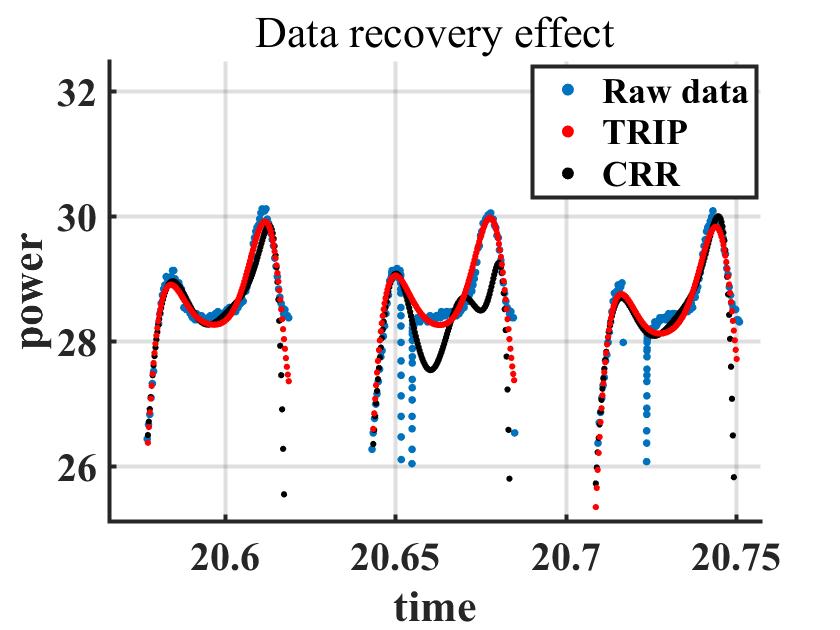}
\end{minipage}
}
\caption{(a) Data and polynomial fitting in one period, and (b) effects of data recovery.}
\label{standard}
\end{figure}

To capture the overall trend, we further extract information from the recovered data. In each period, we use a special point $((i-1)T+t_c,p_i (t_c )), i =1,\dots,n_{period}$ to represent this period, where $t_c$ is a prespecified value. By plotting all these representative points, the trend in SSA power is easily observed as shown in Figure~\ref{trend}(a). We use the $k$ nearest neighbors (KNN) \citep{peterson2009k} algorithm to further denoise the current points to obtain a clearer trend as shown in Figure~\ref{trend}(b). The trend in Figure~\ref{trend} has two jump points, indicating that the SSA was partially damaged at these two times.

\begin{figure}[tbp]
\centering
\subfigure[]
{
\begin{minipage}{7cm}
\centering
\includegraphics[width=7cm]{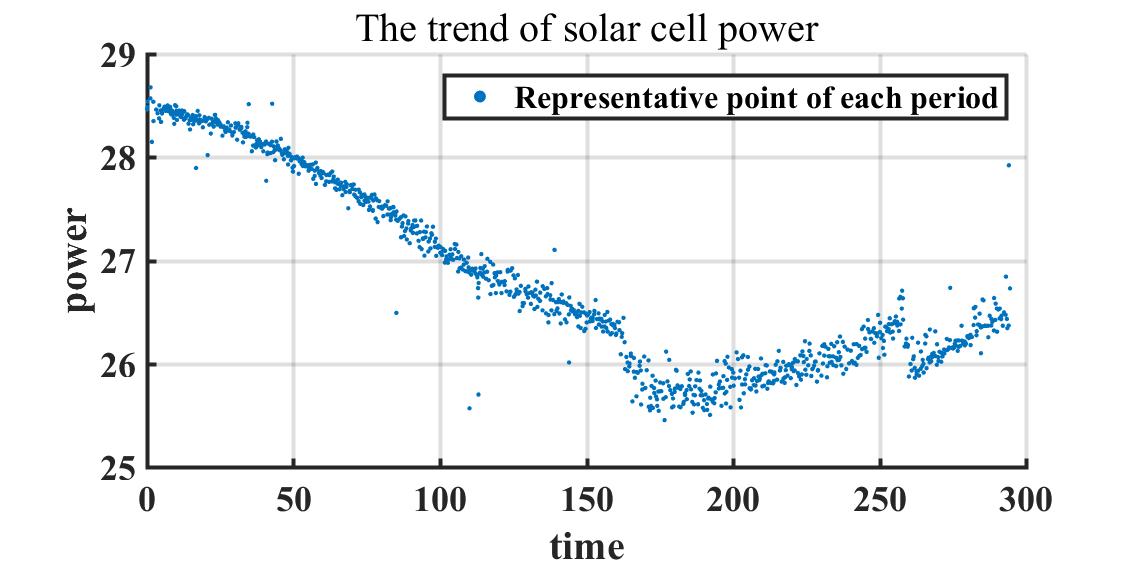}
\end{minipage}
}
\subfigure[]
{
\begin{minipage}{7cm}
\centering
\includegraphics[width=7cm]{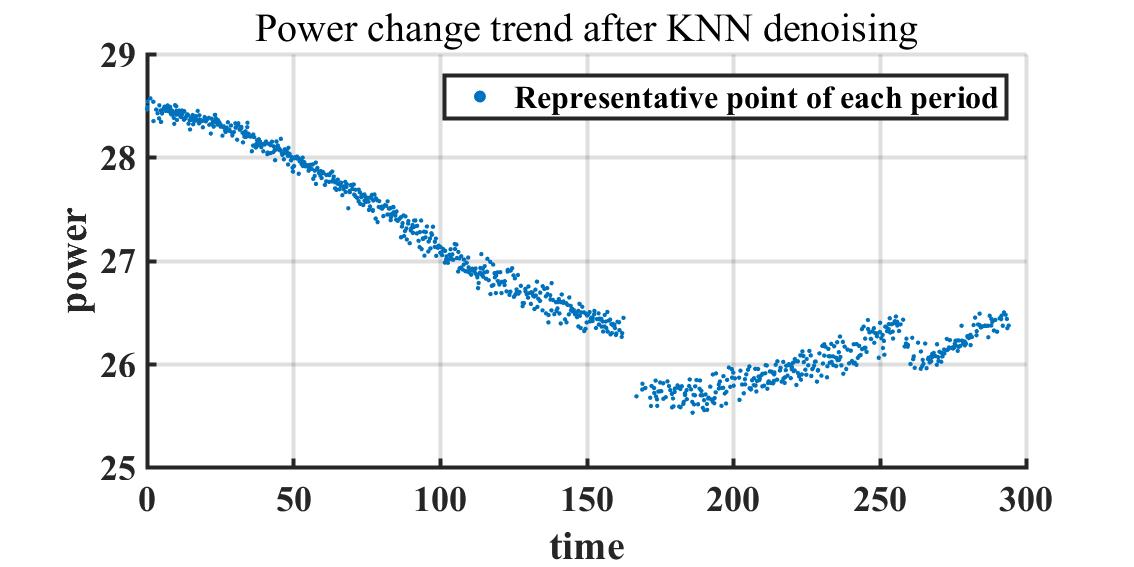}
\end{minipage}
}
\caption{ (a) Extracted overall data trend, and (b) trend after further denoising.}
\label{trend}
\end{figure}

\section{Conclusion}

Described herein was TRIP, a novel robust regression algorithm that performs strongly in resisting AAAs; by adding a prior to the robust regression via hard thresholding, the recovery of coefficients is improved significantly. Another algorithm, i.e., BRHT, was designed to improve the robustness of TRIP and reduce the estimation error by using Bayesian reweighting regression. We proved that both algorithms have strong theoretical guarantees and that they converge linearly under mild conditions. Extensive experiments showed that our algorithms outperform benchmark methods in terms of both robustness and efficiency. Our algorithms also achieved good results in recovering data from a real SSA, thereby showing that they have high practical value.

There are several interesting directions in which the current work could be extended. First, whereas herein we considered only the case in which $\mathbf{y}$ is corrupted, one could consider using prior information to deal better with the case in which both $\mathbf{y}$ and $X$ are corrupted. Second, it would be interesting to reduce further the effect of prior bias on the estimation.

\newpage
\spacingset{1}
\section*{Data availability statement}
The authors confirm that the data supporting the findings of this study are available within the article and its supplementary materials.
\bibliographystyle{apalike}
\bibliography{reference_LP}
\end{document}

\if0\blind	
{
\title{\large\bf Supplementary Materials for Paper Titled "A Bayesian Robust Regression Method for Corrupted Data Reconstruction"}
\author{
	Zheyi Fan\thanks{Academy of Mathematics and Systems Science, Chinese Academy of Sciences, China;
		School of Mathematical Sciences, University of Chinese Academy of Sciences, China. \texttt{fanzheyi@amss.ac.cn}, \texttt{Qingpeihu@amss.ac.cn}}, \\
	\And
	Zhaohui Li\thanks{H.\ Milton Stewart School of Industrial and Systems Engineering, Georgia Institute of Technology, USA. \texttt{zhaohui.li@gatech.edu}}, \\
	\And
	Jingyan Wang\thanks{Beijing Institute of Spacecraft System Engineering, China. \texttt{yanering@tom.com}, \texttt{xiong\_ztb@126.com}}, 
	\And
	Dennis K. J. Lin\thanks{Department of Statistics, Purdue University, West Lafayette, Indiana, \texttt{dkjlin@purdue.edu}}, 
	\And
	Xiao Xiong\footnotemark[3], 
	\And
	Qingpei Hu\footnotemark[1]
}

\maketitle
	} 
\fi
	
	\if1\blind
	{
		\begin{center}
			{\large\bf Supplementary Materials for Paper Titled "A Bayesian Robust Regression Method for Corrupted Data Reconstruction"}
		\end{center}
		\medskip
	} \fi
	\bigskip
\theoremstyle{plain} 
\newtheorem{defination}{Definition}
\spacingset{1.5} % DON'T change the spacing!

\appendix
%\section{Appendix}
\section{Details of Variational Bayesian EM Method}
For the RPM model and the given covariates $X=[\mathbf{x}_{1},...,\mathbf{x}_{n}]$, responses $\mathbf{y}=[y_{1},...,y_{n}]^T$,
and prior distributions $p_{\mathbf{r}}(\mathbf{r})$, $p_{\mathbf{w}}(\mathbf{w})$, the posterior of this RPM model is formulated as
\begin{equation*}
\log p_{\mathbf{w}}(\mathbf{w})+\log p_{\mathbf{r}}(\mathbf{r})+\sum\limits_{i=1}^{n} r_{i}\log\ell(y_{i}\mid \mathbf{w},\mathbf{x}_i,\sigma^2)
\end{equation*}
where $\ell(y_{i}\mid \mathbf{w},\mathbf{x}_i,\sigma^2)$ represent the likelihood of the normal distribution $\mathcal{N} (x_{i}^{T}\mathbf{w},\sigma^2)$. $p_{\mathbf{w}}(\mathbf{w})$, $p_{\mathbf{r}}(\mathbf{r})$ are the priors of $\mathbf{w}$ and $\mathbf{r}$. $p(\mathbf{w})$ is the density of the normal distribution $\mathcal{N} (\mathbf{w}_0,\Sigma_0)$. We now use variational Bayesian EM to approximate the true posterior.
\begin{equation*}
q(\mathbf{w})\prod_{i}q(r_{i})\approx
p(\mathbf{w},\mathbf{r}|\mathbf{y},X,\sigma^2)
\end{equation*}
In the following, we derive update equations for these variational parameters.

\subsection{Derivation of $q(\mathbf{r})$ (variational E step)}

Ignoring terms that do not involve $\mathbf{r}$, we take the expectations of over the remaining terms. We have
\begin{align*}
\log q(\mathbf{r}) &=\mathbb{E}_{q(\mathbf{w})} [\log p(\mathbf{y},\mathbf{w},\mathbf{r}|X)]+const\\
&=\log p_{\mathbf{r}}(\mathbf{r})+\sum_{i}r_i \mathbb{E}_{q(\mathbf{w})}[\log \ell(y_{i}|\mathbf{w},\mathbf{x}_i,\sigma^2)]+const
\end{align*}
where $q(\mathbf{w})$ is the form of the normal distribution $\mathcal{N} (\mathbf{w}_N,V_N)$, as will be shown in the derivation of $q(\mathbf{w})$. Using this fact, we have
\begin{align*}
\mathbb{E}_{q(\mathbf{w})}[\log \ell(y_{i}|\mathbf{w},\mathbf{x}_i,\sigma^2)]
&=\mathbb{E}_{q(\mathbf{w})}[-\frac{1}{2\sigma^2}(y_{i}-\mathbf{w}^{T}\mathbf{x}_i)^2-\frac{1}{2}\log(2\pi \sigma^2)]\\
&=-\frac{1}{2\sigma^2}[(y_{i}-\mathbf{w}_{N}^{T}\mathbf{x}_i)^2+\mathbf{x}_i^{T}V_N\mathbf{x}_i]-\frac{1}{2}\log(2\pi \sigma^2)
\end{align*}

When the priors $p_{\mathbf{r}}(r_i)$ are independent of each other, then 
\begin{equation*}
q(r_i)\propto \exp\{\log p_{\mathbf{r}}(r_i)+r_i\mathbb{E}_{q(\mathbf{w})}[\log \ell(y_{i}\mid \mathbf{w},\mathbf{x}_i,\sigma^2)]\}
\end{equation*}
If the distribution of $q(r_i)$ is complex, we use a Markov chain Monte Carlo method to simulate the distribution. An easy example for $p_{\mathbf{r}}(r_i)$ is the Gamma distribution $Gam(r_i|a_r,b_r)$. Then,
\begin{align*}
q(r_i)&=Gam(r_i|a_{N}^i,b_{N}^i)\\
a_{N}^i&=a_r\\
b_{N}^i&=b_r-\mathbb{E}_{q(\mathbf{w})}[\log \ell(y_{i}\mid\mathbf{w},\mathbf{x}_i,\sigma^2)]\\
\mathbb{E}_{q(\mathbf{w})}[\log \ell(y_{i}\mid \mathbf{w},\mathbf{x}_i,\sigma^2)]
&=-\frac{1}{2\sigma^2}[(y_{i}-\mathbf{w}_{N}^{T}\mathbf{x}_i)^2+\mathbf{x}_i^{T}V_N\mathbf{x}_i]-\frac{1}{2}\log(2\pi \sigma^2)
\end{align*}
for which the expectation of $r_i$ under the distribution of $q(r_i)$ can be easily obtained by $a_{N}^i/b_{N}^i$.

\subsection{Derivation of $q(\mathbf{w})$ (variational M step)}
\begin{align*}
\log q(\mathbf{w}) 
&=\mathbb{E}_{q(\mathbf{r})} [\log p(\mathbf{y},\mathbf{w},\mathbf{r}|X)]+const\\
&=\log p_{\mathbf{w}}(\mathbf{w})+\sum_{i}\mathbb{E}_{q(\mathbf{r})}(r_i) \log \ell(y_{i}\mid \mathbf{w},\mathbf{x}_i,\sigma^2)+const\\
&=-\frac{1}{2}(\mathbf{w}-\mathbf{w}_0)^T\Sigma^{-1}(\mathbf{w}-\mathbf{w}_0)-\frac{1}{2\sigma^2}(\mathbf{y}-X^T\mathbf{w})^{T}E_r(\mathbf{y}-X^T\mathbf{w})+const\\
&=-\frac{1}{2}(\mathbf{w}-\mathbf{w}_N)^{T}V_N^{-1}(\mathbf{w}-\mathbf{w}_N)+const
\end{align*}
where $E_r$ is a matrix with diagonal entries of ${E}_{q(\mathbf{r})}(\mathbf{r})$ and off-diagonal elements of 0. $V_N^{-1}=\frac{1}{\sigma^2}XE_{r}X^T+\Sigma^{-1}$, $\mathbf{w}_N=V_N(\frac{1}{\sigma^{2}}XE_{r}\mathbf{y}+\Sigma^{-1}\mathbf{w}_0)$, and $q(\mathbf{w})=\mathcal{N} (\mathbf{w}_N,V_N)$.

After several iterations, we use $q(\mathbf{w})\prod_{i}q(r_{i})$ to approximate the true posterior and take $\mathbf{w}_N$ as the MAP estimate of $q(\mathbf{w})$.

\section{Adaptive Data Corruption Method}
As the original CRR algorithm can tolerate OAAs to a significant degree, we must use all information contained in the data to fool the estimator. To achieve this goal, we need to find the most suitable subset to corrupt such that the rest of the data can more likely be generated from a completely different distribution. This problem can be formulated as follows:
\begin{equation}
(\hat{\mathbf{w}},\hat S)=\arg\min_{\substack{\mathbf{w}\in\mathbb{R}^{p},S\subset[n]\\ |S|=n-k}}\sum_{i \in S}(y_{i}-x_{i}^{T}\mathbf{w})^2 -\delta \|\mathbf{w}-\mathbf{w}^*\|_2^2
\end{equation}
where $\delta$ is the penalty coefficient that determines the extent to which the parameter leaves the standard value. $\hat S$ is the chosen subset that cannot be corrupted. If $\delta$ is not very large, then $\sum_{i \in \hat S}(y_{i}-x_{i}^{T}\mathbf{w}^*)^2$ will be similar to $\sum_{i \in \hat S}(y_{i}-x_{i}^{T}\hat{\mathbf{w}})^2$, and so $\hat{\mathbf{w}}$ can fool the regression model into thinking that $\hat{\mathbf{w}}$ is the true parameter. From this analysis, we find that $\hat{\mathbf{w}}$ can be used to construct the corrupted data. After getting the covariates of the corrupted data, we can define the response of the corrupted data as $y_{c_i}=\mathbf{x}_{c_i}^T\hat{\mathbf{w}} $, where $\mathbf{x}_{c_i}$ is the $i^{th}$ covariate of the corrupted data.

This problem is very similar to that in Eq. (7) in Section 4.1, where $M$ is replaced by $-\delta I$ and $\mathbf{w}_0$ is replaced by $\mathbf{w}^*$. Hence, these two problems can be solved by the same method. Similar to TRIP, we proposed an adaptive data corruption algorithm (ADCA) to solve the corruption problem by replacing some parameters in TRIP. ADCA seriously destroys the data, and when the corruption ratio increases, the solution of Eq. (1) may not be close to the true parameter. However, we will see that, even in this situation, TRIP and BRHT achieve good performance, as shown in Section 6.

\theoremstyle{plain} 
\newtheorem{lemma}{Lemma}
\theoremstyle{plain} 
\newtheorem{theorem_apd}[lemma]{Theorem}
\section{Supplementary Material for Proofs of TRIP and BRHT Algorithms}
\subsection{SSC/SSS guarantees }
\setcounter{lemma}{7}
\renewcommand{\thelemma}{\arabic{lemma}}
In this section, we introduce some theoretical properties of SSC and SSS from \citep{bhatia2015robust}, which will be used for the convergence analysis of the proposed algorithms. 

\begin{defination}
A random variable $x\in\mathbb{R}$ is called sub-Gaussian if the following quantity is finite
\begin{equation*}
    \sup_{p\ge 1}p^{-1/2}(E[|x|^p])^{1/p}
\end{equation*}
Moreover, the smallest upper bound on this quantity is referred to as the sub-Gaussian norm of $x$ and denoted as $\|x\|_{\psi_2}$
\end{defination}

\begin{defination}
A vector-valued random variable $\mathbf{x}\in\mathbb{R}^d$ is called sub-Gaussian if its unidimensional marginals $\langle\mathbf{x},\mathbf{v}\rangle$ are sub-Gaussian for all $\mathbf{v}\in S^{d-1}$. Moreover, its sub-Gaussian norm is defined as follows

\begin{equation*}
    \|x\|_{\psi_2}=\sup_{\mathbf{v}\in S^{d-1}}\|\langle\mathbf{x},\mathbf{v}\rangle\|_{\psi_2}
\end{equation*}
\end{defination}

\begin{lemma}\label{SSC_all}
	Let $X\in\mathbb{R}^{d\times n}$ be a matrix whose columns are sampled i.i.d from a standard Gaussian distribution i.e. $\mathbf{x}_{i}\sim\mathcal{N}(0,I)$. Then for any $\epsilon>0$, with probability at least $1-\delta$, $X$ satisfies
	\begin{align*}
		\lambda_{max}(XX^T)\le n+(1-2\epsilon)^{-1}\sqrt{cnd+c'n\log \frac{2}{\delta}}\\
		\lambda_{min}(XX^T)\ge n-(1-2\epsilon)^{-1}\sqrt{cnd+c'n\log \frac{2}{\delta}}
	\end{align*}
	where $c=24e^2log\frac{3}{\epsilon}$ and $c'=24e^2$.
\end{lemma}
	
\begin{theorem_apd}\label{SSC_local}
	Let $X\in\mathbb{R}^{d\times n}$ be a matrix whose columns are sampled i.i.d from a standard Gaussian distribution i.e. $\mathbf{x}_{i}\sim\mathcal{N}(0,I)$. Then for any $k>0$, with probability at least $1-\delta$, the matrix $X$ satisfies the SSC and SSS properties with constants
	\begin{align*}
		&\Lambda_{k}\le k(1+3e\sqrt{6\log\frac{en}{k}})+\textit{O}(\sqrt{nd+n\log\frac{1}{\delta}})\\
		&\lambda_{k}\ge n-(n-k)(1+3e\sqrt{6\log\frac{en}{n-k}})-\Omega(\sqrt{nd+n\log\frac{1}{\delta}})
	\end{align*}
	
\end{theorem_apd}

\begin{lemma}
	Let $X\in\mathbb{R}^{d\times n}$ be a matrix with columns sampled from some sub-Gaussian distribution with sub-Gaussian norm $K$ and convariance $\Sigma$. Then for any $\delta>0$, with probability at least $1-\delta$, each of the following statements holds true:
	\begin{align*}
		\lambda_{max}(XX^T)\le \lambda_{max}(\Sigma)\cdot n+C_{K}\cdot\sqrt{dn}+t\sqrt{n}\\
		\lambda_{min}(XX^T)\ge \lambda_{min}(\Sigma)\cdot n-C_{K}\cdot\sqrt{dn}-t\sqrt{n}
	\end{align*}
	where $t=\sqrt{\frac{1}{c_{K}}\log \frac{2}{\delta}}$ and $c_{K}$, $C_{K}$ are absolute constants that depend only on the sub-Gaussian norm $K$ of the distribution.
\end{lemma}

\subsection{Convergence Proof for TRIP }\label{TRIP proof}

\setcounter{lemma}{0}
\renewcommand{\thelemma}{\arabic{lemma}}

\begin{theorem_apd}
Let $X=[\mathbf{x}_1,\dots,\mathbf{x}_n]\in \mathbb{R}^{d\times n}$ be the given data matrix and $\mathbf{y}=X^{T}\mathbf{w}^{*}+\mathbf{b}^{*}+\boldsymbol{\epsilon}$ be the corrupted output with sparse corruptions of $\|\mathbf{b}^*\|_0\le k$. For a specific positive semi-definite matrix $M$, the data matrix $X$ satisfies the SSC and SSS properties such that $2\frac{\Lambda_{k+k^{*}}}{\lambda_{min}(XX^T+M)}<1$. Then, if $k>k^{*}$, it is guaranteed with a probability of at least $1-\delta$ that, for any $\varepsilon,\delta>0$, $\|\mathbf{b}^{T_0}-\mathbf{b}^*\|_{2}\le \varepsilon + \textit{O}(e_0)+\textit{O}(\frac{\sqrt{\Lambda_{k+k^*}}\lambda_{max}(M)}{\lambda_{min}(XX^T+M)})\|\mathbf{w}^*-\mathbf{w}_0\|_2$ after $T_{0}=\textit{O}(\log (\frac{\|\mathbf{b}^*\|_2}{\varepsilon}))$ iterations of TRIP, where $e_0=\textit{O}(\sigma\sqrt{(k+k^*)\log\frac{n}{\delta(k+k^*)}})$ under the normal design.
\end{theorem_apd}

\begin{proof}
First, we consider the iteration of the TRIP algorithm:
\[
\mathbf{b}^{t+1}\gets HT_k(P_{MX}\mathbf{b}^{t}+(I-P_{MX})y-P_{MM}\mathbf{w}_0)
\]
After considering $\mathbf{y}=X^{T}\mathbf{w}^{*}+\mathbf{b}^{*}+\boldsymbol{\epsilon}$, the iteration step can be rewritten as: 
\[
\mathbf{b}^{t+1}\gets HT_k(\mathbf{b}^{*}+X^{T}\boldsymbol{\lambda}^t+\mathbf{g}+\mathbf{f})
\]
where
\begin{align*}
\boldsymbol{\lambda}^t&=(XX^T+M)^{-1}X(\mathbf{b}^{t}-\mathbf{b}^{*})\\
\mathbf{g}&=(I-P_{MX})\boldsymbol{\epsilon}\\
\mathbf{f}&=P_{MM}(\mathbf{w}^*-\mathbf{w}_0)
\end{align*}
Because $k>k^*$, we use the property of the hard thresholding step:
\begin{align*}
\|\mathbf{b}_{I^{t+1}}^{t+1}-(\mathbf{b}_{I^{t+1}}^{*}+X_{I^{t+1}}^{T}\boldsymbol{\lambda}^t+\mathbf{g}_{I^{t+1}}+\mathbf{f}_{I^{t+1}})\|_{2}&\le
\|\mathbf{b}_{I^{t+1}}^{*}-(\mathbf{b}_{I^{t+1}}^{*}+X_{I^{t+1}}^{T}\boldsymbol{\lambda}^t+\mathbf{g}_{I^{t+1}}+\mathbf{f}_{I^{t+1}})\|_{2} \\
&=\|X_{I^{t+1}}^{T}\boldsymbol{\lambda}^t+\mathbf{g}_{I^{t+1}}+\mathbf{f}_{I^{t+1}}\|_2
\end{align*}
Using the trigonometric inequality:
\[
\|\mathbf{b}_{I^{t+1}}^{t+1}-\mathbf{b}_{I^{t+1}}^{*}\|_2\le
2\|X_{I^{t+1}}^{T}\boldsymbol{\lambda}^t+\mathbf{g}_{I^{t+1}}+\mathbf{f}_{I^{t+1}}\|_2\le
2\|X_{I^{t+1}}^{T}\boldsymbol{\lambda}^t\|_2+2\|\mathbf{g}_{I^{t+1}}\|_2+2\|\mathbf{f}_{I^{t+1}}\|_2
\]
Through the SSS and SSC properties of $X$, we obtain:
\begin{align*}
\|X_{I^{t+1}}^{T}\boldsymbol{\lambda}^t\|_2
&=\|X_{I^{t+1}}^{T}(XX^T+M)^{-1}X(\mathbf{b}^{t+1}-\mathbf{b}^{*})\|_2\\
&=\|X_{I^{t+1}}^{T}(XX^T+M)^{-1}X_{I^t}(\mathbf{b}_{I^t}^{t+1}-\mathbf{b}_{I^t}^{*})\|_2\\
&\le \frac{\Lambda_{k+k^{*}}}{\lambda_{min}(XX^T+M)}\|\mathbf{b}_{I^t}^{t+1}-\mathbf{b}_{I^t}^{*}\|_2\\
&=\frac{\Lambda_{k+k^{*}}}{\lambda_{min}(XX^T+M)}\|\mathbf{b}^{t+1}-\mathbf{b}^{*}\|_2
\end{align*}
According to Bhatia~\citep{bhatia2017consistent}, there is a probability of at least $1-\delta$ that, for any set $S$ of size up to $k+k^*$, we can find a uniform bound:
\[
\|\boldsymbol{\epsilon}_S\|_2\le \sigma \sqrt{k+k^*}\sqrt{1+2e\sqrt{6\log\frac{en}{\delta (k+k^*)}}}\doteq e_0
\]
As for $\|X\boldsymbol{\epsilon}\|_2$, Bhatia~\citep{bhatia2017consistent} gives a consistent bound of
$\|X\boldsymbol{\epsilon}\|_2^2\le 2\sigma^2\|X\|_{F}^{2}\log(\frac{d}{\delta})\le 
2\sigma^2d\Lambda_{n}\log(\frac{d}{\delta})$, and so:
\begin{align*}
\|\mathbf{g}_{I^{t+1}}\|_2
&=\|\boldsymbol{\epsilon}_{I^{t+1}}-X_{I^{t+1}}^T(XX^T+M)^{-1}X\boldsymbol{\epsilon}_{I^{t+1}}\|_2
\le \|\boldsymbol{\epsilon}_{I^{t+1}}\|_2+\|X_{I^{t+1}}^T(XX^T+M)^{-1}X\boldsymbol{\epsilon}_{I^{t+1}}\|_2\\
&\le e_0+\sigma\frac{\sqrt{\Lambda_{k+k^{*}}\Lambda_{n}}}{\lambda_{min}(XX^T+M)}\sqrt{2d\log(\frac{d}{\delta})}\le e_0+\sigma\frac{\sqrt{\Lambda_{k+k^{*}}\Lambda_{n}}}{\lambda_{n}}\sqrt{2d\log(\frac{d}{\delta})}\\
&\le (1+\sqrt{\frac{2d}{n}\log(\frac{d}{\delta})})e_0
\end{align*}
The last inequality holds when $n$ is sufficiently large. Then, we consider $\mathbf{f}_{I^{t+1}}$:
\begin{align*}
\|\mathbf{f}_{I^{t+1}}\|_2&=\|X_{I^{t+1}}^T(XX^T+M)^{-1}M(\mathbf{w}^*-\mathbf{w}_0)\|_2\\
&\le\frac{\sqrt{\Lambda_{k+k^*}}\lambda_{max}(M)}{\lambda_{min}(XX^T+M)}\|\mathbf{w}^*-\mathbf{w}_0\|_2
\end{align*}
We substitute the three calculated terms into the original result to obtain:
\begin{align*}
\|\mathbf{b}^{t+1}-\mathbf{b}^{*}\|_2\le{}&
2\frac{\Lambda_{k+k^{*}}}{\lambda_{min}(XX^T+M)} \|\mathbf{b}^{t}-\mathbf{b}^{*}\|_2+2(1+\sqrt{\frac{2d}{n}\log(\frac{d}{\delta})})e_0\\
&+2\frac{\sqrt{\Lambda_{k+k^*}}\lambda_{max}(M)}{\lambda_{min}(XX^T+M)}\|\mathbf{w}^*-\mathbf{w}_0\|_2
\end{align*}
We let $\eta=2\frac{\Lambda_{k+k^{*}}}{\lambda_{min}(XX^T+M)}$. Because $\mathbf{b}^0=0$:
\begin{align*}
\|\mathbf{b}^{t+1}-\mathbf{b}^{*}\|_2\le{}&
\eta^{t}\|\mathbf{b}^{*}\|_2+\frac{2}{1-\eta}(1+\sqrt{\frac{2d}{n}\log(\frac{d}{\delta})})e_0\\
&+\frac{2}{1-\eta}\frac{\sqrt{\Lambda_{k+k^*}}\lambda_{max}(M)}{\lambda_{min}(XX^T+M)}\|\mathbf{w}^*-\mathbf{w}_0\|_2
\end{align*}
Suppose that $n>d\log(d)$. Then, $1+\sqrt{\frac{2d}{n}\log(\frac{d}{\delta})}=\textit{O}(1)$. From the expression of $e_0$, we have that $e_0=\textit{O}(\sigma\sqrt{(k+k^*)\log\frac{n}{\delta(k+k^*)}})$.
Then, after $T_0=\textit{O}(\log(\frac{\|\mathbf{b}^*\|_2}{\varepsilon}))$, we obtain:
\[
\|\mathbf{b}^{T_0}-\mathbf{b}^*\|_{2}\le \varepsilon + \textit{O}(e_0)+\textit{O}(\frac{\sqrt{\Lambda_{k+k^*}}\lambda_{max}(M)}{\lambda_{min}(XX^T+M)})\|\mathbf{w}^*-\mathbf{w}_0\|_2
\]
\end{proof}

\begin{theorem_apd}
Under the conditions of Theorem 1 and assuming that $\mathbf{x}_i\in \mathbb{R}^d$ are generated from the standard normal distribution, for $k>k^{*}$, it is guaranteed with a probability of at least $1-\delta$ that, for any $\varepsilon,\delta>0$, the current estimation coefficient $\mathbf{w}_{T_{0}}$ satisfies 
$\|\mathbf{w}_{T_{0}}-\mathbf{w}^*\|_2\le \textit{O}(\frac{1}{\sqrt{n}})(\varepsilon+e_0)+\textit{O}(\frac{\sqrt{k+k^*}\lambda_{\max}(M)}{n^{3/2}})\|\mathbf{w}^*-\mathbf{w}_0\|_2$ after $T_{0}=\textit{O}(\log (\frac{\|\mathbf{b}^*\|_2}{\varepsilon}))$ steps.
\end{theorem_apd}

\begin{proof}
\[
\mathbf{w}^t
=(XX^T)^{-1}X(\mathbf{y}-\mathbf{b}^t)
=(XX^T)^{-1}X(X^T\mathbf{w}^*+\mathbf{b}^*+\boldsymbol{\epsilon}-\mathbf{b}^t)
=\mathbf{w}^*+(XX^T)^{-1}X(\boldsymbol{\epsilon}+\mathbf{b}^*-\mathbf{b}^t)
\]
\begin{align*}
\|\mathbf{w}^t-\mathbf{w}^*\|_2
&=\|(XX^T)^{-1}X(\boldsymbol{\epsilon}+\mathbf{b}^*-\mathbf{b}^t)\|_2
\le\frac{1}{\lambda_n}(\|X\boldsymbol{\epsilon}\|_2+\|X(\mathbf{b}^*-\mathbf{b}^t)\|_2)\\
&\le \frac{\sqrt{\Lambda_n}}{\lambda_n}\sigma\sqrt{2d\log(\frac{d}{\delta})}+\frac{1}{\lambda_n}\|X(\mathbf{b}^*-\mathbf{b}^t)\|_2)\\
&\le\frac{\sqrt{\Lambda_n}}{\lambda_n}\sigma\sqrt{2d\log(\frac{d}{\delta})}
+\frac{\sqrt{\Lambda_n}}{\lambda_n}\bigg[\eta^{t}\|\mathbf{b}^{*}\|_2+\frac{2}{1-\eta}(1+\sqrt{\frac{2d}{n}\log(\frac{d}{\delta})})e_0\\
&+\frac{2}{1-\eta}\frac{\sqrt{\Lambda_{k+k^*}}\lambda_{max}(M)}{\lambda_{min}(XX^T+M)}\|\mathbf{w}^*-\mathbf{w}_0\|_2\bigg]
\end{align*}
when $n$ is sufficiently large. By Lemma \ref{SSC_all} and Theorem \ref{SSC_local}, $\sqrt{\Lambda_n}/\lambda_n$ can then be approximated as $\textit{O}(1/\sqrt{n})$ and $\sqrt{\Lambda_{k+k^*}}$ can be approximated as $\textit{O}(\sqrt{k+k^*})$. Then, we have:
\[
\|\mathbf{w}_t-\mathbf{w}^*\|_2\le \textit{O}(\frac{1}{\sqrt{n}})(\varepsilon+e_0)+\textit{O}(\frac{\sqrt{k+k^*}\lambda_{\max}(M)}{n^{3/2}})\|\mathbf{w}^*-\mathbf{w}_0\|_2
\]
\end{proof}

	\begin{theorem_apd}
		Let $X=[\mathbf{x}_1,...,\mathbf{x}_n]\in \mathbb{R}^{d\times n}$ be the given matrix with each $\mathbf{x}_i\sim \mathcal{N}(\mathbf{0},\Sigma)$. Let $\mathbf{y}=X^{T}\mathbf{w}^{*}+\mathbf{b}+\boldsymbol{\epsilon}$ and $\|\mathbf{b}\|_0\le k^{*}$. Also, let $k^{*}\le k$ and suppose $\lim_{n\to \infty}\frac{\lambda_{min}(M)}{n}=\xi$. Then if the following equation holds
		\begin{equation*}
			2\frac{k+k^{*}}{n}(1+3e\sqrt{6\log\frac{en}{k+k^{*}}})<1+\xi
		\end{equation*}
		 and $n\ge\Omega(d+\log\frac{1}{\delta})$. Then, with probability at least $1-\delta$, the data satisfies $2\frac{\Lambda_{k+k^{*}}}{\lambda_{min}(XX^T+M)}<1$, More specifically, after $T_{0}=\textit{O}(\log (\frac{\|\mathbf{b}^*\|_2}{\varepsilon}))$ steps in TRIP algorithm, the estimation coefficient $\mathbf{w}_{T_{0}}$ satisfies 
		 $\|\mathbf{w}_{T_{0}}-\mathbf{w}^*\|_2\le \textit{O}(\frac{1}{\sqrt{n}})(\varepsilon+e_0)+\textit{O}(\frac{\sqrt{k+k^*}\lambda_{\max}(M)}{n^{3/2}})\|\mathbf{w}^*-\mathbf{w}_0\|_2$.
	\end{theorem_apd}
	
	\begin{proof}
		We notice that if $\mathbf{x}\sim\mathcal{N}(\mathbf{0},\Sigma)$, then $\Sigma^{-1/2}\mathbf{x}\sim\mathcal{N}(\mathbf{0},I)$. Thus by Theorem \ref{SSC_local} and Lemma \ref{SSC_all}, with the probability at least $1-\delta$, the data matrix $\tilde{X}=\Sigma^{1/2}X$ satisfies SSC and SSS properties with the following constants
		\begin{align*}
			&\Lambda_{k}\le k(1+3e\sqrt{6\log\frac{en}{k}})+\textit{O}(\sqrt{nd+n\log\frac{1}{\delta}}),\\
			&\lambda_{min}(XX^T)\ge n-(1-2\epsilon)^{-1}\sqrt{cnd+c'n\log \frac{2}{\delta}}.
		\end{align*}
	As seen in Theorem 1, the convergence of TRIP needs to satisfies $2\frac{\Lambda_{k+k^{*}}}{\lambda_{min}(XX^T+M)}<1$. We notice that $\lambda_{min}(XX^T+M)\ge \lambda_{min}(XX^T)+\lambda_{min}(M)$, so the convergence condition can be scaled to $2\Lambda_{k+k^{*}}\le \lambda_{min}(XX^T)+\lambda_{min}(M)$. Using the above bounds, the condition is translated into
	\begin{equation*}
		\underbrace{2\frac{k+k^{*}}{n}(1+3e\sqrt{6\log\frac{en}{k+k^{*}}})}_{(A)}+\underbrace{\textit{O}(\sqrt{\frac{d}{n}+\frac{1}{n}\log\frac{1}{\delta}})}_{(B)}<1+\frac{\lambda_{min}(M)}{n}.
	\end{equation*}
	For $n=\Omega(d+\frac{1}{\delta})$ and suppose $n$ is large enough, the part $(B)$ goes to 0. Also because $\lim_{n\to \infty}\frac{\lambda_{min}(M)}{n}=\xi$, so the condition becomes 
	\begin{equation*}
		2\frac{k+k^{*}}{n}(1+3e\sqrt{6\log\frac{en}{k+k^{*}}})<1+\xi.
	\end{equation*}
	\end{proof}
	
	The condition $2\frac{k+k^{*}}{n}(1+3e\sqrt{6\log\frac{en}{k+k^{*}}})<1+\xi$ seems quite abstract. By approximating $f(t)=2t(1+3e\sqrt{6\log\frac{e}{t}})$ using its second order Taylor's expansion at $t= 1/10$, which is shown in Figure \ref{figure4}. We can give an approximated breakdown point of TRIP algorithm when $\xi$ is not too large, i.e.,
	%  where the approximation error is ignorable.
	\begin{equation}
		k^{*}\le k\le (0.3023-\sqrt{0.0887-0.0040\xi})n.
	\end{equation}
		\begin{figure*}[hp]
		\centering
		\subfigure[]
		{
			\begin{minipage}{5cm}
				\centering
				\includegraphics[width=5cm]{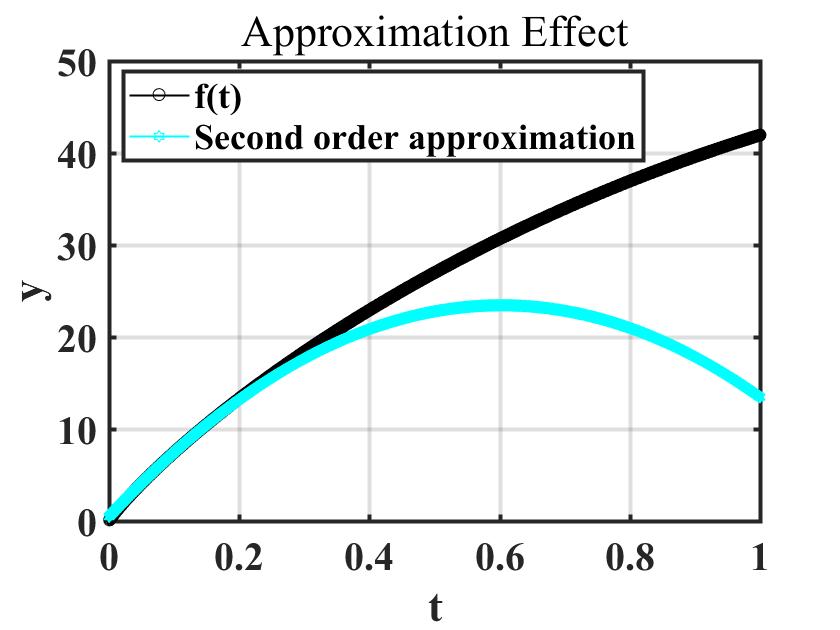}
			\end{minipage}
		}
		\subfigure[]
		{
			\begin{minipage}{5cm}
				\centering
				\includegraphics[width=5cm]{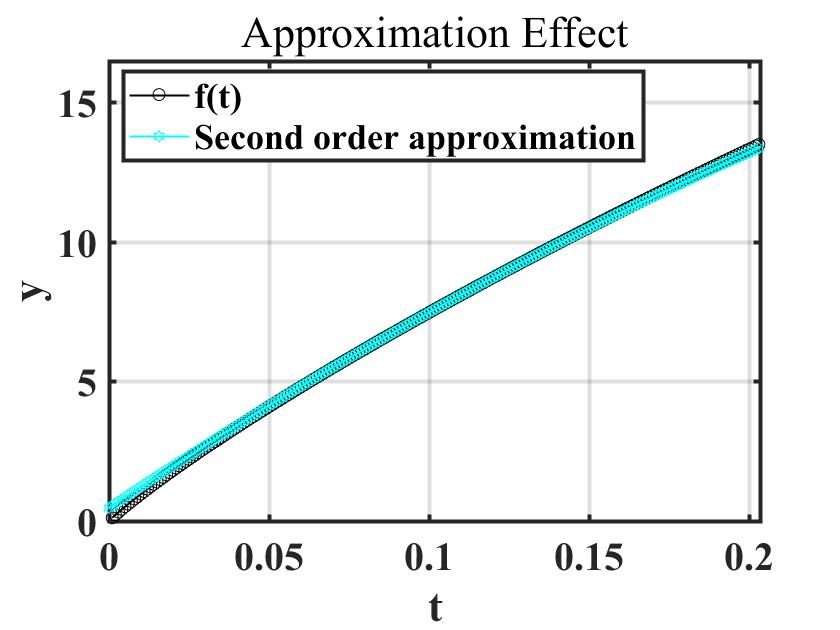}
			\end{minipage}
		}
		
		\caption{(a) The approximation of the second order Taylor's expansion on $[0,1]$. (b) Approximation on the interval $[0,0.2]$. } 
		\label{figure4}
	\end{figure*}

\subsection{Convergence Proof for BRHT}
\subsubsection{Proof of Theorem 5}
\begin{lemma}\label{lemma1}
For any real function $f(x)$:
\begin{equation}
\sup_{x\ge 0} [f(x)+ax] \le \sup_{x\ge 0} [f(x)+bx]
\end{equation}
for any $b \ge a \ge 0$.
\end{lemma}

\begin{proof}
Suppose the lemma does not hold, that is, $a\ge 0, b\ge a$, but 
\[\sup_{x\ge 0} [f(x)+ax] > \sup_{x\ge 0} [f(x)+bx]\]
We select the array $\{x_n\}=\{x_1,x_2,\dots\}$ such that $\lim_{i\to \infty}[f(x_i)+ax_i]=\sup_{x\ge 0} [f(x)+ax]$. Then, we consider the set 
$S\doteq\{f(x_i)+bx_i|x_i \in \{x_n\}\}$. It is easy to see that $\sup S \ge \sup_{x\ge 0} [f(x)+ax]$ and $\sup S \le \sup_{x\ge 0} [f(x)+bx]$. As shown above, however, $\sup_{x\ge 0} [f(x)+ax] > \sup_{x\ge 0} [f(x)+bx]$. This is a contradiction, and Lemma \ref{lemma1} is proved.
\end{proof}

\begin{theorem_apd}
Suppose that the prior of $r_i$ is independently and identically distributed (iid). We consider $\mathbf{w}_t,\mathbf{r}_t=\arg \max_{\mathbf{w}\in \mathbb{R}^{d},\mathbf{r}\in \mathbb{R}_{+}^{n}}M(\mathbf{w},\mathbf{r},\mathbf{b}_{t})$, which is the estimation in the $t^{th}$ iteration step of the BRHT algorithm, and $\mathbf{b}_t= HT_k(\mathbf{y}-X^T\mathbf{w}_{t-1})$ is obtained from the hard thresholding step. Then, we have that $U(\mathbf{w}_t,\mathbf{r}_t,S_{t+1})\ge U(\mathbf{w}_{t-1},\mathbf{r}_{t-1},S_{t})$.
\end{theorem_apd}
\begin{proof}
After obtaining $\mathbf{b}_{t}$ by $\mathbf{b}_{t}=HT_{k}(\mathbf{y}-X^{T}\mathbf{w}_{t-1})$, we consider $M(\mathbf{w}_{t-1},\mathbf{r}_{t-1},\mathbf{b}_{t})$, that is:
\begin{align*} 
M(\mathbf{w}_{t-1},\mathbf{r}_{t-1},\mathbf{b}_{t})=
{}&\log p_{\mathbf{w}}(\mathbf{w}_{t-1})+\sum_{i \in S_{t}}[\log p_{\mathbf{r}}(r_i^{t-1})+r_i^{t-1} \log\ell(y_{i}\mid \mathbf{w}_{t-1},\mathbf{x}_i,\sigma^2)]\\
&+\sum_{j \in [n]\backslash S_{t}}[p_{\mathbf{r}}(r_j^{t-1})+r_j^{t-1}\ell(0)]
\end{align*}
where $\ell(0)$ is the value of the likelihood of $\mathcal{N} (0,\sigma^2)$. This is because, after the hard thresholding step and if $i$ is not chosen from the clean set,' $y_i-b_i^{t}=y_i-(y_i-X^{T}\mathbf{w}_{t-1})=X^{T}\mathbf{w}_{t-1}$. Thus, it can be seen that $\ell(y_{i}-b_i^{t}\mid \mathbf{w}_{t-1},\mathbf{x}_i,\sigma^2)=\ell(\mathbf{x}_i^{T}\mathbf{w}_{t-1}\mid \mathbf{w}_{t-1},\mathbf{x}_i,\sigma^2)=\ell(0)$.
We consider a pseudo-reweighting process (this is just for the convenience of the proof and does not appear in the algorithm, but does not affect the result of the algorithm). We try to maximize $M(\mathbf{w}_{t-1},\mathbf{r},\mathbf{b}_{t})$ by varying $\mathbf{r}$. Because of the independence of $p_{\mathbf{r}}(r_i)$ and the definition of $M(\mathbf{w}_{t-1},\mathbf{r}_{t-1},\mathbf{b}_{t})$, the value of $\mathbf{r}$ in $S_{t}$ is unchanged.
\begin{align*}
\tilde{M}_{t-1}={}&\max_{\mathbf{r}\in\mathbb{R}^n}M(\mathbf{w}_{t-1},\mathbf{r},\mathbf{b}_{t})\\
={}&\log p_{\mathbf{w}}(\mathbf{w}_{t-1})+\sum_{i \in S_{t}}[\log p_{\mathbf{r}}(r_i^{t-1})+r_i^{t-1} \log\ell(y_{i}\mid \mathbf{w}_{t-1},\mathbf{x}_i,\sigma^2)]+kg(0)\\
={}&U(\mathbf{w}_{t-1},\mathbf{r}_{t-1},S_{t})+kg(0)
\end{align*}
where $g(0)$ is defined as $\max_{r_i}[p_{\mathbf{r}}(r_i)+r_i\ell(0)]$. Next, we consider the update of $\mathbf{w}$. Because:
\[\mathbf{w}_t,\mathbf{r}_t=\arg \max_{\mathbf{w}\in \mathbb{R}^{d},\mathbf{r}\in \mathbb{R}_{+}^{n}}M(\mathbf{w},\mathbf{r},\mathbf{b}_{t})\]
it is easy to see that: 
\[M(\mathbf{w}_t,\mathbf{r}_t,\mathbf{b}_{t})\ge
\max_{\mathbf{r}\in\mathbb{R}_{+}^{n}}M(\mathbf{w}_{t-1},\mathbf{r},\mathbf{b}_{t})
=\tilde{M}_{t-1}\]
Finally, we examine $\tilde{M}_{t}=\max_{\mathbf{r}\in\mathbb{R}^n}M(\mathbf{w}_{t},\mathbf{r},\mathbf{b}_{t+1})$. The explicit form of $\tilde{M}_{t}$ can be given by $\tilde{M}_{t-1}$. We compare the $\tilde{M}_{t}$ and $M(\mathbf{w}_t,\mathbf{r}_t,\mathbf{b}_{t})$:
\begin{align*}
\tilde{M}_{t}-M(\mathbf{w}_t,\mathbf{r}_t,\mathbf{b}_{t})
={}&\log p_{\mathbf{w}}(\mathbf{w}_{t})+\sum_{i \in S_{t+1}}[\log p_{\mathbf{r}}(r_i^{t})+r_i^{t} \log\ell(y_{i}\mid \mathbf{w}_{t},\mathbf{x}_i,\sigma^2)]+kg(0)\\
&-\{\log p_{\mathbf{w}}(\mathbf{w}_{t})+\sum_{j \in S_{t}}[\log p_{\mathbf{r}}(r_j^{t})+r_j^{t} \log\ell(y_{j}\mid \mathbf{w}_{t},\mathbf{x}_j,\sigma^2)]\\
&+\sum_{j \in [n]\backslash S_{t}}[p_{\mathbf{r}}(r_j^{t})+r_j^{t} \log\ell(y_{j}-b_j^{t}\mid \mathbf{w}_{t},\mathbf{x}_j,\sigma^2)]\}
\end{align*}
\begin{align*}
={}&\sum_{i\in S_{t+1}\backslash S_{t}}[\log p_{\mathbf{r}}(r_i^{t})+r_i^{t} \log\ell(y_{i}\mid\mathbf{w}_{t},\mathbf{x}_i,\sigma^2)]
-\sum_{j\in S_{t}\backslash S_{t+1}}[\log p_{\mathbf{r}}(r_j^{t})+r_j^{t} \log\ell(y_{j}\mid\mathbf{w}_{t},\mathbf{x}_j,\sigma^2)]\\
&+kg(0)-\sum_{j \in [n]\backslash S_{t}}[p_{\mathbf{r}}(r_j^{t})+r_j^{t} \log\ell(y_{j}-b_j^{t}\mid\mathbf{w}_{t},\mathbf{x}_j,\sigma^2)]
\end{align*}
Following the hard thresholding step, $\forall i \in S_{t+1}\backslash S_{t}$ and $\forall j \in S_{t}\backslash S_{t+1}$, $|y_i-\mathbf{x}_i^{T}\mathbf{w}_t|\le |y_j-\mathbf{x}_j^{T}\mathbf{w}_t|$, and so $\log\ell(y_{i}\mid\mathbf{w}_{t},\mathbf{x}_i,\sigma^2)\ge \log\ell(y_{j}\mid\mathbf{w}_{t},\mathbf{x}_j,\sigma^2)$. By Lemma \ref{lemma1}, we have that: 
\begin{equation*}
\log p_{\mathbf{r}}(r_i^{t})+r_i^{t} \log\ell(y_{i}\mid\mathbf{w}_{t},\mathbf{x}_i,\sigma^2)
\ge \log p_{\mathbf{r}}(r_j^{t})+r_j^{t} \log\ell(y_{j}\mid \mathbf{w}_{t},\mathbf{x}_j,\sigma^2)
\end{equation*}
and because $\forall j \in [n]\backslash S_{t}$, $\log\ell(y_{j}-b_{j}^{t}\mid \mathbf{w}_{t},\mathbf{x}_j,\sigma^2)\le \ell(0)$, we have:
\begin{equation*}
\log p_{\mathbf{r}}(r_j^{t})+r_j^{t}\log\ell(y_{j}-b_{j}^{t}\mid \mathbf{w}_{t},\mathbf{x}_j,\sigma^2) \le g(0)
\end{equation*} 
This proves that: 
\begin{equation*}
\tilde{M}_{t} \ge M(\mathbf{w}_t,\mathbf{r}_t,\mathbf{b}_{t}) 
\end{equation*}

Note that $M(\mathbf{w}_t,\mathbf{r}_t,\mathbf{b}_{t})\ge \tilde{M}_{t-1}$. Using the expressions for $\tilde{M}_{t}$ and $\tilde{M}_{t-1}$:
\begin{align*}
U(\mathbf{w}_{t},\mathbf{r}_{t},S_{t+1})+kg(0) \ge& U(\mathbf{w}_{t-1},\mathbf{r}_{t-1},S_{t})+kg(0)\\
U(\mathbf{w}_{t},\mathbf{r}_{t},S_{t+1})\ge&U(\mathbf{w}_{t-1},\mathbf{r}_{t-1},S_{t})
\end{align*}
\end{proof}

\subsubsection{Proof of Theorems 6 and 7}

To prove Theorem 6, we require a certain assumption. We will show that this assumption is reasonable through a brief description in Appendix C.3.2.

\theoremstyle{plain} 
\newtheorem{assumption}{Assumption}

\begin{assumption}\label{assumption1}
	Let $X$ be the given data matrix and  $\mathbf{y}=X^{T}\mathbf{w}^{*}+\mathbf{b}^{*}+\boldsymbol{\epsilon}$ be the output. For any specific positive semi-definite matrix $M$, there exist $l>0$ and $0<\gamma\le1+\epsilon$, where $\epsilon$ is a small positive number, that for any estimation $\hat{\mathbf{b}}$ of $\mathbf{b}^{*}$, and let $I_{\hat{\mathbf{b}}}=supp(\hat{\mathbf{b}})\cup supp(\mathbf{b}^*)$, it holds that
	\[
	u_1=\|\boldsymbol{\epsilon}_{I_{\hat{\mathbf{b}}}}+X_{I_{\hat{\mathbf{b}}}}^T(\mathbf{w}^{*}-\mathbf{w}_1)\|_2\le	\gamma\|\boldsymbol{\epsilon}_{I_{\hat{\mathbf{b}}}}+X_{I_{\hat{\mathbf{b}}}}^T(\mathbf{w}^{*}-\mathbf{w}_2)\|_2=\gamma u_2
	\]
	where $\mathbf{w}_1$ and $\mathbf{w}_2$ are obtained from:
	
	\[
	\mathbf{w}_1=VBEM(X,\mathbf{y}-\hat{\mathbf{b}},p_{\mathbf{r}}(\mathbf{r}),p_{\mathbf{w}}(\mathbf{w}))
	\]
	
	\[
	\mathbf{w}_2 =\arg\min_{\mathbf{w} \in \mathbb{R}^d} \sum_{i=1}^{n}\|y_i-\hat{b}_i-\mathbf{x}_i^{T}\mathbf{w}\|^2+(\mathbf{w}-\mathbf{w}_0)^{T}M(\mathbf{w}-\mathbf{w}_0)
	\]
	
	and $p_{\mathbf{w}}(\mathbf{w})=\mathcal{N} (\mathbf{w}_0,l\sigma^2 M^{-1})$
\end{assumption}
This assumption can be easily understood as making the Bayesian reweighting regression more robust and accurate than simple regression, thus providing a more reliable solution in each iteration step of the BRHT algorithm. This can be explained from the following two aspects: 1) the Bayesian reweighting regression adds smaller weights to points with large deviations, so the regression is less affected by outliers, especially when the estimation $\hat{\mathbf{b}}$ is not very accurate. 2) By considering the robustness of the Bayesian reweighting regression, smaller prior weights are required to recover the true coefficient. Thus, $\mathbf{w}_2$ is closer to the true coefficient $\mathbf{w}^*$ than $\mathbf{w}_1$, which is reflected in the prior shrinkage coefficient $t$ and the error shrinkage coefficient $\gamma$. 

\begin{theorem_apd}
Consider a data matrix $X$ and a specific positive semi-definite matrix $M$ satisfying the SSC and SSS properties such that $2\frac{\Lambda_{k+k^{*}}}{\lambda_{min}(XX^T+M)}<1$. Then, there exist $l>0$ and $0<\gamma\le1+\epsilon$, where $\epsilon$ is a small number, such that if $k>k^{*}$ and $\Sigma_0$ in the prior $p_{\mathbf{w}}(\mathbf{w})$ is $l \sigma^2 M^{-1}$, it is guaranteed with a probability of at least $1-\delta$ that, for any $\varepsilon,\delta>0$, $\|\mathbf{b}^{T_0}-\mathbf{b}^*\|_{2}\le \varepsilon + \textit{O}(e_0)+\textit{O}(\frac{\sqrt{\Lambda_{k+k^*}}\lambda_{max}(M)}{\lambda_{min}(XX^T+M)})\gamma\|\mathbf{w}^*-\mathbf{w}_0\|_2$ after $T_{0}=\textit{O}(\log (\frac{\gamma\|\mathbf{b}^*\|_2}{\varepsilon}))$ iterations of BRHT, where $e_0=\textit{O}(\sigma\sqrt{(k+k^*)\log\frac{n}{\delta(k+k^*)}})$ under the normal design. 
\end{theorem_apd}

\begin{proof}
The iteration step of the BRHT algorithm is:
\[
\mathbf{b}^{t+1}\gets HT_k(\mathbf{y}-X^{T}\mathbf{w}^t)
\]
where $\mathbf{w}^t= VBEM(X,\mathbf{y}-\mathbf{b}^{t},p_{\mathbf{r}}(\mathbf{r}),p_{\mathbf{w}}(\mathbf{w}))$ and:
\begin{align*}
\|\mathbf{b}_{I^{t+1}}^{t+1}-(\mathbf{y}_{I^{t+1}}-X^{T}_{I^{t+1}}\mathbf{w}^t)\|_{2}
&\le 
\|\mathbf{b}_{I^{t+1}}^{*}-(\mathbf{y}_{I^{t+1}}-X_{I^{t+1}}^{T}\mathbf{w}^t)\|_2\\
&=
\|\mathbf{b}_{I^{t+1}}^{*}-(\mathbf{b}_{I^{t+1}}^{*}+\mathbf{\epsilon}_{I_{t+1}}+X_{I_{t+1}}^T(\mathbf{w}^{*}-\mathbf{w}^t))\|_{2}\\
&=\|\mathbf{\epsilon}_{I_{t+1}}+X_{I_{t+1}}^T(\mathbf{w}^{*}-\mathbf{w}^t)\|_2
\end{align*}
By defining $\hat{\mathbf{w}}^t=(XX^T+M)^{-1}(X(\mathbf{y}-\mathbf{b}^t)+M\mathbf{w}_0)$ and using the trigonometric inequality, we obtain:
\begin{align*}
\|\mathbf{b}_{I^{t+1}}^{t+1}-\mathbf{b}_{I^{t+1}}^{*}\|_2
&\le
2\|\mathbf{\epsilon}_{I_{t+1}}+X_{I_{t+1}}^T(\mathbf{w}^{*}-\mathbf{w}^t)\|_2\\
&\le 
2\gamma\|\mathbf{\epsilon}_{I_{t+1}}+X_{I_{t+1}}^T(\mathbf{w}^{*}-\hat{\mathbf{w}}^t)\|_2\\
&=
2\gamma\|X_{I^{t+1}}^{T}\mathbf{\lambda}^t+\mathbf{g}_{I^{t+1}}+\mathbf{f}_{I^{t+1}}\|_2
\end{align*}
The second inequality holds because of assumption 1. $\mathbf{\lambda}^t$, $\mathbf{g}$, $\mathbf{f}$ have the same meaning as in Theorem 1. Therefore, through the same proof procedure as for Theorem 1, the above inequality can be finally transformed into the following formula:
\begin{align*}
\|\mathbf{b}^{t+1}-\mathbf{b}^{*}\|_2\le{}&
2\gamma\frac{\Lambda_{k+k^{*}}}{\lambda_{min}(XX^T+M)} \|\mathbf{b}^{t}-\mathbf{b}^{*}\|_2+2\gamma(1+\sqrt{\frac{2d}{n}\log(\frac{d}{\delta})})e_0\\
&+2\gamma\frac{\sqrt{\Lambda_{k+k^*}}\lambda_{max}(M)}{\lambda_{min}(XX^T+M)}\|\mathbf{w}^*-\mathbf{w}_0\|_2
\end{align*}
We let $\eta=2\gamma\frac{\Lambda_{k+k^{*}}}{\lambda_{min}(XX^T+M)}$, and because $\mathbf{b}^0=0$, we can write:
\begin{align*}
\|\mathbf{b}^{t+1}-\mathbf{b}^{*}\|_2\le{}&
\eta^{t}\|\mathbf{b}^{*}\|_2+\frac{2\gamma}{1-\eta}(1+\sqrt{\frac{2d}{n}\log(\frac{d}{\delta})})e_0\\
&+\frac{2\gamma}{1-\eta}\frac{\sqrt{\Lambda_{k+k^*}}\lambda_{max}(M)}{\lambda_{min}(XX^T+M)}\|\mathbf{w}^*-\mathbf{w}_0\|_2
\end{align*}
Suppose that $n>d\log(d)$. Then, $1+\sqrt{\frac{2d}{n}\log(\frac{d}{\delta})}=\textit{O}(1)$. From the expression for $e_0$, we have that $e_0=\textit{O}(\sigma\sqrt{(k+k^*)\log\frac{n}{\delta(k+k^*)}})$.
Then, after $T_0=\textit{O}(\log(\frac{\|\mathbf{b}^*\|_2}{\varepsilon}))$, we have:
\[
\|\mathbf{b}^{T_0}-\mathbf{b}^*\|_{2}\le \varepsilon + \textit{O}(e_0)+\textit{O}(\frac{\sqrt{\Lambda_{k+k^*}}\lambda_{max}(M)}{\lambda_{min}(XX^T+M)})\gamma\|\mathbf{w}^*-\mathbf{w}_0\|_2
\]
\end{proof}

\begin{theorem_apd}
Under the conditions of Theorem 6 and assuming that $\mathbf{x}_i\in \mathbb{R}^d$ are generated from the standard normal distribution, there exist $\alpha>0$ and $0<\gamma\le1+\epsilon$, where $\epsilon$ is a small number, such that if $k>k^{*}$ and $\Sigma_0$ in the prior $p_{\mathbf{w}}(\mathbf{w})$ is $\alpha \sigma^2 M^{-1}$, it is guaranteed with a probability of at least $1-\delta$ that, for any $\varepsilon,\delta>0$, the current estimation coefficient $\mathbf{w}_{T_{0}}$ satisfies 
$\|\mathbf{w}_{T_{0}}-\mathbf{w}^*\|_2\le \textit{O}(\frac{1}{\sqrt{n}})(\varepsilon+e_0)+\textit{O}(\frac{\sqrt{k+k^*}\lambda_{\max}(M)}{n^{3/2}})\gamma\|\mathbf{w}^*-\mathbf{w}_0\|_2$ after $T_{0}=\textit{O}(\log (\frac{\gamma\|\mathbf{b}^*\|_2}{\varepsilon}))$ steps.
\end{theorem_apd}
The proof of Theorem 7 is the same as that for Theorem 2, so it is omitted here.
\subsubsection{Rationality of Assumption 1}
In this section, we use some simulations to check whether assumption 1 is true in the iteration of the BRHT algorithm. For this problem, we choose some special $M$ under AAA to make our description more representative. For each corruption ratio, we choose $M$ so as to achieve the minimum fitting error $\|\mathbf{w}_t-\mathbf{w}^*\|_2$ in the TRIP algorithm. We then find a prior shrinkage coefficient $l$ for this $M$ and simulate the BRHT algorithm to show that there exists an error shrinkage coefficient $\gamma$ such that the following formula holds in all iterative steps:
\[
u_{1t}=\|\boldsymbol{\epsilon}_{I_{t}}+X_{I_{t}}^T(\mathbf{w}^{*}-\mathbf{w}_{1t})\|_2\le \gamma\|\boldsymbol{\epsilon}_{I_{t}}+X_{I_{t}}^T(\mathbf{w}^{*}-\mathbf{w}_{2t})\|_2=\gamma u_{2t}
\]
where $\mathbf{w}_{1t}$ and $\mathbf{w}_{2t}$ are obtained from:
\[
\mathbf{w}_{1t}=VBEM(X,\mathbf{y}-\mathbf{b}^{t},p_{\mathbf{r}}(\mathbf{r}),p_{\mathbf{w}}(\mathbf{w}))
\]

\[
\mathbf{w}_{2t} =\arg\min_{\mathbf{w} \in \mathbb{R}^d} \sum_{i=1}^{n}\|y_i-b^t_i-\mathbf{x}_i^{T}\mathbf{w}\|^2+(\mathbf{w}-\mathbf{w}_0)^{T}M(\mathbf{w}-\mathbf{w}_0)
\]

and $p_{\mathbf{w}}(\mathbf{w})=\mathcal{N} (\mathbf{w}_0,l \sigma^2 M^{-1})$. For the chosen corruption ratio, we take eight evenly spaced points from 0.2 to 0.55. This is because the CRR method collapses when the corruption ratio exceeds 0.2, so including the prior become very important at this time. By selecting an appropriate prior shrinkage coefficient $l$ for different $M$, the overall result is as shown in Figure \ref{figure3}.

\begin{figure*}[htbp]
\centering
\subfigure[]
{
\begin{minipage}{3cm}
\centering
\includegraphics[width=3cm]{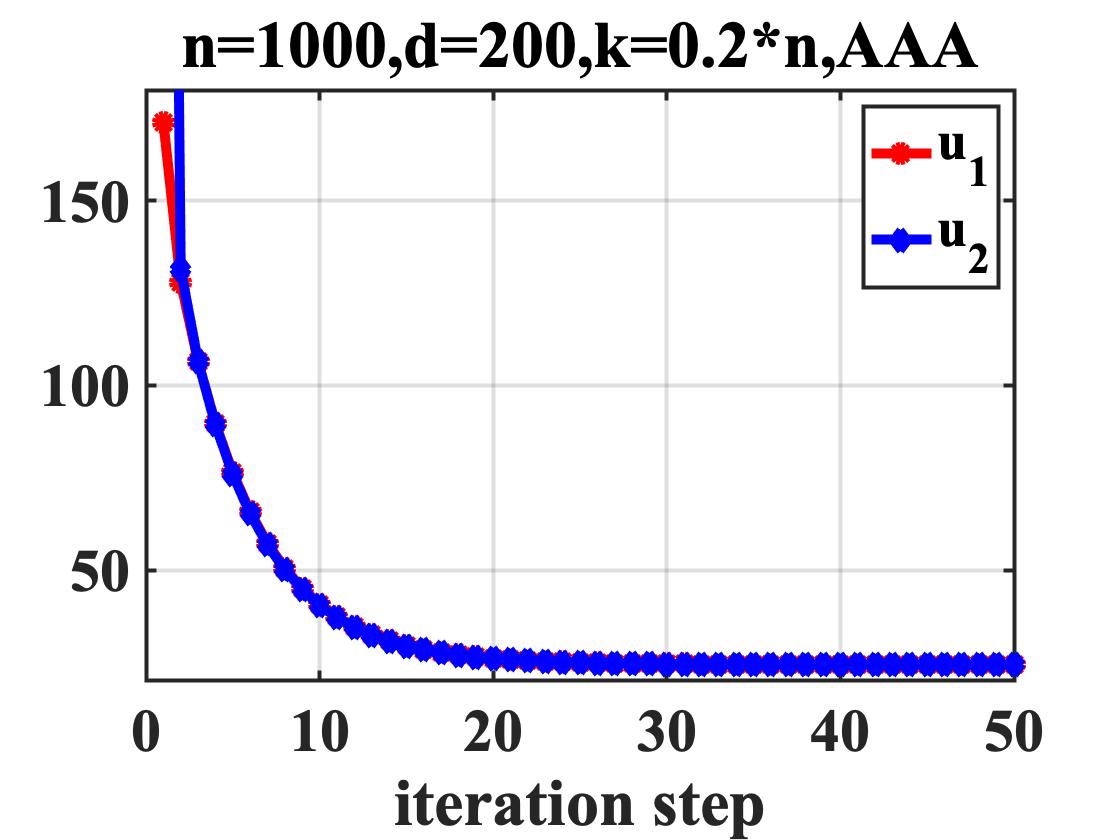}
\end{minipage}
}
\subfigure[]
{
\begin{minipage}{3cm}
\centering
\includegraphics[width=3cm]{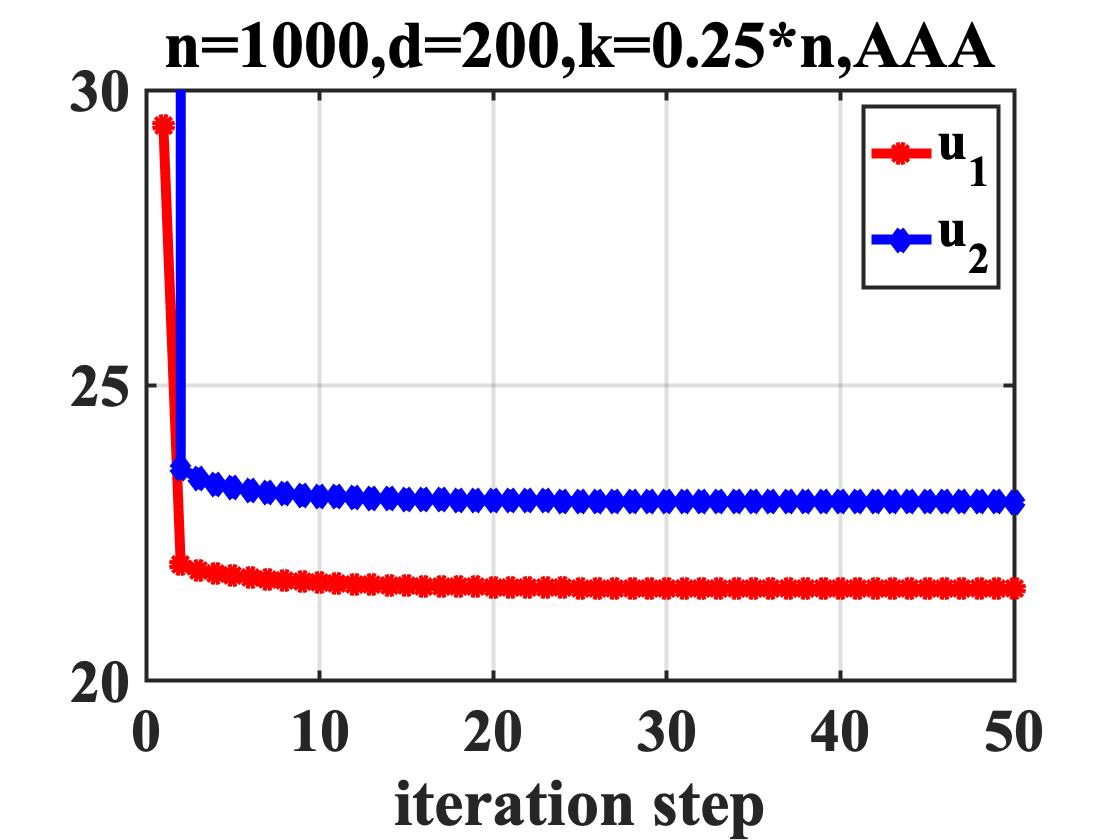}
\end{minipage}
}
\subfigure[]
{
\begin{minipage}{3cm}
\centering
\includegraphics[width=3cm]{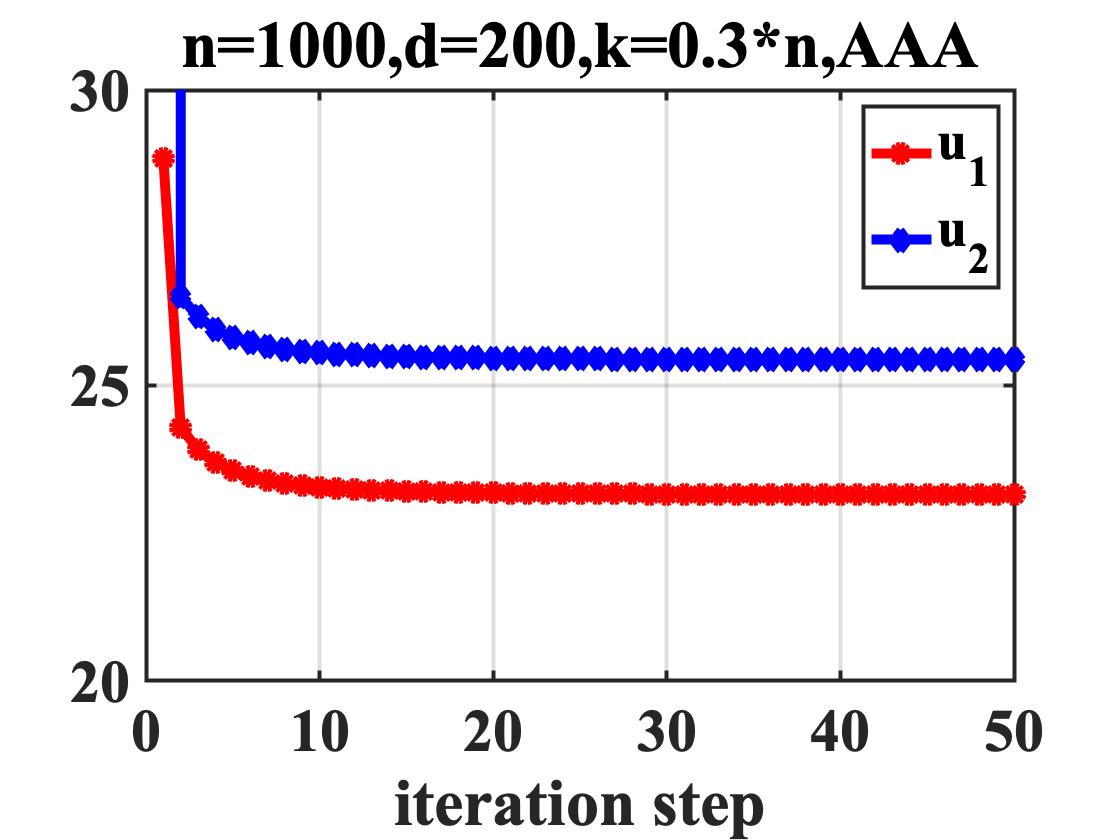}
\end{minipage}
}
\subfigure[]
{
\begin{minipage}{3cm}
\centering
\includegraphics[width=3cm]{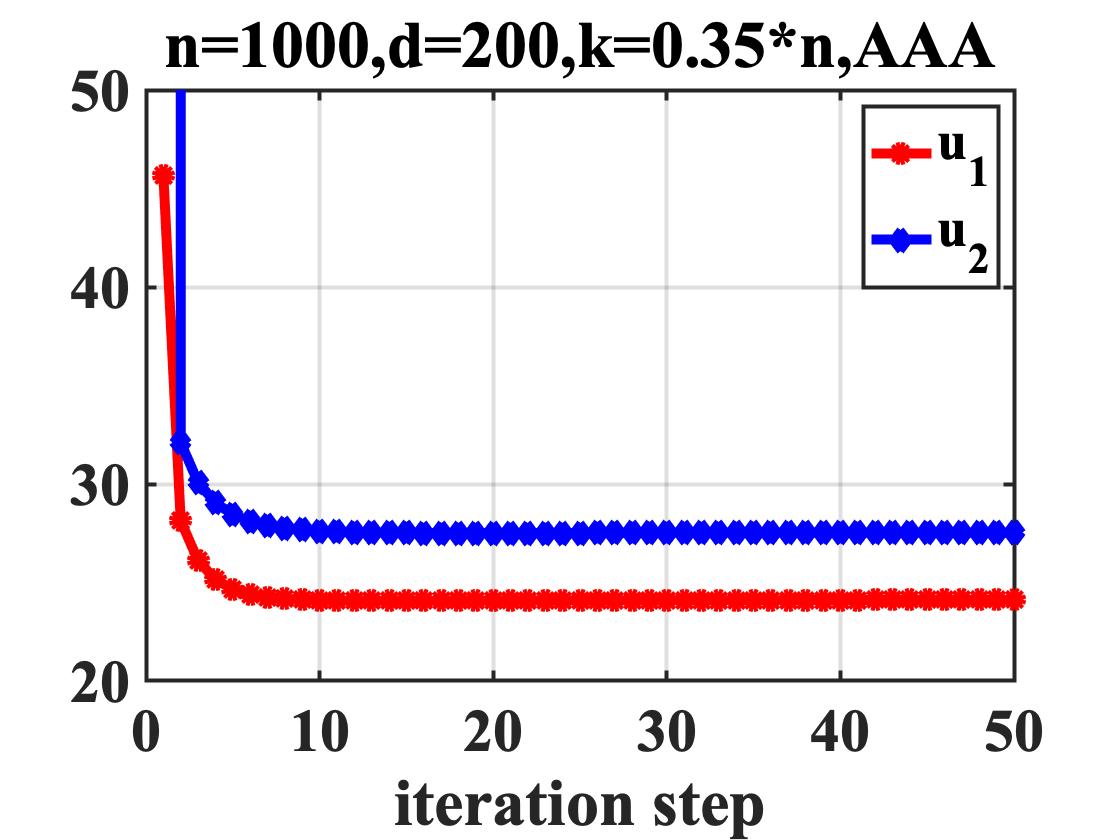}
\end{minipage}
}
\\
\subfigure[]
{
\begin{minipage}{3cm}
\centering
\includegraphics[width=3cm]{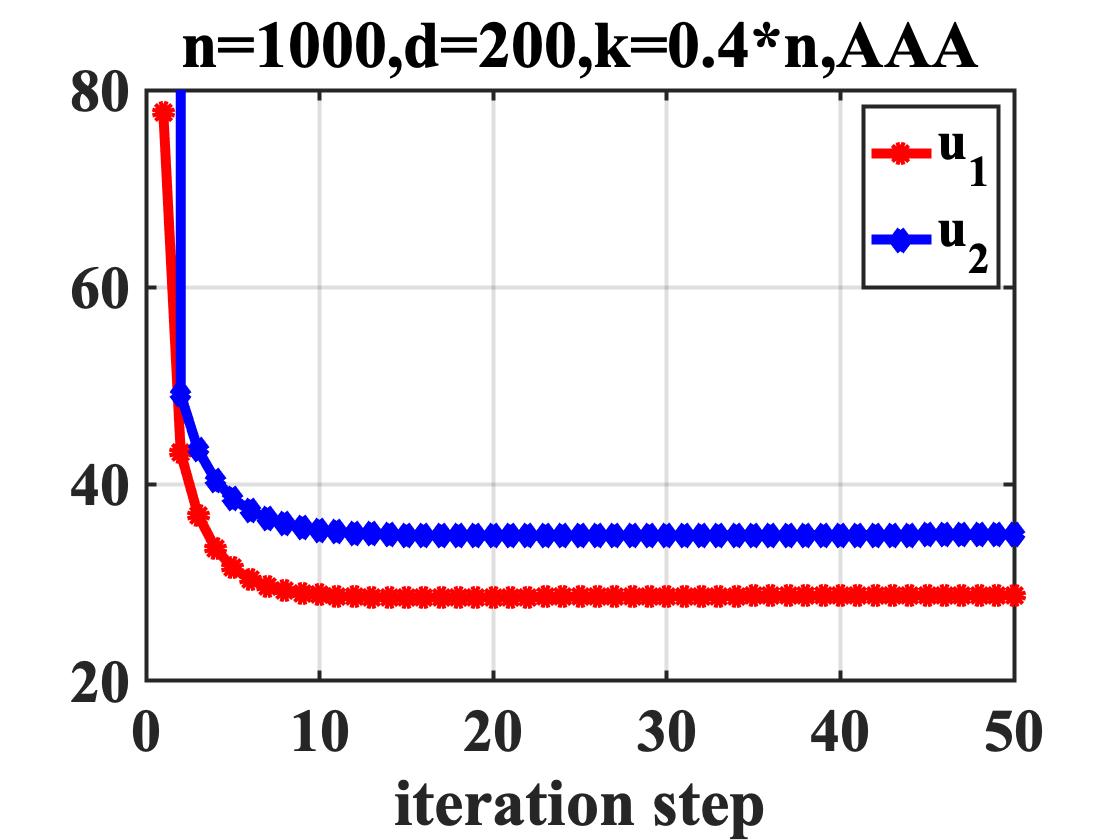}
\end{minipage}
}
\subfigure[]
{
\begin{minipage}{3cm}
\centering
\includegraphics[width=3cm]{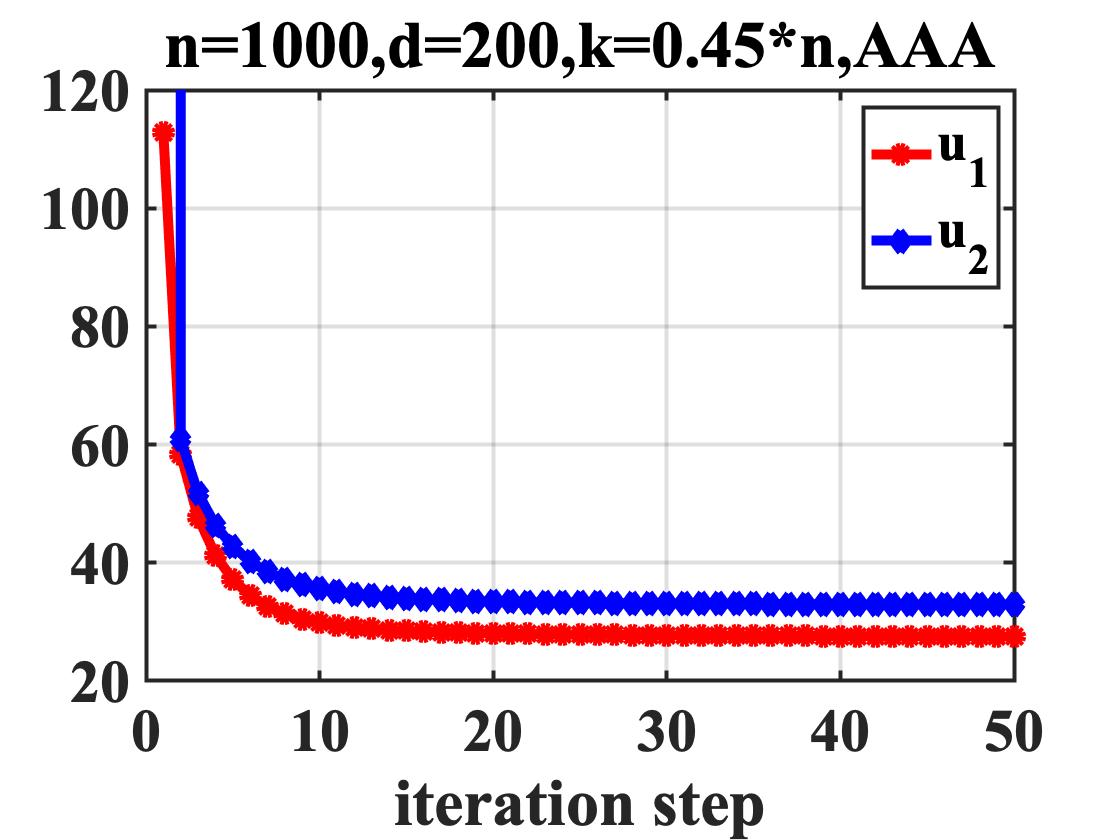}
\end{minipage}
}
\subfigure[]
{
\begin{minipage}{3cm}
\centering
\includegraphics[width=3cm]{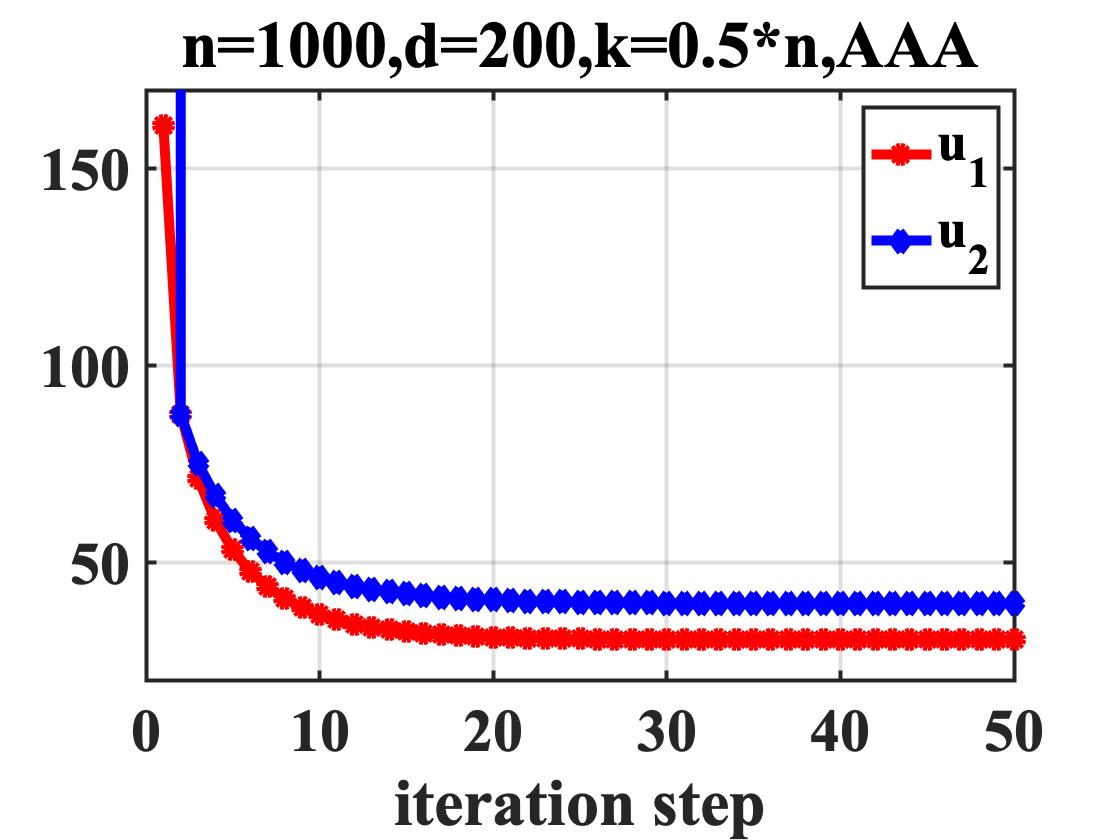}
\end{minipage}
}
\subfigure[]
{
\begin{minipage}{3cm}
\centering
\includegraphics[width=3cm]{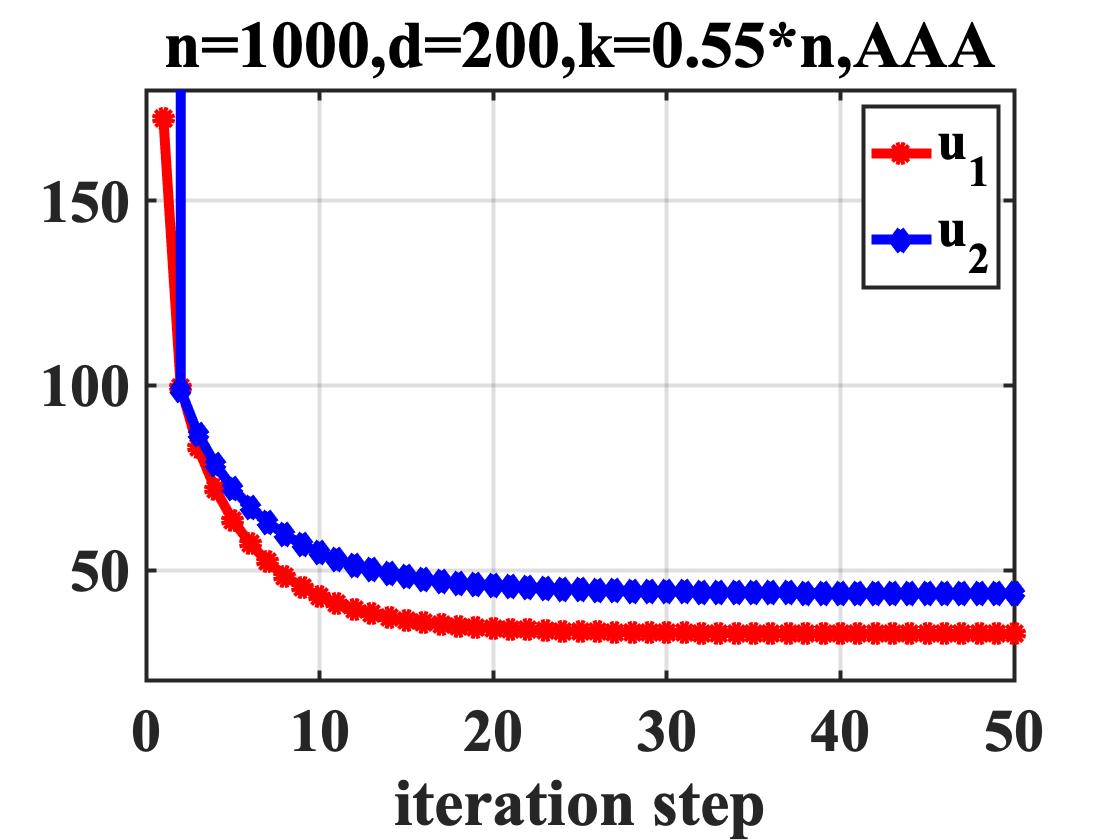}
\end{minipage}
}
\caption{Variation trends of $u_{1t}$ and $u_{2t}$ during the iteration process.} 
\label{figure3}
\end{figure*}

\spacingset{1.5}
We can find $\gamma$ for each $M$ by calculating $\max (u_{1t}/u_{2t})$, where $t=1,...$ are the iterative steps until convergence. The results are presented in Table \ref{tabel1}. For most corruption rates, $\gamma<1$. However, there are still some cases where $\gamma$ is greater than 1. This is because, in the iteration process, there are few steps in which $u_{1t}$ and $u_{2t}$ are very close, while in most cases they are well separated. We use another criterion, $mean (u_{1t}/u_{2t})$, to show this phenomenon; the results are presented in Table \ref{tabel2}. We can see that, as the corruption ratio increases, $mean (u_{1t}/u_{2t})$ basically exhibits a downward trend and all values are less than $1$. This indicates that the performance of BRHT is usually better than that suggested by the theory. This experiment can be used to explain why BRHT usually outperforms TRIP, even when TRIP uses the optimal parameters. 

\spacingset{1}
\begin{table}
\centering
\begin{tabularx}{14em}% 
{|*{2}{>{\centering\arraybackslash}X|}}
\hline
Corruption rate & $\gamma$ \\ \hline
0.2 & 1.007 \\ \hline
0.25 & 0.937 \\ \hline
0.3 & 0.917 \\ \hline
0.35 & 0.876 \\ \hline
0.4 & 0.879 \\ \hline
0.45 & 0.958 \\ \hline
0.5 & 1.000 \\ \hline
0.55 & 1.008 \\ \hline
\end{tabularx}
\caption{Calculated $\gamma$ for each corruption rate} 
\label{tabel1}
\end{table}

\begin{table}
\centering
\begin{tabularx}{14em}%
{|*{2}{>{\centering\arraybackslash}X|}}
\hline
Corruption rate & mean($u_{1t}/u_{2t}$) \\ \hline
0.2 & 0.989 \\ \hline
0.25 & 0.917 \\ \hline
0.3 & 0.892 \\ \hline
0.35 & 0.857 \\ \hline
0.4 & 0.806 \\ \hline
0.45 & 0.827 \\ \hline
0.5 & 0.772 \\ \hline
0.55 & 0.758 \\ \hline
\end{tabularx}
\caption{Average error ratio for each corruption rate} 
\label{tabel2}
\end{table}

\spacingset{1.5}

\section{Choice of Hyperparameters in BRHT}
In the BRHT algorithm, the most important parameter for model performance is the parameter in the weight prior $p_{\mathbf{r}}(\mathbf{r})$. According to assumption 1, we must ensure that the Bayesian reweighting regression provides a more robust and accurate solution than traditional least-squares regression. This requires the weight ${E}_{q(\mathbf{r})}(\mathbf{r})$ in the variational M step of the VBEM algorithm to be relatively insensitive to $\mathbb{E}_{q(\mathbf{w})}[\log \ell(y_{i}\mid\mathbf{w},\mathbf{x}_i,\sigma^2)]$, or very few points will have large weights and others will have little impact on the estimates. A relatively sensitive weight will lead to bias, as only a few points of information will be used, and the effects will be even worse when some outliers have not been detected. Additionally, the weight cannot be too stable, or BRHT will have almost the same performance as TRIP. Here, we present a useful way to determine the hyperparameters so that all uncorrupted points have relatively large weights when the regression result is correct.

Consider the variational E step in the VBEM method. We have:
\begin{equation*}
q(r_i)\propto exp\{\log p_{\mathbf{r}}(r_i)+r_i\mathbb{E}_{q(\mathbf{w})}[\log \ell(y_{i}\mid \mathbf{w},\mathbf{x}_i,\sigma^2)]\}
\end{equation*}

We use the true likelihood $\log \ell(y_{i}\mid \mathbf{w}^*,\mathbf{x}_i,\sigma^2)$ to replace $\mathbb{E}_{q(\mathbf{w})}[\log \ell(y_{i}\mid \mathbf{w},\mathbf{x}_i,\sigma^2)]$, and we find that:
\[
\log \ell(y_{i}\mid \mathbf{w}^*,\mathbf{x}_i,\sigma^2)=-\frac{1}{2\sigma^2}(y_{i}-\mathbf{x}_i^T\mathbf{w}^{*})^2-\frac{1}{2}\log(2\pi\sigma^2)
\]
If $y_i$ is not corrupted, then $\frac{1}{\sigma^2}(y_{i}-\mathbf{x}_i^T\mathbf{w}^{*})^2=\frac{1}{\sigma^2}\epsilon_i^2\le \chi^2(0.95)$ holds with at least 95\% probability, where $\chi^2(0.95)$ is the $95$\% quantile of the $\chi^2$ distribution with 1 degree of freedom. Under the above condition, with at least $95$\% probability, it is easy to see that: 
\[
-\frac{1}{2}\chi^2(0.95)-\frac{1}{2}\log(2\pi\sigma^2)
\le
\log \ell(y_{i}\mid \mathbf{w}^*,\mathbf{x}_i,\sigma^2)\le -\frac{1}{2}\log(2\pi\sigma^2)
\]

Here, we define two posterior distributions of weights in the extreme case where all points fit well or deviate greatly in the true regression model:
\begin{align*}
&q_{1}(r_i)\propto exp[\log p_{\mathbf{r}}(r_i)+r_i(-\frac{1}{2}\log(2\pi\sigma^2))]\\
&q_{2}(r_i)\propto exp[\log p_{\mathbf{r}}(r_i)+r_i(-\frac{1}{2}\chi^2(0.95)-\frac{1}{2}\log(2\pi\sigma^2))]
\end{align*}
Then, the hyperparameter in the weight prior $p_{\mathbf{r}}(\mathbf{r})$ is determined by the following rule:
\[
\mathbb{E}_{q_{2}(r_i)}(r_i)\ge \beta\mathbb{E}_{q_{1}(r_i)}(r_i)
\]
In this paper, $\beta=\frac{1}{2}$. The parameter $\sigma^2$ can be replaced by a robust estimate such as the M estimator. After the weight prior $p_{\mathbf{r}}(\mathbf{r})$ has been determined, the hyperparameter $\Sigma$ in prior $p_{\mathbf{w}}(\mathbf{w})$ can be selected by cross-validation using $\Sigma$ in the specific form $\Sigma=sI$.

\section{Additional Experimental Results}\label{extra_result}
In this section, we give more experimental results of TRIP and BRHT in comparison with alternative methods. We also show the robustness of our methods under other attacks. First, We compare TRIP and BRHT with the TORRENT method proposed by Bhatia et al.~\cite{bhatia2015robust} on both OAA and AAA. TORRENT can resist AAA when the white noise $\boldsymbol{\epsilon}$ is not considered in the model. In order to evaluate the influence of white noise on robust regression, the true data are generated in two ways, one with white noise ($y_{i}=\mathbf{x}_{i}^T{w}^*+\epsilon_{i}$) and the other without white noise ($y_{i}=\mathbf{x}_{i}^T{w}^*$). Other settings are the same as those in Section 6. The experimental results are shown in Figure \ref{extra_exp1}. Under these two attacks, the performance of TORRENT algorithm is very consistent with that of CRR in both noisy and noiseless settings. TORRENT performs slightly better than CRR in the absence of white noise, as shown in Figure \ref{extra_exp1}(e). However, both CRR and TORRENT perform poorly under AAA. It can be seen that the TRIP and BRHT algorithms are very robust in all cases. 
\begin{figure*}[tbp]
\centering
\subfigure[]
{
\begin{minipage}{3cm}
\centering
\includegraphics[width=3cm]{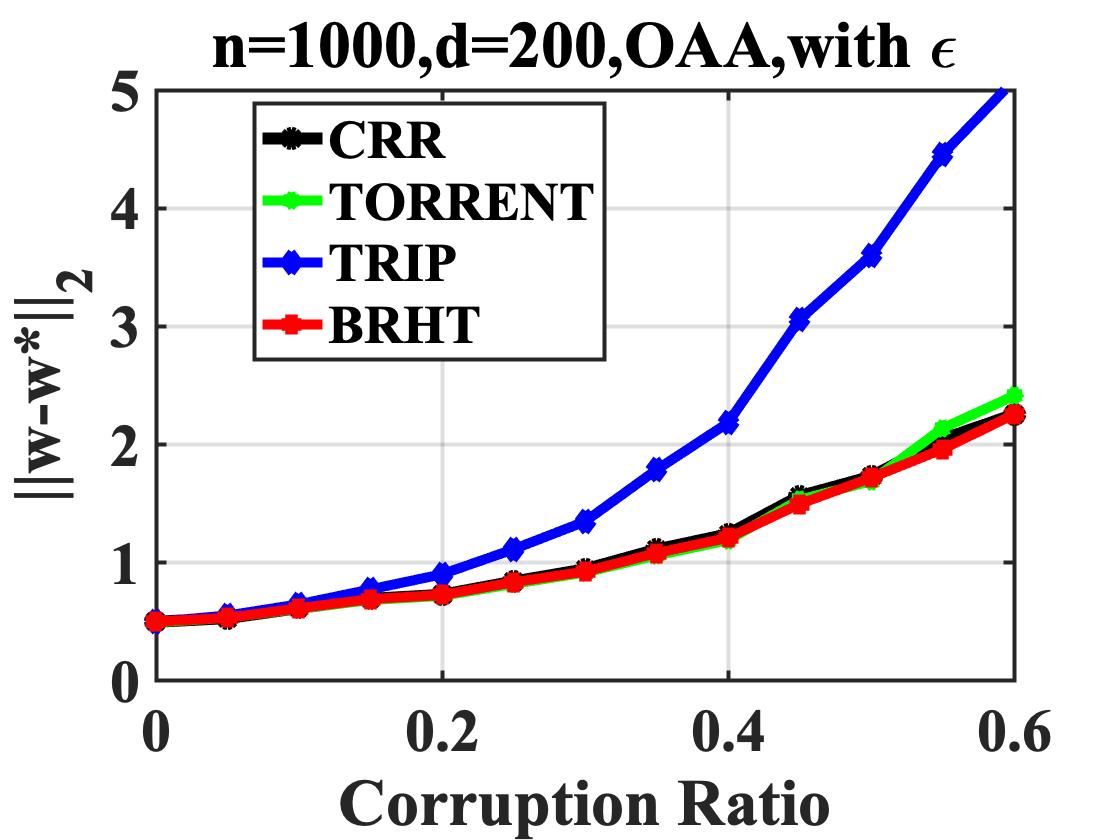}
\end{minipage}
}
\subfigure[]
{
\begin{minipage}{3cm}
\centering
\includegraphics[width=3cm]{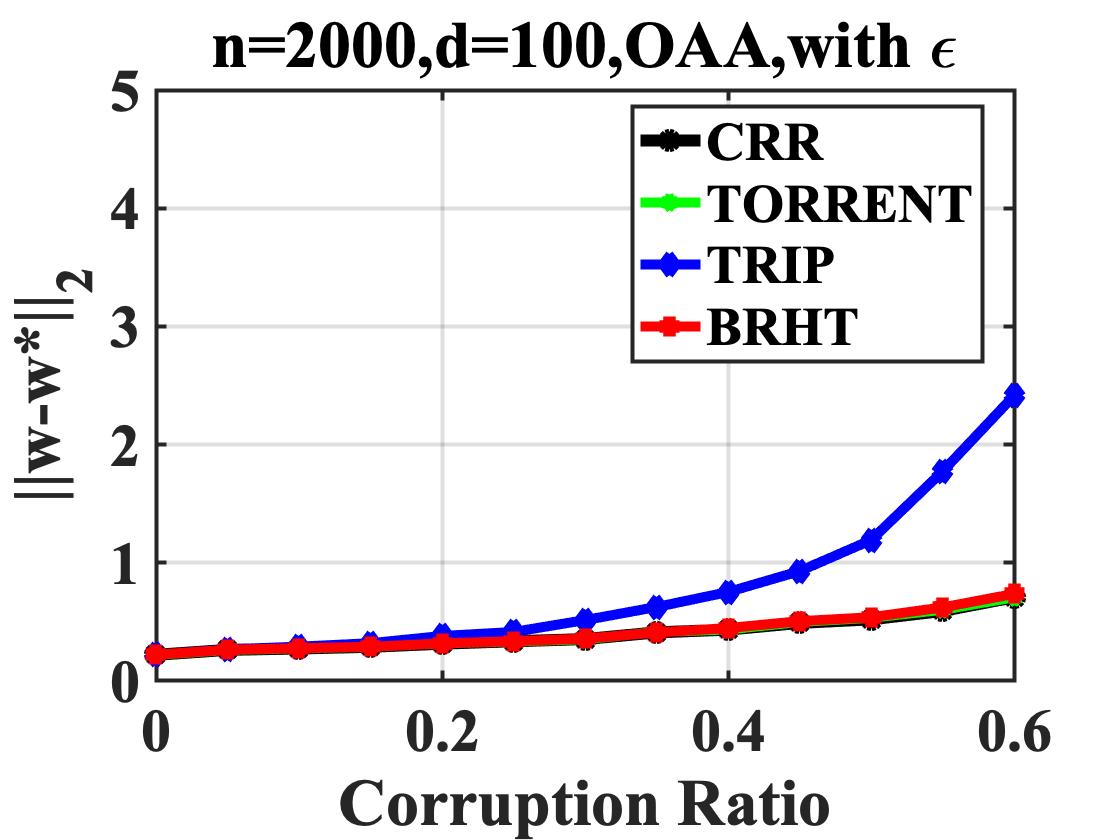}
\end{minipage}
}
\subfigure[]
{
\begin{minipage}{3cm}
\centering
\includegraphics[width=3cm]{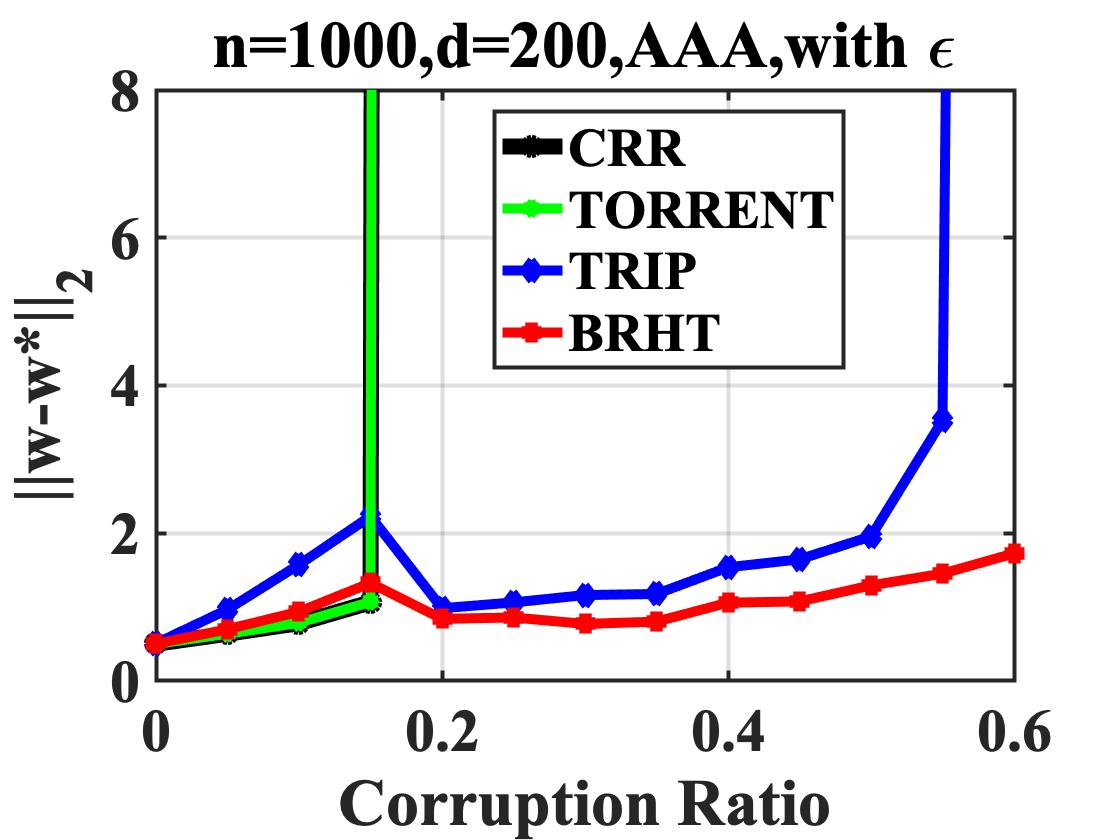}
\end{minipage}
}
\subfigure[]
{
\begin{minipage}{3cm}
\centering
\includegraphics[width=3cm]{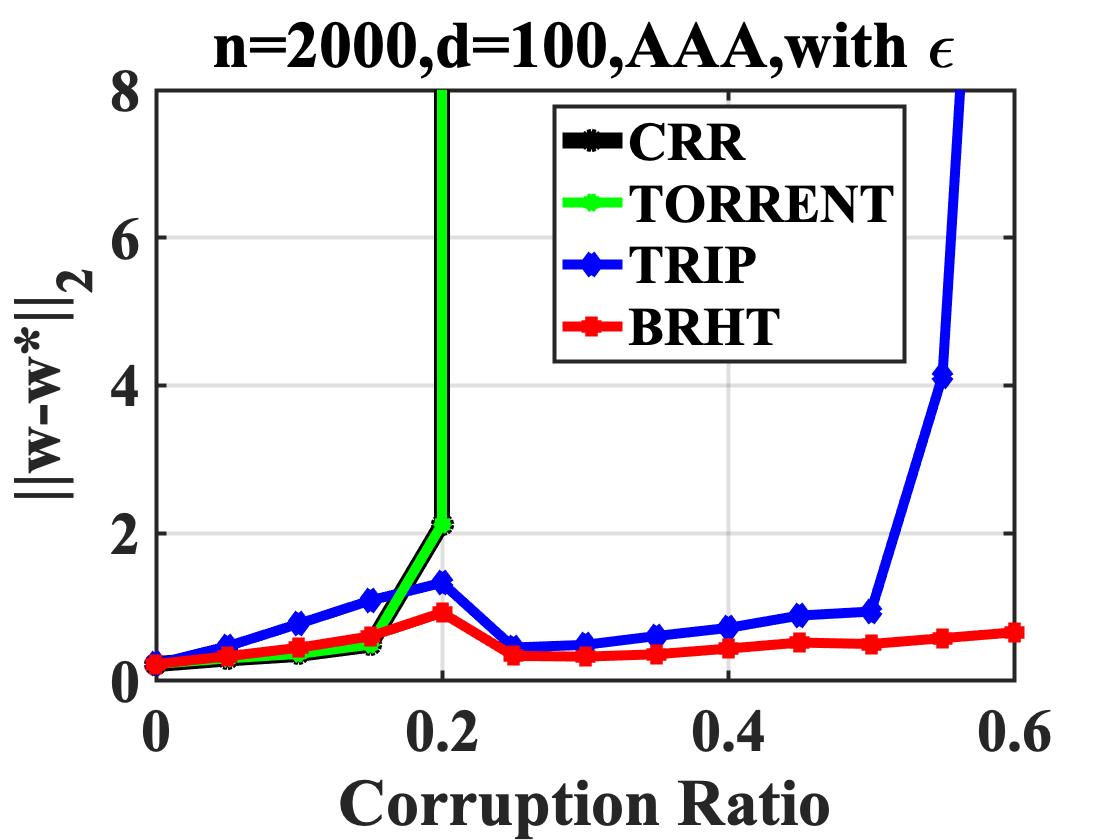}
\end{minipage}
}
\\
\subfigure[]
{
\begin{minipage}{3cm}
\centering
\includegraphics[width=3cm]{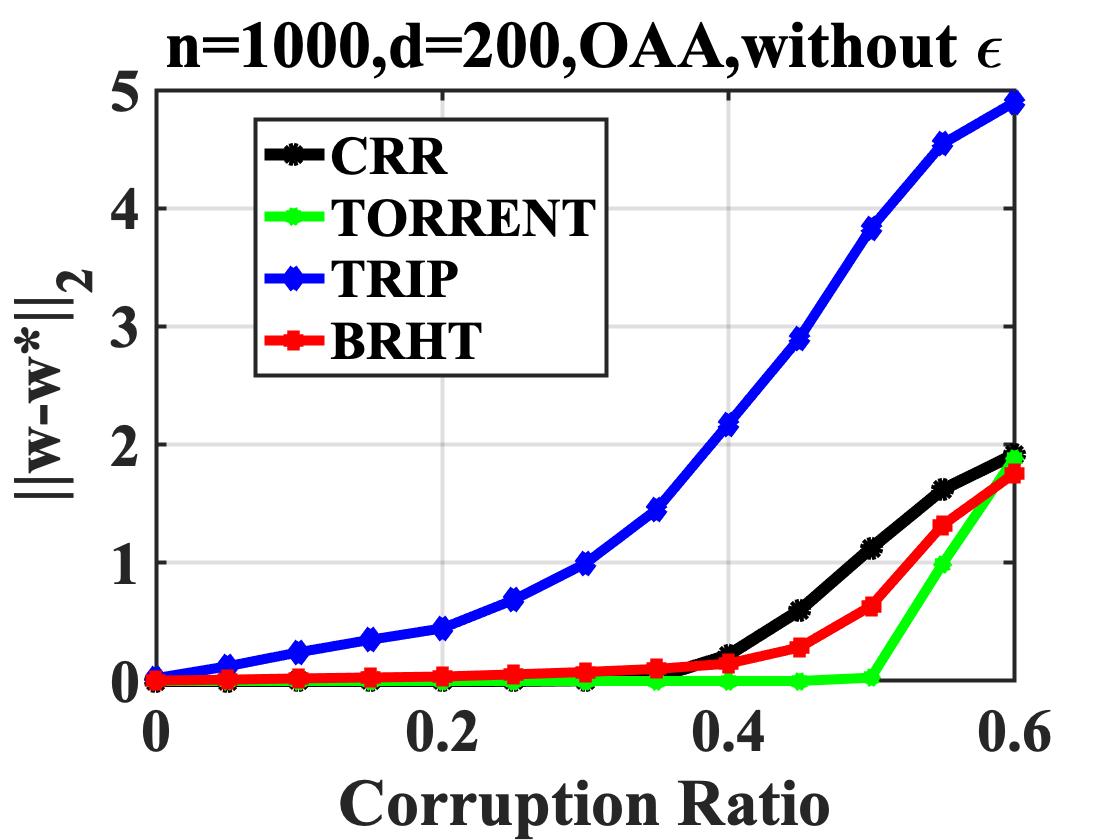}
\end{minipage}
}
\subfigure[]
{
\begin{minipage}{3cm}
\centering
\includegraphics[width=3cm]{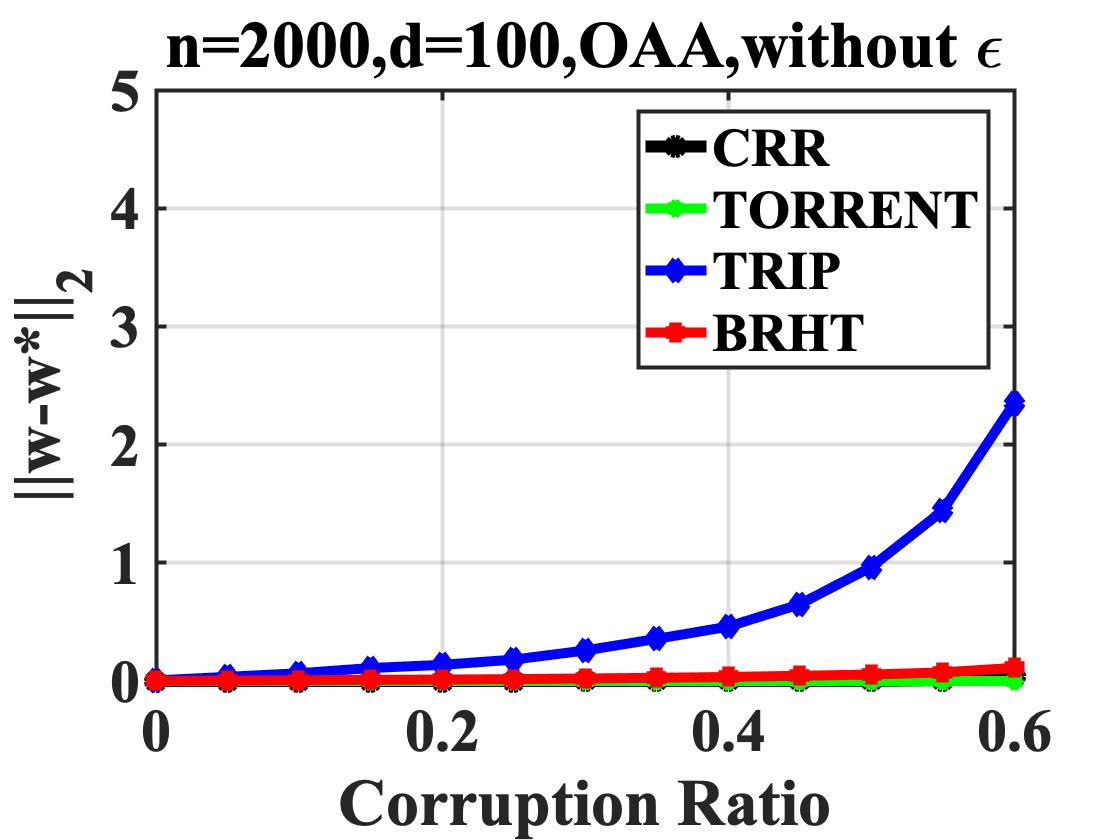}
\end{minipage}
}
\subfigure[]
{
\begin{minipage}{3cm}
\centering
\includegraphics[width=3cm]{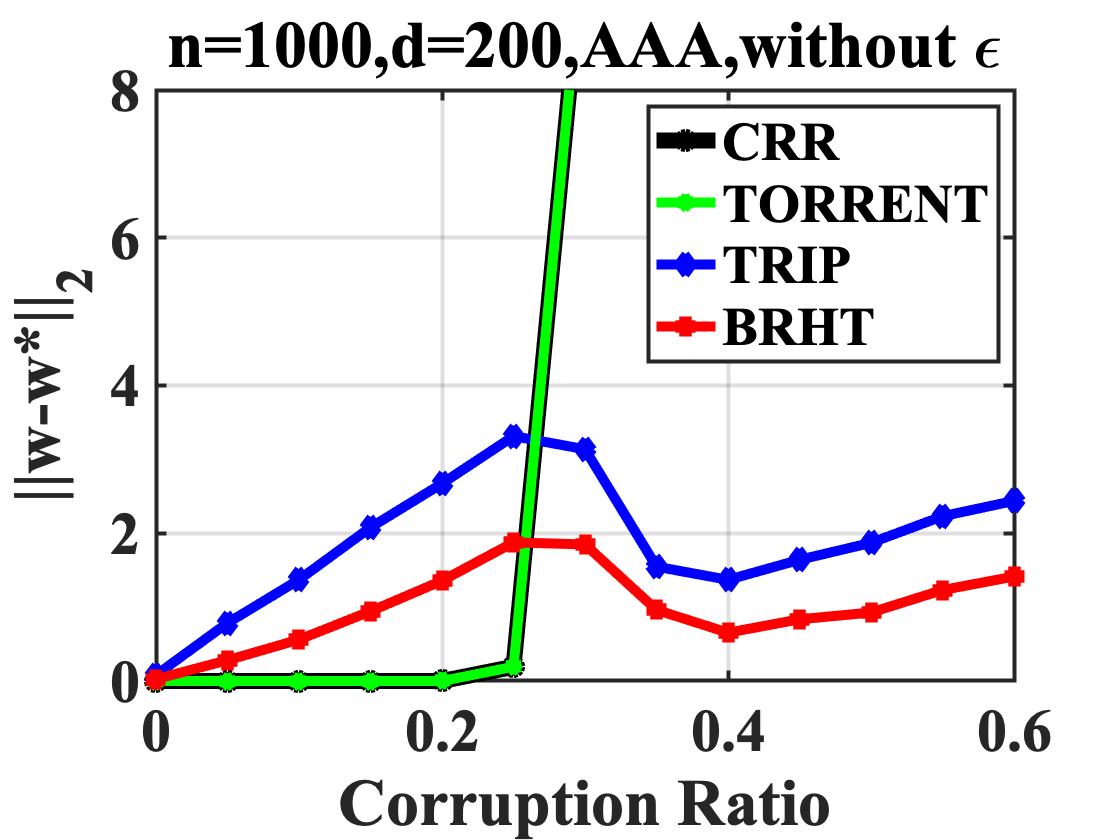}
\end{minipage}
}
\subfigure[]
{
\begin{minipage}{3cm}
\centering
\includegraphics[width=3cm]{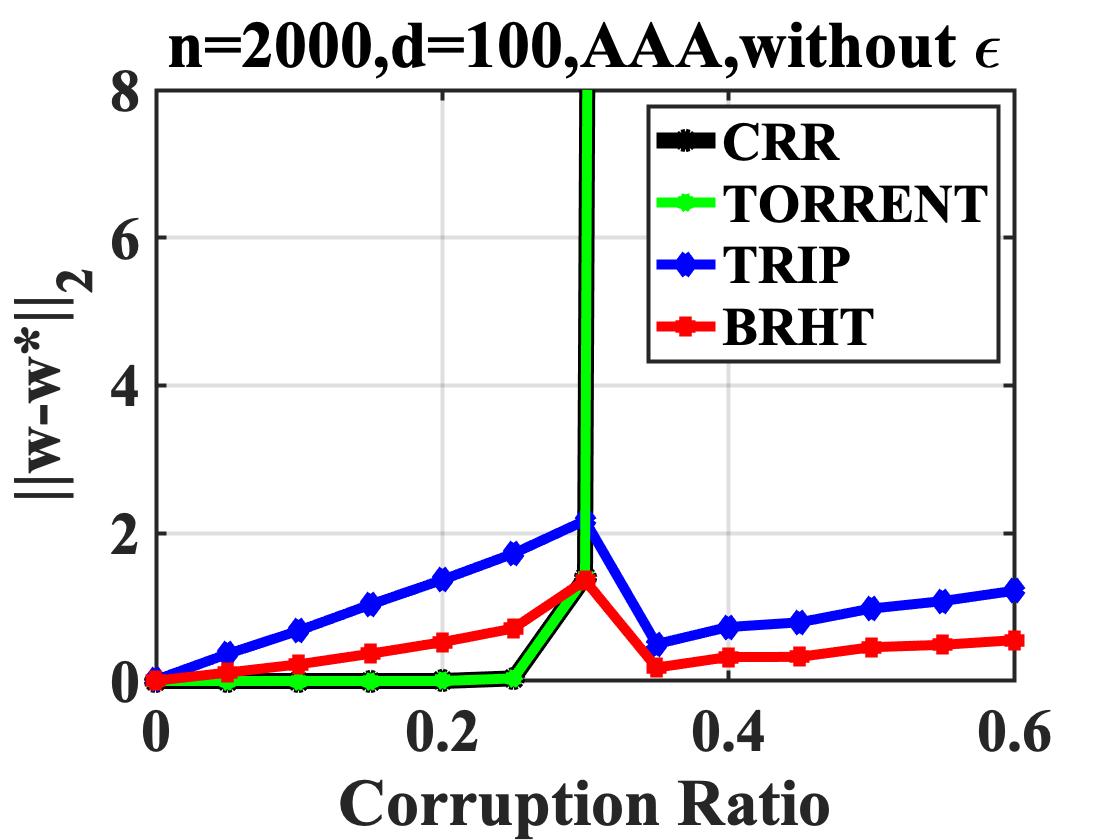}
\end{minipage}
}
\caption{Recovery of parameters with respect to the number of data points $n$, dimensionality $d$, and corruption ratio  $\alpha$. (a),(b),(c),(d) consider the case with white noise $\boldsymbol{\epsilon}$, while (e),(f),(g),(h) do not consider white noise. The performance of TORRENT and CRR is similar, and TRIP and BRHT are still more robust under AAAs than CRR and TORRENT.} 
\label{extra_exp1}
\end{figure*}

We also consider another leverage point attack (LPA) on data sets. For a point $(\mathbf{x}_{i},y_i)$, the leverage value is defined as $h_{ii}=\mathbf{x}_{i}^{T}(XX^{T})^{-1}\mathbf{x}_{i}$. In the linear regression, the regression result can be strongly affected by high leverage points \citep{chatterjee1986influential}. Therefore, if we corrupt those high leverage points, the regression result is more likely to be unstable. If we set the covariant $\mathbf{x}_i$ as iid in $\mathcal{N} (0,I_{{d}})$, then the high leverage points are roughly those points with large norms $\|\mathbf{x}_i\|_2$ since $\frac{1}{n}XX^{T}$ converges to $I_d$ as $n\to \infty$. According to the above analysis, we set the LPA as follows: choose $k$ points with the largest covariant norm $\|\mathbf{x}_i\|_2$ and set their corresponding $y_i$ to $0$. In this experiment, the true coefficient $\mathbf{w}^*$ is chosen to be a random unit norm vector and the covariant $\mathbf{x}_i$ are iid in $\mathcal{N} (0,I_{{d}})$. The true data (before attack) are also generated in two ways, one with white noise $y_{i}=\mathbf{x}_{i}^T{w}^*+\epsilon_{i}$ and the other without white noise $y_{i}=\mathbf{x}_{i}^T{w}^*$, where $\epsilon_{i}$ are iid in $\mathcal{N} (0,\sigma^2)$. We set $\sigma=1$ in the experiments. The experimental results are shown in Figure \ref{extra_exp2}. Under LPA, CRR performs poorly and usually collapses first among these methods. Rob-ULA has relatively better performance when the proportion of outliers is high, but there will be relatively large errors in the case of low proportion of outliers. TORRENT is very robust under LPA, especially in the absence of white noise. However, if the data dimension is high and the sample size is small, TORRENT is easier to collapse. The proposed TRIP and BRHT are still better than CRR, and will maintain a robust result even there are lots of outliers. The estimation errors of BRHT are smaller than TRIP, which shows BRHT is the most robust algorithm in this experiment.

\begin{figure*}[tbp]
	\centering
	\subfigure[]
	{
		\begin{minipage}{3cm}
			\centering
			\includegraphics[width=3cm]{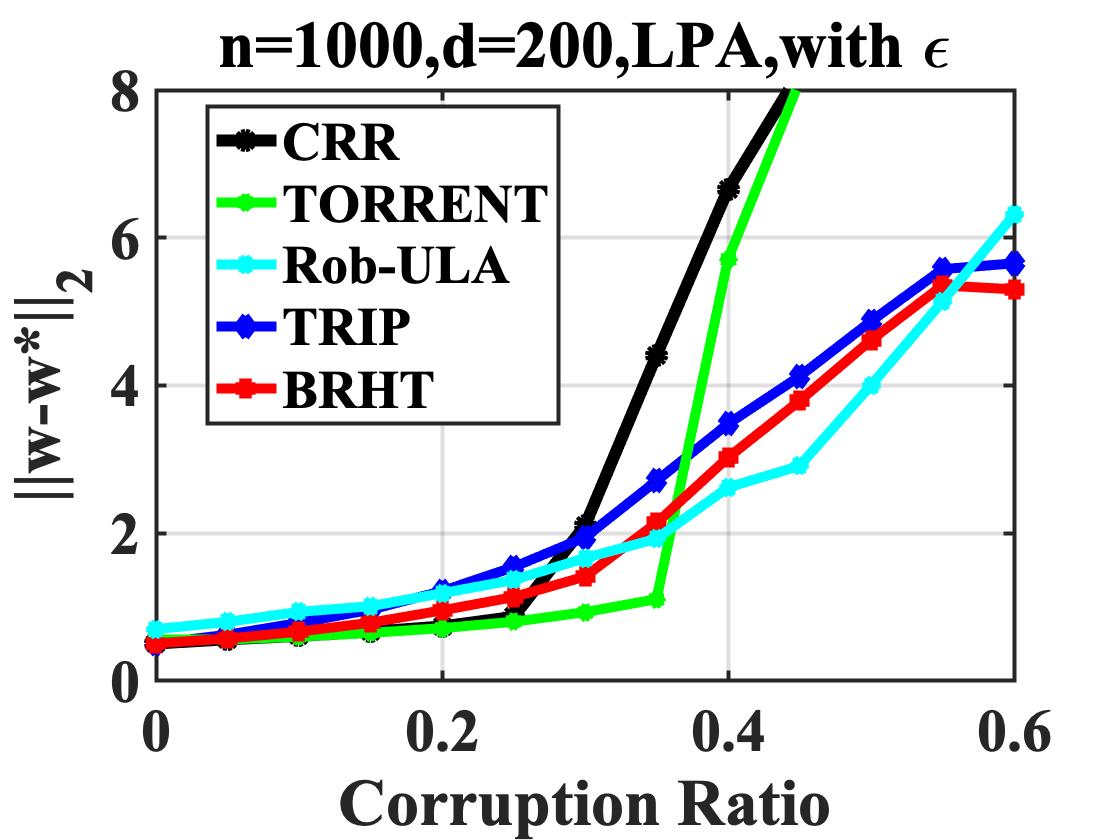}
		\end{minipage}
	}
	\subfigure[]
	{
		\begin{minipage}{3cm}
			\centering
			\includegraphics[width=3cm]{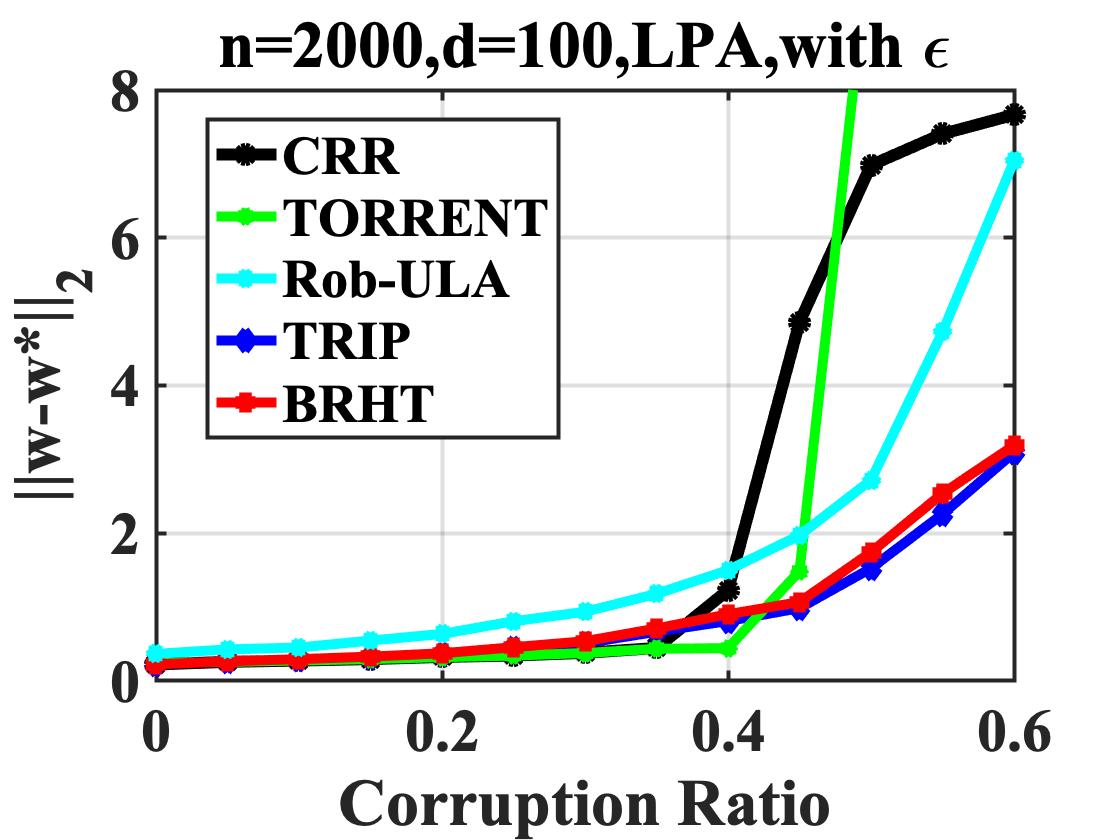}
		\end{minipage}
	}
	\subfigure[]
	{
		\begin{minipage}{3cm}
			\centering
			\includegraphics[width=3cm]{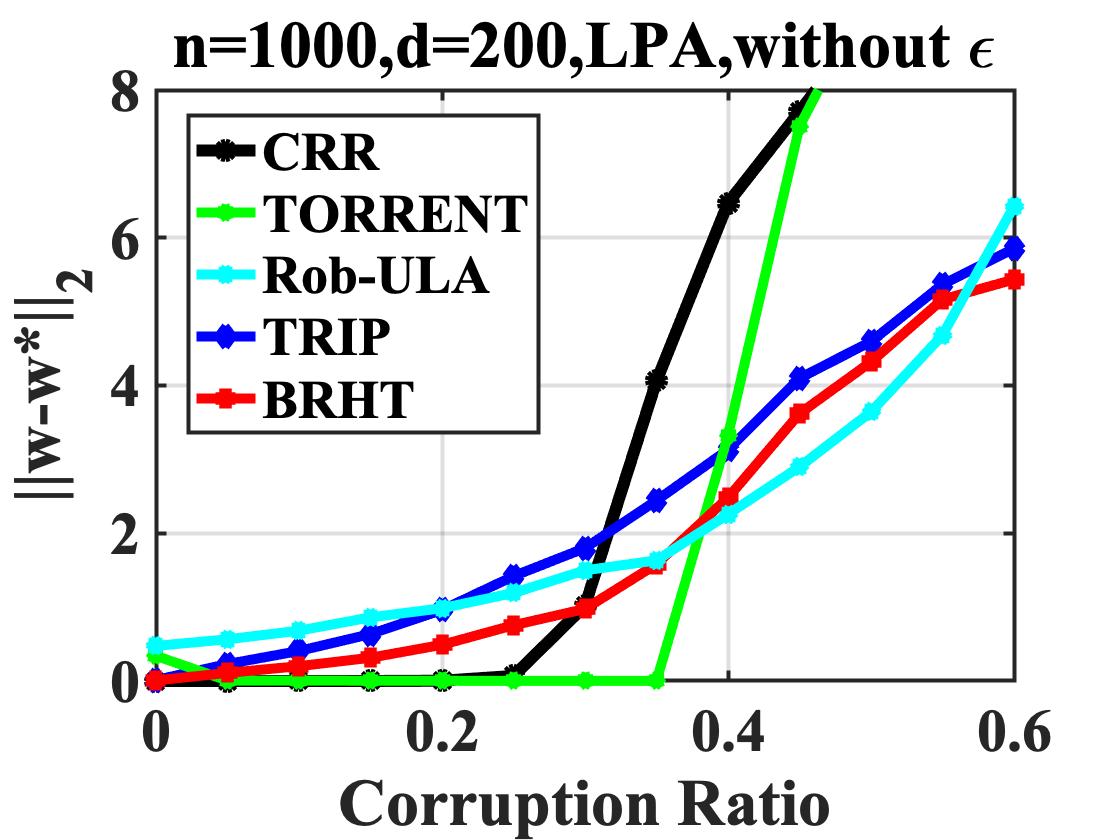}
		\end{minipage}
	}
	\subfigure[]
	{
		\begin{minipage}{3cm}
			\centering
			\includegraphics[width=3cm]{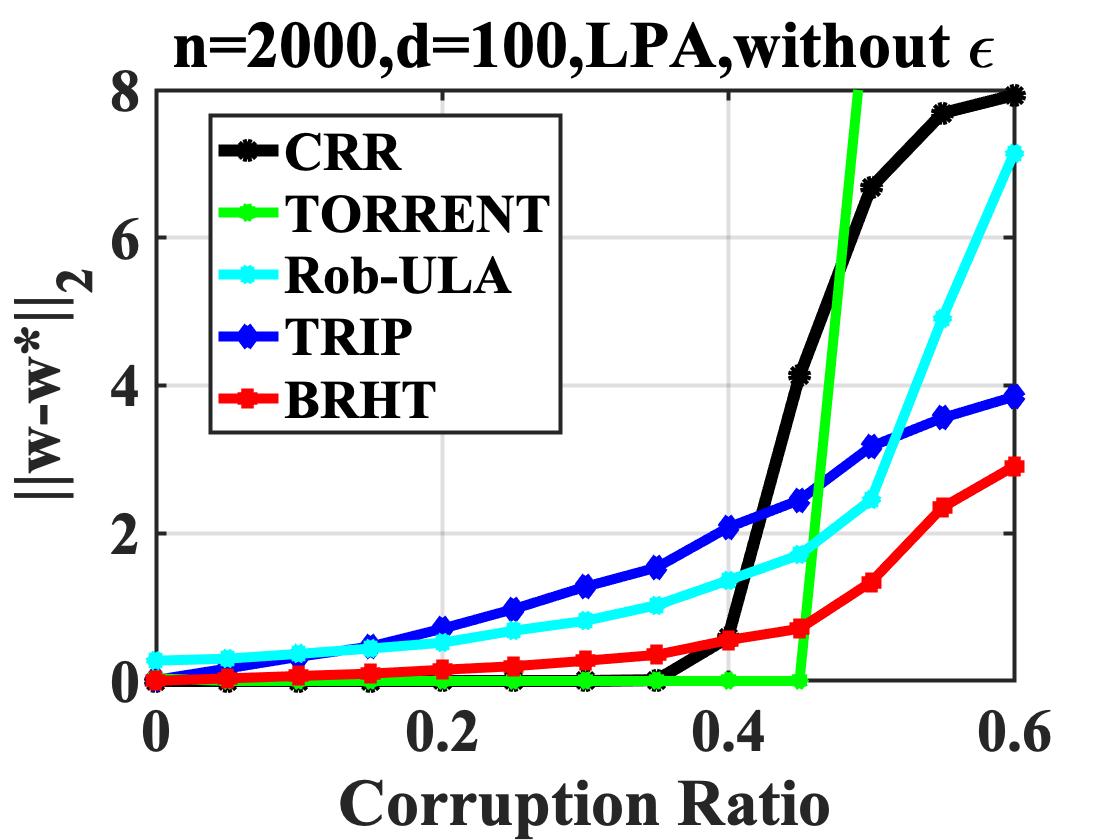}
		\end{minipage}
	}
	\caption{Recovery of parameters concerning the number of data points $n$, dimensionality $d$, and corruption ratio $\alpha$ under LPA.} 
\label{extra_exp2}
\end{figure*}

\newpage
\bibliographystyle{apalike}
\bibliography{reference_LP}